\documentclass[11pt]{article}

\usepackage{times}
\usepackage{makecell}
\date{}

\usepackage{xspace}
\usepackage{amsmath}
\usepackage{graphicx}
\usepackage{framed}
\usepackage{bbm}
\usepackage{algorithm}
\usepackage{amsthm}
\usepackage[algo2e, vlined, ruled, linesnumbered]{algorithm2e}
\usepackage{natbib}
\usepackage{mathrsfs}
\usepackage{amssymb}
\usepackage{bm}
\usepackage[colorlinks=true, linkcolor=blue, citecolor=blue, breaklinks=true]{hyperref}
\usepackage{tabulary}
\usepackage{xcolor}
\usepackage{prettyref}
\usepackage{mathtools}
\usepackage{amsmath}
\usepackage{enumitem}

\newcommand{\jz}[1]{{\color{olive}JZ: #1}}
\usepackage{multirow}
\usepackage{nicefrac}
\usepackage{amssymb}
\usepackage{pifont}
\definecolor{Green}{rgb}{0.13, 0.65, 0.3}
\usepackage{colortbl}
\newcommand{\KL}{\text{KL}}

\newcommand{\postupdate}{\textsc{PosteriorUpdate}\xspace}
\newcommand{\breg}{\textup{Breg}}
\newcommand{\est}{\textbf{\textup{Est}}\xspace}
\newcommand{\picirc}{\pi_{\phi^\star}}
\newcommand{\tilD}{\overline{D}}

\newcommand{\cmark}{\ding{51}}%
\newcommand{\xmark}{\ding{55}}%
\DeclareMathOperator*{\argmin}{argmin} 
\DeclareMathOperator*{\argmax}{argmax} 

\newcommand{\hatP}{\widehat{P}}
\newcommand{\bi}{\mathsf{av}}
\newcommand{\sbe}{\mathsf{sq}}
\newcommand{\hbi}{\mathsf{av}}
\newcommand{\hsbe}{\mathsf{sq}}
\newcommand{\decname}{Dig-DEC\xspace}

\newcommand{\tra}{\textbf{\textup{TermA}}}
\newcommand{\qcom}{\textbf{\textup{TermB}}}
\newcommand{\rew}{\textbf{\textup{TermC}}}
\newcommand{\leftt}{-}
\newcommand{\rightt}{+}

\newcommand{\Reg}{\text{\rm Reg}}

\newcommand{\numberthis}{\refstepcounter{equation}\tag{\theequation}}

\newcommand{\calT}{\mathcal{T}}

\newcommand{\calA}{\mathcal{A}}

\newcommand{\calS}{\mathcal{S}}
\newcommand{\calM}{\mathcal{M}}

\newcommand{\calI}{\mathcal{I}}

\newcommand{\calP}{\mathcal{P}}

\newcommand{\calD}{\mathcal{D}}
\newcommand{\calF}{\mathcal{F}}
\newcommand{\calO}{\mathcal{O}}

\newcommand{\err}{\xi}
\newcommand{\ellest}{\ell}

\newcommand{\calR}{\mathcal{R}}

\newcommand{\otil}{\widetilde{O}}

\newcommand{\E}{\mathbb{E}}

\newcommand{\bbR}{\mathbb{R}}
\newcommand{\order}{O}

\newcommand{\inner}[1]{\left\langle#1\right\rangle}

\newtheorem{theorem}{Theorem}
\newtheorem{assumption}{Assumption}

\newtheorem{lemma}[theorem]{Lemma}

\newtheorem{definition}[theorem]{Definition}
\newtheorem{example}{Example}

\usepackage{tikz}

\newcommand{\nonl}{\renewcommand{\nl}{\let\nl}}

\usepackage{prettyref}
\newcommand{\pref}[1]{\prettyref{#1}}

\newcommand{\savehyperref}[2]{\texorpdfstring{\hyperref[#1]{#2}}{#2}}
\newrefformat{eq}{\savehyperref{#1}{Eq.~\textup{(\ref*{#1})}}}
\newrefformat{eqn}{\savehyperref{#1}{Equation~\ref*{#1}}}
\newrefformat{lem}{\savehyperref{#1}{Lemma~\ref*{#1}}}
\newrefformat{lemma}{\savehyperref{#1}{Lemma~\ref*{#1}}}
\newrefformat{def}{\savehyperref{#1}{Definition~\ref*{#1}}}
\newrefformat{line}{\savehyperref{#1}{Line~\ref*{#1}}}
\newrefformat{thm}{\savehyperref{#1}{Theorem~\ref*{#1}}}
\newrefformat{corr}{\savehyperref{#1}{Corollary~\ref*{#1}}}
\newrefformat{cor}{\savehyperref{#1}{Corollary~\ref*{#1}}}
\newrefformat{sec}{\savehyperref{#1}{Section~\ref*{#1}}}
\newrefformat{subsec}{\savehyperref{#1}{Section~\ref*{#1}}}
\newrefformat{app}{\savehyperref{#1}{Appendix~\ref*{#1}}}
\newrefformat{assum}{\savehyperref{#1}{Assumption~\ref*{#1}}}
\newrefformat{ex}{\savehyperref{#1}{Example~\ref*{#1}}}
\newrefformat{fig}{\savehyperref{#1}{Figure~\ref*{#1}}}
\newrefformat{foot}{\savehyperref{#1}{Footnote~\ref*{#1}}}
\newrefformat{alg}{\savehyperref{#1}{Algorithm~\ref*{#1}}}
\newrefformat{rem}{\savehyperref{#1}{Remark~\ref*{#1}}}
\newrefformat{conj}{\savehyperref{#1}{Conjecture~\ref*{#1}}}
\newrefformat{prop}{\savehyperref{#1}{Proposition~\ref*{#1}}}
\newrefformat{proto}{\savehyperref{#1}{Protocol~\ref*{#1}}}
\newrefformat{prob}{\savehyperref{#1}{Problem~\ref*{#1}}}
\newrefformat{claim}{\savehyperref{#1}{Claim~\ref*{#1}}}
\newrefformat{que}{\savehyperref{#1}{Question~\ref*{#1}}}
\newrefformat{op}{\savehyperref{#1}{Open Problem~\ref*{#1}}}
\newrefformat{fn}{\savehyperref{#1}{Footnote~\ref*{#1}}}
\newrefformat{tab}{\savehyperref{#1}{Table~\ref*{#1}}}
\newrefformat{fig}{\savehyperref{#1}{Figure~\ref*{#1}}}
\newrefformat{proc}{\savehyperref{#1}{Procedure~\ref*{#1}}}

\newcommand{\dec}{\mathsf{dec}}
\newcommand{\mfdec}{\mathsf{dig}\text{-}\mathsf{dec}}
\newcommand{\odec}{\mathsf{o}\text{-}\mathsf{dec}}
\usepackage{caption}

\usepackage{accents}

\newcommand{\air}{\mathsf{AIR}}

\usepackage{subcaption}
\DeclareMathOperator*{\esttt}{est}
\DeclareMathOperator*{\unif}{unif}

\newcommand{\bo}{\pmb}

\renewcommand{\KL}{\text{KL}}

\usepackage[utf8]{inputenc} 
\usepackage[T1]{fontenc}    
\usepackage{hyperref}       
\usepackage{url}            
\usepackage{booktabs}       
\usepackage{amsfonts}       
\usepackage{nicefrac}       
\usepackage{microtype}      
\usepackage{xcolor}         

\PassOptionsToPackage{noend}{algorithm2e}

\usepackage{titletoc}
\usepackage{appendix}

\title{An Improved Model-Free Decision-Estimation Coefficient with Applications in Adversarial MDPs}

\author{%
    Haolin Liu \\
    \scalebox{0.9}{University of Virginia}\\
    \scalebox{0.9}{\texttt{srs8rh@virginia.edu}} 
    \and 
    Chen-Yu Wei \\
    \scalebox{0.9}{University of Virginia} \\ 
    \scalebox{0.9}{\texttt{chenyu.wei@virginia.edu}} 
    \and Julian Zimmert \\
    \scalebox{0.9}{Google Research} \\ 
    \scalebox{0.9}{\texttt{zimmert@google.com}}
}

\begin{document}

\maketitle

\begin{abstract}

We study decision making with structured observation (DMSO). Previous work \citep{foster2021statistical, foster2023tight} has characterized the complexity of DMSO via the decision-estimation coefficient (DEC), but left a gap between the regret upper and lower bounds that scales with the size of the model class. To tighten this gap, \cite{foster2024model} introduced optimistic DEC, achieving a bound that scales only with the size of the value-function class. However, their optimism-based exploration is only known to handle the stochastic setting, and it remains unclear whether it extends to the adversarial setting.


We introduce \decname, a model-free DEC that removes optimism and drives exploration purely by information gain. \decname is always no larger than optimistic DEC and can be much smaller in special cases. Importantly, the removal of optimism allows it to handle adversarial environments without explicit reward estimators. By applying \decname to hybrid MDPs with stochastic transitions and adversarial rewards, we obtain the first \emph{model-free} regret bounds for \emph{hybrid} MDPs with \emph{bandit} feedback under linear reward and several \emph{general} transition structures, resolving the main open problem left by \cite{liu2025decision}.


We also improve the online function-estimation procedure in model-free learning: For average estimation error minimization, we refine \cite{foster2024model}'s estimator to achieve sharper concentration, improving their regret bounds from $T^{\frac{3}{4}}$ to $T^{\frac{2}{3}}$ (on-policy) and from $T^{\frac{5}{6}}$ to $T^{\frac{7}{9}}$ (off-policy). For squared error minimization in Bellman-complete MDPs, we redesign their two-timescale procedure, improving the regret bound from $T^{\frac{2}{3}}$ to $\sqrt{T}$. This is the first time a DEC-based method achieves performance matching that of optimism-based approaches \citep{jin2021bellman, xie2022role} in Bellman-complete MDPs.

\end{abstract}
\section{Introduction}\label{sec: intro}
\cite{foster2021statistical, foster2023tight} developed the framework of decision-estimation coefficient (DEC) that characterizes the complexity of general online decision making problems and provides a general algorithmic principle called Estimation-to-Decision (E2D). 
In the state-of-the-art result by \cite{foster2023tight}, regret lower and upper bounds are established with a gap of $\log|\calM|$, where $\calM$ is the model class where the underlying true model lies. This $\log|\calM|$ reflects the price of \emph{model estimation}. 
Essentially, the lower bound in \cite{foster2023tight} only captures the complexity of decision-making / exploration, while the upper bound additionally includes the complexity of model estimation. 
Since E2D is a model-based algorithm that learns over models, it necessarily incurs this cost of model estimation. 

On the other hand, a large class of existing reinforcement learning (RL) algorithms are model-free value-based algorithms, which only estimate value functions. 
To better capture the decision-making complexity in this case, \citep{foster2024model} proposed a variant of E2D, called optimistic E2D, that achieves a regret upper bound characeterized by the complexity measure called optimistic DEC. However, unlike the model-based DEC/E2D framework \cite{foster2021statistical, foster2023tight} which drives exploration only through information gain, optimistic DEC/E2D leverages the \emph{optimism} principle to drive exploration, which may not be fundamental and could lead to sub-optimal performance in certain cases.  
Overall, the precise tradeoff between model estimation complexity and decision-making complexity, along with the gap between upper and lower bounds, remain largely unsolved.

A parallel line of reserach seeks to relax the assumption that the environment remains stationary. \cite{foster2022complexity} and \cite{xu2023bayesian} studied the pure adversarial setting where the environment can choose a different model in every round. In this case, their algorithms only estimate the optimal policy and the price of estimation becomes $\log|\Pi|$ where $\Pi$ is the policy class. 
In such pure adversarial environment, however, the decision-making complexity could become prohibitively high and is often vacuous in Markov decision processes (MDPs). 
A simpler and more tractable setting is the that of \emph{hybrid} MDPs where the transition is stochastic but the reward is adversarial. This setting has been studied in various settings: tabular MDPs \citep{neu2013online, rosenberg2019online, jin2020learning, shani2020optimistic}, linear (mixture) MDPs \citep{luo2021policy, dai2023refined, sherman2023improved, liu2023towards, kong2023improved, li2024improved}, and low-rank MDPs \citep{zhao2023learning, liu2024beating}. The work of \cite{liu2025decision} first leveraged the DEC framework to obtain results for \emph{bilinear classes}. However, they only gave a model-based algorithm (incurring large estimation error) and a model-free algorithm that requires full-information reward feedback, leaving the model-free bandit case open.

We provide a unified framework that advances both directions discussed above: 
\begin{itemize}[leftmargin=1em, before=\vspace{0pt}, after=\vspace{0pt}]
 \setlength{\itemsep}{0pt} 
    \setlength{\parskip}{0pt} 
    \setlength{\topsep}{0pt}
\item In the stochastic setting, we introduce a new model-free DEC notion, \decname, that improves over the optimistic DEC of \cite{foster2024model}. Our approach does not rely on the optimism principle, but adheres more closely to the general idea of DEC that drives exploration purely with information gain. For canonical settings such as bilinear classes or Bellman-complete MDPs with bounded Bellman eluder dimension or coverability, we recover their complexities with improved $T$-dependence in the regret, while in some constructed settings, the improvement can be arbitrarily large. 
\item We establish the first sublinear regret for \emph{model-free} learning in \emph{hybrid} bilinear classes and Bellman-complete coverable MDPs with linear reward and bandit feedback, resolving the open question in \cite{liu2025decision}. 
\item We improve the online function estimation procedure both in the case of average estimation error and squared estimation error. This allows us to improve the $T^{\frac{3}{4}}/T^{\frac{5}{6}}$ regret of \cite{foster2024model} to $T^{\frac{2}{3}}/T^{\frac{7}{9}}$ in the former case, and improve the $T^\frac{2}{3}$ regret of \cite{foster2024model} to $\sqrt{T}$ in the latter case. The techniques we use to achieve them could be of independent interest. 
\end{itemize}
Tables that compare our results with previous ones are provided in \pref{app: comparison tables}. 
Notably, our framework generalizes the Algorithmic Information Ratio (AIR) framework of \cite{xu2023bayesian} and \cite{liu2025decision}, substantially simplifying the analysis while enhancing algorithmic flexibility (\pref{sec: general framework}). This generalization may facilitate future development in this line of research.

We remark that, similar to \cite{foster2024model}, the term ``model-free'' learning in our work does not mean that the learner has no access to the model class $\calM$ or has computational constraints. Instead, it only means that the regret bound is independent of the size of the model set $\calM$. This implicitly restricts the learner from making fine-grained estimation over $\calM$.


\section{Preliminary}
\label{sec:pre}
We consider Decision Making with Structured Observations (DMSO) \citep{foster2021statistical}. Let 
$\calM$ be a model space, $\Pi$ a policy space, $\calO$ an observation space, and $V$ a value function. For simplicity, we $|\Pi|$ is finite. Each model $M \in \calM$ is a mapping from policy space $\Pi$ to a distribution over observations $\Delta\left(\calO\right)$. Every model $M\in\calM$ is associated with a value function $V_M: \Pi\to [0,1]$ that specifies the expected payoff of policy $\pi\in\Pi$ in model $M$. We denote $\pi_M=\argmax_{\pi\in\Pi}V_M(\pi)$.  

The learner interacts with the environment for $T$ rounds. In each round $t =1,\ldots, T$, the environment first chooses a model $M_t \in \calM$ without revealing it to the learner. Then the learner selects a policy $\pi_t \in \Pi$, and observes an observation $o_t \sim M_t(\cdot|\pi_t)$. The regret with respect to policy $\pi^\star\in\Pi$ is 
\begin{align*}
    \Reg(\pi^\star) = \sum_{t=1}^T \left(V_{M_t}(\pi^\star) - V_{M_t}(\pi_t)\right). 
\end{align*}

\paragraph{Markov Decision Process} 
\mbox{A Markov decision process is defined by a tuple $(\calS, \calA, P, R, H, s_1)$,} where $\calS$ is the state space, $\calA$ is the action space, $P:\calS\times\calA\rightarrow\Delta(\calS)$ is the transition kernel, $R:\calS\times\calA\rightarrow\Delta([0,1])$ is the reward distribution (with abuse of notation, we also use $R(s,a)$ to denote the expected reward $R(s,a)\in [0,1]$), $H$ the horizon, and $s_1$ the initial state. Assume $\calS=\bigcup_{h=1}^H \calS_h$ with $\calS_i \cap \calS_j =\emptyset$ for $i\neq j$, and $\calS_1=\{s_1\}$. 
In every step $h=1,2,\ldots, H$ within an episode, the learner observes the state $s_h\in\calS_h$ and selects an action $a_h\in\calA$. The learner then transitions to the next state via $s_{h+1}\sim P(\cdot|s_h,a_h)$, which is only supported on $\calS_{h+1}$, and receives the reward $r_h\sim R(s_h,a_h)$. 
We assume that the reward is constrained such that $\sum_{h=1}^H r_h\in [0,1]$ for any policy almost surely. Given a policy $\pi:\calS\to\calA$, the $Q$-function and $V$-function for $s\in\calS_h$ are defined by $Q^\pi(s, a)=\E^\pi[\sum_{h'=h}^{H}r_h\,|\,s_h=s,a_h=a]$ and  $V^\pi(s)=Q^\pi(s,\pi(s))$. The $Q$-function and $V$-function of an optimal policy $\pi^\star$ are abbreviated with 
$Q^\star$ and $V^\star$. We use $Q^\pi(s,a;M)$ and $Q^\star(s,a;M)$ to denote the $Q$-functions under model $M=(P,R)$. 

Learning in MDPs is a DMSO problem where $\calM=\calP\times \calR$ with $\calP$ being the set of transition kernels and $\calR$ the set of reward functions. A \emph{round} in DMSO corresponds to an MDP episode, and observation $o = (s_1, a_1, r_1, s_2, a_2, r_2, \ldots, r_H)$ is the trajectory.  For any function $g$, we write $\E^{\pi, M}[g(o)] = \E_{o\sim M(\cdot|\pi)}[g(o)]$. If $g(o)$ only depends on $(s_1,a_1,s_2,a_2, \ldots, a_H)$, we also write it as  $\E^{\pi, P}[g(o)]$.  We use $V_M(\pi) = \E^{\pi, M}[\sum_{h=1}^H r_h]$ to denote the expected total reward obtained by policy $\pi$ in MDP $M$, and $d_h^{\pi, M}(s,a)$ (or $d_h^{\pi, P}(s,a)$) the occupancy measure on step $h$ under policy $\pi$ and model $M$ (or transition $P$).

\subsection{$\Phi$-Restricted Learning} \label{sec: hybrid setting}

For DMSO, \cite{foster2021statistical, foster2023tight} and \cite{chen2022unified} studied the \emph{stochastic} setting where $M_t=M^\star$ for~all~$t$. They showed that the DEC characterizes the regret lower bound and captures the complexity of decision making. They proposed model-based algorithms with near-optimal upper bounds up to the model estimation complexity $\log|\calM|$. 
On the other hand, \cite{foster2022complexity} and \cite{xu2023bayesian} studied the pure \emph{adversarial} setting where $M_t$ arbitrarily changes over time. For this setting, they identified that DEC of the convexified model class characterizes the regret lower bound, which could be significantly larger than DEC of the original model class. Their upper bound replaces $\log|\calM|$ by $\log|\Pi|$, reflecting that they perform policy-based learning without finegrained estimation of the model. 

Several works go beyond pure model learning or pure policy learning. \cite{foster2024model} considered model-free value learning in the stochastic setting where only the value function is estimated, aiming to only incur $\log|\calF|$ estimation complexity, where $\calF$ is the value function set.  \cite{liu2025decision} and \cite{chen2025decision} considered the hybrid setting where part of the environment is stochastic and part adversarial, and the target of estimation is only on the optimal policy and the stochastic part of the environment. 

We base our presentation in \cite{liu2025decision}'s formulation, which can cover all cases mentioned above. 
\begin{definition}[Infosets and $\Phi$ \citep{liu2025decision,chen2025decision}]\label{def: infoset}
   Let $\Phi$ be a collection of subsets of $\calM\times\Pi$ satisfying: 1) The subsets are disjoint, i.e., for any $\phi, \phi'\in \Phi$, if $\phi\neq \phi'$, then $\phi\cap \phi'=\emptyset$. 2) Every $\phi$ contains a single policy, i.e., if $(M,\pi), (M', \pi')\in \phi$, then $\pi=\pi'$.
   We call a $\phi\in\Phi$ an \emph{information set} (\emph{infoset}). 
   Due to 2) above, each $\phi\in\Phi$ is associated with a unique policy. We denote this policy as $\pi_\phi$. We also define $\Psi \triangleq \bigcup_{\phi \in \Phi} \phi \subseteq \calM \times \Pi$. 
\end{definition}

With \pref{def: infoset}, for given $\rho\in\Delta(\Phi)$, $p\in\Delta(\Pi)$, $\nu\in\Delta(\Psi)$, and $\eta>0$, \cite{liu2025decision} defined $\Phi$-AIR:
\begin{align}
    \air^{\Phi}_{\eta}(p,\nu; \rho) = \E_{\pi\sim p} \E_{(M,\pi^\star)\sim \nu} \E_{o\sim M(\cdot|\pi)}\left[V_M(\pi^\star) - V_M(\pi) - \frac{1}{\eta} \KL(\nu_{\bo{\phi}}(\cdot|\pi, o), \rho)\right],  \label{eq:phiair}
\end{align}
where $\nu_{\bo{\phi}}(\cdot|\pi,o)$\footnote{We use the notational convention in \cite{liu2025decision}: the bold subscript in $\nu_{\bo{\phi}}(\cdot|\pi,o)$ specifies the \emph{identity} of the variable represented by `~$\cdot$~', instead of a \emph{realized value} of that variable. The subscript may be omitted~when~clear.} is the posterior over $\phi$ given $(\pi, o)$, which satisfies $\nu(\phi|\pi,o)\propto \sum_{(M,\pi^\star)\in \phi} \nu(M,\pi^\star) M(o|\pi)$. 
$\Phi$-AIR can characterize the decision-making complexity in the $\Phi$-restricted environment defined below: 
\begin{definition}[$\Phi$-resitricted environment \citep{liu2025decision, chen2025decision}]\label{def: restricted env} A $\Phi$-restricted environment is an (adversarial) decision making problem in which the environment commits to $\phi^\star\in\Phi$ at the beginning of the game and henceforth selects $(M_t,\pi_{\phi^\star})\in \phi^\star$ in every round $t$ arbitrarily based on the history. 
\end{definition}

\begin{theorem}[\cite{liu2025decision}]\label{thm: air phi} For $\Phi$-restricted environment defined in \pref{def: restricted env}, there exists an algorithm ensuring $\E[\Reg(\picirc)]\leq  \E\big[\sum_{t}\min_p \max_\nu \air_\eta^\Phi(p,\nu;\rho_t)\big] + \frac{\log|\Phi|}{\eta}$. 
\end{theorem}

\subsection{Results and Open Questions in \cite{liu2025decision}} \label{sec: open question}
\cite{liu2025decision}'s main results are based on $\Phi$-AIR: For \emph{model-free} learning in \emph{stochastic} MDPs, \cite{liu2025decision} obtained $\sqrt{T}$ regret for linear $Q^\star/V^\star$ MDPs (before their result, the best known rate is $T^{\frac{2}{3}}$). Unfortunately, their algorithm cannot handle other canonical settings such as bilinear classes, MDPs with bounded Bellman-eluder dimension, or MDPs with bounded coverability. For \emph{model-based} learning in \emph{hybrid} MDPs where the transition is fixed but the reward function changes arbitrarily over time, \cite{liu2025decision} obtained near-optimal regret bounds for general cases up to a $\log(|\calP||\Pi|)$ factor.  

An attempt was made by \cite{liu2025decision} to handle \emph{model-free} learning in \emph{hybrid} MDPs based on an extension of the optimistic DEC approach \citep{foster2024model}.  However, their result only handles \emph{full-information} reward feedback. Extension to the bandit setting is challenging under this framework as the optimistic update requires an explicit construction of the reward estimator. 

In this work, we focus on model-free learning in both stochastic and hybrid MDPs. Our results generalize those of \cite{liu2025decision} in both directions: Our framework handles all canonical settings for \emph{model-free} learning in \emph{stochastic} MDPs, improving previous results by \cite{foster2024model}. It also handles \emph{model-free} learning in \emph{hybrid} MDPs with \emph{bandit} feedback under the same reward assumption as \cite{liu2025decision}.  




\section{Settings and Assumptions}
\label{sec: settings and assumptions}
Below, we show how to view model-free learning in stochastic and hybrid MDPs as learning in $\Phi$-restricted environments (\pref{def: restricted env}), and introduce the assumptions used in the paper.  

\subsection{The Stochastic Setting}

\begin{definition}[Stochastic setting]\label{def: stochastic sett}
    In the stochastic setting, the environment commits to $M^\star$ at the beginning of the game and sets $M_t=M^\star$ in every round $t$. 
\end{definition}
For model-free learning in the stochastic setting, we assume the following: 
\begin{assumption}[$\Phi$ for model-free learning in stochastic MDPs]\label{assum: function approximation stochastic}
In the stochastic setting, in addition to $(\calM,\Pi, \calO, V)$ in the DMSO framework (\pref{sec:pre}), the learner is provided with a function set~$\calF$. Each model $M\in \calM$ \emph{induces} a function $f\in\calF$. Assume that models inducing the same $f$ have the same $Q^\star$ function and hence the same optimal policy $\pi_M$ (for example, an $\calF$ that contains all possible $Q^\star$ functions satisfies this, though $\calF$ could also provide additional information). With this, $\Phi$ is created by partitioning $\calM$ according to the function they induces: Define $\Phi=\{\phi_f: f\in\calF\}$ where $\phi_f=\{(M,\pi_M): M\text{~induces~}f\}$.  With abuse of notation, we write $M\in\phi$ to indicate that $(M, \pi_M)\in \phi$.  We denote by $\pi_\phi$ the common optimal policy for all $M\in\phi$, and by $f_\phi(s,a)$ the $Q^\star$ function induced by $M\in\phi$, i.e., $f_\phi(s,a)=Q^\star(s,a; M)$ for all $M\in\phi$. Define $f_\phi(s)=\max_{a} f_\phi(s,a)$. We also use $V_\phi(\pi_\phi):=f_\phi(s_1)$ to denote the value of policy $\pi_\phi$ under any model in $\phi$.   
\end{assumption}



\subsection{The Hybrid Setting}
\begin{definition}[Hybrid setting]\label{def: hybrid sett} In the hybrid setting, the environment commits to $P^\star\in\calP$ at the beginning of the game. In every round, the environment selects $R_t\in\calR$ arbitrarily based on the history and sets $M_t=(P^\star, R_t)$. 
\end{definition}
For model-free learning in the hybrid setting, the definition of $\Phi$ becomes more involved as it partitions over three dimensions $(\Pi, \calP, \calR)$ in different ways. Formally, the partition should satisfy the following \pref{assum: function approximation hybrid}. We provide an illustration in \pref{fig: illustrate} in \pref{app: illustrate} to help the reader understand this assumption. 

\begin{assumption}[$\Phi$ for learning in hybrid MDPs \citep{liu2025decision}]\label{assum: function approximation hybrid}
     The learner is provided with a function set $\calF^\pi$ for every $\pi\in\Pi$. For any fixed $\pi$, each transition $P\in\calP$ induces a function $f\in\calF^\pi$. $\Phi$ is created by partitioning $\calP\times \calR\times \Pi$ firstly according to $\pi$, and then according to the $f$ the transition induces in $\calF^\pi$: Define $\Phi=\{\phi_{\pi, f}:  \pi\in\Pi, f\in\calF^\pi\}$, where $\phi_{\pi, f}=\{(P, R, \pi): P \text{\ induces\ } f \text{\ in\ }\calF^\pi, R\in\calR\}$. We write $P\in\phi$ if there exists $R, \pi$ such that $(P, R, \pi)\in\phi$, and write $M=(P,R)\in\phi$ if $P\in\phi$. We denote by $\pi_\phi$ the unique $\pi\in\Pi$ defining $\phi\in\Phi$.      
\end{assumption}

The next assumption describes the requirement for the function set in our work. 
\begin{assumption}[Unique reward to value mapping given $\phi$ \citep{liu2025decision}]\label{assum: unique mapping}
    Let $\Phi$ satisfy \pref{assum: function approximation hybrid}. Assume that for any fixed $\phi$ and $P, P'\in \phi$, it holds that $Q^{\pi_\phi}(s,a; (P, R)) = Q^{\pi_\phi}(s,a; (P', R))$ for any $s, a, R$. We denote $f_\phi(s,a; R) = Q^{\pi_\phi}(s,a; (P, R))$ for any $P\in\phi$, and define $f_\phi(s; R) = \E_{a\sim \pi_\phi(\cdot|s)}[f_\phi(s,a; R)]$. We also use $V_{\phi, R}(\pi_\phi) = f_\phi(s_1; R)$ to denote the value of policy $\pi_\phi$ under $(P, R)$ for any $P\in \phi$.

\end{assumption}

To understand \pref{assum: function approximation hybrid} and \pref{assum: unique mapping} better, we take adversarial linear MDP \cite{liu2023towards} for example. In adversarial linear MDPs, the learner is given a known feature mapping $\varphi(s,a)\in\mathbb{R}^d$, such that the reward function can be represented as $R(s,a)=\varphi(s,a)^\top \theta_R$ and the transition as $P(s'|s,a) = \varphi(s,a)^\top \omega_P(s')$. In this case, one can show that for any $\pi$, $Q^\pi(s,a; P_1,R) = Q^\pi(s,a; P_2,R)$ $\forall s,a,R$ if and only if $\E^{\pi,P_1}[\phi(s_h,a_h)] = \E^{\pi,P_2}[\phi(s_h,a_h)]$ for all $h$. Based on \pref{assum: unique mapping}, we would like to put such $P_1$ and $P_2$ in the same partition under $\pi$ (see \pref{fig: illustrate} for an illustration).   In other words, in \pref{assum: function approximation hybrid}, each $f\in\calF^\pi$ corresponds to a unique value of $(\E^{\pi,P}[\phi(s_h,a_h)])_{h\in[H]}\in\mathbb{R}^{dH}$, and as long as two $P$'s share this value, they both belong to $\phi_{\pi,f}$.    

We remark that while \pref{assum: unique mapping} is a reasonable generalization of \pref{assum: function approximation stochastic} to the hybrid setting, it does not capture all learnable hybrid MDPs we are aware of. For example, if the transition space is partitioned according to \pref{assum: unique mapping} for hybrid low-rank MDPs with \emph{unknown reward feature}, then $\log|\Phi|$ will scale \emph{polynomially} with the number of possible feature mappings. In contrast, the work of \cite{liu2024beating} handles this case with the regret scaling only \emph{logarithmically} with the number of possible feature mappings. There is still technical difficulty in handling this case in our framework, and we leave it as future work.\footnote{The algorithm of \cite{liu2024beating} begins with reward-free exploration to learn a feature mapping, followed by online learning over that fixed feature mapping. While this two-phase approach could potentially be integrated into our DEC framework in special cases, our goal is to explore approaches that avoid such design to address more general scenarios.} We also remark that the previous work by \cite{liu2025decision} has the same limitation even in the full-information case. 

Therefore, in this work, for the hybrid setting, we consider linear reward with \emph{known} features, formally stated in the next assumption.

\begin{assumption}[Linear reward with known feature]\label{assum: known feature}
    There exists a feature mapping $\varphi: \calS\times\calA\to \mathbb{R}^d$ known to the learner such that for any $R\in\calR$, $R(s_h,a_h) = \varphi(s_h,a_h)^\top \theta_h(R)$ for all  $(s_h,a_h)\in\calS_h\times \calA$ for some $\theta_h(R)\in\mathbb{R}^d$.
\end{assumption}

While the stochastic setting (\pref{def: stochastic sett}) and the hybrid setting (\pref{def: hybrid sett}) are special cases of $\Phi$-restricted environments (\pref{def: restricted env}), the adversary in these special cases has additional restriction: for example, in the stochastic setting, the adversary is allowed to choose $M^\star\in\phi^\star$ at the beginning of the game, but has to stick to $M^\star$ throughout interactions. Similarly, $P^\star$ has to be fixed in the hybrid setting. This is different from the general $\Phi$-restricted setting where the adversary is allowed to choose $M_t\in\phi^\star$ arbitrarily in every round. However, using such a ``coarser'' partition $\Phi$ to model these settings is crucial for obtaining an improved estimation error that only scales with the size of the value function set. 


\section{General Framework}\label{sec: general framework}
This section introduce a general framework and complexity measure for the $\Phi$-restricted environment, which covers model-free learning in stochastic and hybrid MDPs as special cases. For given $\rho\in\Delta(\Phi)$, define for $p\in\Delta(\Pi)$ and $\nu\in\Delta(\Psi)$ 
\begin{align}
    \air^{\Phi, D}_{\eta}(p,\nu; \rho) = \E_{\pi \sim p}\E_{(M,\pi^\star) \sim \nu}\left[V_{M}(\pi^\star) -V_{M}(\pi) - \frac{1}{\eta} D^\pi(\nu\|\rho)\right],  \label{eq: new AIR} 
\end{align}
for some divergence measure $D^\pi(\nu\|\rho)$ convex in $\nu$ for any $\pi $ and $\rho$. $\Phi$-AIR defined in \pref{eq:phiair} is a special case where $D^\pi(\nu\|\rho) = \E_{M\sim \nu}\E_{o\sim M(\cdot|\pi)}[\KL(\nu_{\bo{\phi}}(\cdot|\pi,o), \rho)]$. The general algorithm designed based on \pref{eq: new AIR} is shown in \pref{alg:general}. 



\begin{algorithm}[H]
    \caption{General Framework}
     \label{alg:general}
    \textbf{Input:} Set of partitions $\Phi$ and its union $\Psi$ (defined in \pref{sec: hybrid setting}). \\
    $\rho_1(\phi) = 1/|\Phi|,\, \forall \phi \in \Phi$. \\
    \For{$t=1, 2, \ldots, T$}{
       Set $p_t, \nu_t$ as the solution of the following minimax optimization (defined in \pref{eq: new AIR}): 
        \begin{align}
            \min_{p \in \Delta(\Pi)}\max_{\nu \in \Delta(\Psi)} \air^{\Phi, D}_\eta(p,\nu; \rho_t). 
            \label{eq:minmax}
        \end{align}
        Execute $\pi_t \sim p_t$, and observe $o_t\sim M_t(\cdot|\pi_t)$. 
        \vspace{-5pt}
        \begin{flalign}
        &\text{Update\ } \rho_{t+1}= \postupdate(\nu_t, \rho_t, \pi_t, o_t). && \label{eq:rho}
        \end{flalign}
        
        }
        \vspace{-3pt}
\end{algorithm}
\pref{alg:general} has two main steps. First, given the infoset distribution $\rho_t\in\Delta(\Phi)$, solve the policy distribution $p_t$ and the worst-case world distribution $\nu_t$ in the saddle-point problem \pref{eq:minmax}. This is similar to the previous AIR framework in \cite{xu2023bayesian} and \cite{liu2025decision}. After taking policy $\pi_t\sim p_t$ and receiving the observation $o_t\sim M_t(\cdot|\pi_t)$, perform a posterior update by incorporating new information from $o_t$ (\pref{eq:rho}) and obtain the new infoset distribution $\rho_{t+1}\in\Delta(\Phi)$. In \cite{xu2023bayesian} and \cite{liu2025decision}, this posterior update step is simply $\rho_{t+1}(\phi) = \nu_t(\phi|\pi_t, o_t)$, but it could take different forms in our case depending on the specific divergence $D$ instantiated later. 

The ability of our algorithm to handle a general divergence $D$ is enabled by our new analysis techniques. The update rule $\rho_{t+1}(\phi)=\nu_t(\phi|\pi_t, o_t)$ in \cite{xu2023bayesian} and \cite{liu2025decision} and the corresponding regret analysis heavily relies on a ``constructive minimax theorem'' \citep{xu2023bayesian} that is restricted to strictly convex divergence measures and somewhat cumbersome to generalize to divergence other than $\KL$. Our new analysis, on the other hand, is more flexible and nicely connects to the standard analysis of mirror descent. 



Our analysis goes as follows. For any $(M,\pi)\in\calM\times\Pi$,  denote $\delta_{M,\pi}\in \Delta(\calM\times \Pi)$ as the Kronecker delta function centered at $(M,\pi)$. That is, $\delta_{M,\pi}(M,\pi)=1$ and $\delta_{M,\pi}(M',\pi')=0$ for any other $(M',\pi')$. By a simple first-order optimality condition (\pref{lem: bregman}) and the fact that $\nu_t$ is a best response to $p_t$ (\pref{eq:minmax}), we have (recall the definition of $\picirc$ in \pref{def: restricted env})
\begin{align}
    &\E_{\pi \sim p_t}\left[V_{M_t}(\picirc) - V_{M_t}(\pi) - \frac{1}{\eta}D^\pi(\delta_{M_t, \picirc} \| \rho_t)\right]  \label{eq:general-decom}
    \\&\le \max_{\nu \in \Delta(\Psi)}\E_{\pi \sim p_t}\E_{(M,\pi^\star) \sim \nu}\left[V_{M}(\pi^\star) -V_{M}(\pi) - \frac{1}{\eta}D^\pi(\nu\| \rho_t)\right] - \E_{\pi \sim p_t}\left[\frac{1}{\eta} \breg_{D^\pi(\cdot\|\rho_t)}(\delta_{M_t, \picirc},  \nu_t)\right] \nonumber
\end{align}
where $\breg_{F}(x,y) = F(x) - F(y) - \inner{\nabla F(y), x-y}\geq 0$ is the Bregman divergence defined with a convex function $F$. 
Since $p_t$ is minimax solution in \pref{eq:minmax}, after rearrangement of \pref{eq:general-decom} and summation over $t$, we get
\begin{align*}
    &\sum_{t=1}^T \left(V_{M_t}(\picirc) - \E_{\pi\sim p_t} \left[V_{M_t}(\pi)\right]\right)  \label{eq:minimax-guarantee} \numberthis
    \\[-15pt]&\le   \sum_{t=1}^T \min_{p \in \Delta(\Pi)}\max_{\nu \in \Delta(\Psi)}\air^{\Phi, D}_\eta(p,\nu; \rho_t) + \frac{1}{\eta} \overbrace{\sum_{t=1}^T \E_{\pi\sim p_t}\left[D^{\pi}(\delta_{M_t,\picirc}\|\rho_t) - \breg_{D^{\pi}(\cdot\|\rho_t)}(\delta_{M_t,\picirc}, \nu_t)\right] }^{\est}, 
\end{align*}
where we use the definition in \pref{eq: new AIR}. 
From \pref{eq:minimax-guarantee}, we have the following theorem. 
\begin{theorem}\label{thm: general thm}
   \pref{alg:general} achieves $\E[\Reg(\picirc)]\leq \E\big[\sum_{t}\min_p \max_\nu \air^{\Phi, D}_\eta(p,\nu;\rho_t) + \frac{\est}{\eta}\big]$. 
\end{theorem}
The \postupdate in \pref{eq:rho} has  to be further designed in order to minimize $\est$. In \pref{app: recovering}, we show how our new analysis recovers previous results of \cite{xu2023bayesian} and \cite{liu2025decision} easily. We remark that when recovering \cite{liu2025decision}'s result for model-based learning in hybrid MDPs with full-information feedback, we chooses $D$ such that $\est$ does not even scale with $\log|\Phi|$, while they achieve it with a more complex two-level algorithm. This shows the flexibility of our framework. 
In the next two subsection, we discuss about the two terms in the regret bound of \pref{thm: general thm}. 


\subsection{Divergence Measure in \pref{alg:general} and $\mfdec$}
To handle the MDPs of interest in \pref{sec: settings and assumptions}, we will instantiate \pref{alg:general} with the following divergence $D$: 
\begin{align}
    D^\pi(\nu\|\rho) = \E_{M \sim \nu}\E_{o \sim M(\cdot|\pi)}\left[\KL\left(\nu_{\bo{\phi}}(\cdot|\pi, o), \rho\right) + \E_{\phi\sim \rho}\left[\tilD^\pi(\phi\|M)\right]\right], \label{eq: our divergence}
\end{align}
where $\tilD^\pi(\phi\|M)$ is another divergence that measures the discrepancy between infoset $\phi$ and model $M$. Two choices of $\tilD$ will be introduced later in \pref{sec: postupdate sec}: \emph{averaged estimation error} and \emph{squared estimation error}. 

With this definition of $D^\pi(\nu\|\rho)$, the first term in the regret bound in \pref{thm: general thm} can be bounded by the following complexity: 
\begin{align*}
    &\mfdec_\eta^{\Phi, \tilD}
    \triangleq \max_{\rho\in\Delta(\Phi)} \min_{p\in\Delta(\Pi)} \max_{\nu\in \Delta(\Psi)} \air_\eta^{\Phi, D}(p, \nu; \rho)\\
    &=\max_{\rho\in\Delta(\Phi)} \min_{p\in\Delta(\Pi)} \max_{\nu\in \Delta(\Psi)} \\ 
    &\qquad \E_{\pi\sim p}\E_{(M, \pi^\star)\sim \nu}\left[ V_M(\pi^\star) - V_M(\pi) - \frac{1}{\eta}\E_{o\sim M(\cdot|\pi)}\left[\KL(\nu_{\bo{\phi}}(\cdot|\pi, o), \rho)\right] - \frac{1}{\eta} \E_{\phi\sim \rho}\left[\tilD^\pi(\phi\|M)\right]  \right].    \numberthis \label{eq: digdec}
\end{align*}
As both the KL and the $\tilD$ terms in \pref{eq: digdec} are measures of information gain, we call this complexity notion \emph{dual information gain decision-estimation coefficient} (Dig-DEC). In \pref{sec: compare DEC}, we compare in more detail how DigDEC is upper bounded by optimistic DEC --- the complexity achieved by the prior work \citep{foster2024model} in the stochastic setting, and when the improvement can be arbitrarily large. 


\subsection{\postupdate and bounds for $\est$}\label{sec: postupdate sec}
The $\tilD$ we would like to use in \pref{eq: our divergence} depends on the MDP class we consider. Below, we describe two classes of problems that are associated with different choices of $\tilD$, under which the achievable rates for $\est$ are different. 

\subsubsection{Average Estimation Error}
\begin{assumption}[Average estimation error]\label{assum: avg err}
    There exists an estimation function $\ellest_h: \Phi\times \calO\to [-B,B]^{N}$ for every $h$ such that for any $\phi\in\Phi$ and any $M\in\phi$, it holds that for any $\pi\in\Pi$, 
    \begin{align*}
        \E^{\pi, M}[\ellest_h(\phi; o_h)] = 0.    
    \end{align*}
    Additionally, assume that the adversary is restricted such that for any $\pi,\phi$ and $t,t'\in[T]$, it holds that $\E^{\pi, M_t}[\ellest_h(\phi; o_h)]=\E^{\pi, M_{t'}}[\ellest_h(\phi; o_h)]$.
\end{assumption}

The estimation function $\ell$ in \pref{assum: avg err} will be instantiated as the average Bellman error in \pref{lem: av assum reduction} for all concrete examples. In this case, \pref{assum: avg err}  is essentially the standard realizability assumption. We adopt the more general terminology of “estimation error’’ following \cite{du2021bilinear}.

\begin{theorem}
\label{thm: avg est}
    Assume \pref{assum: avg err} holds. Then \pref{alg:general batched} with \pref{alg:MF-Epoch-final} as \postupdate with $\tilD^\pi(\phi\|M) = \tilD_\bi^\pi(\phi\|M) \triangleq \max_{j\in[N]}\frac{1}{B^2H}\sum_{h=1}^H  \left(\E^{\pi, M}\left[\ellest_h(\phi; o_h)_j\right]\right)^2$ ensures 
    \begin{align*}
        \E[\est] \lesssim N\log(|\Phi|)T^\frac{1}{3}. 
    \end{align*}
\end{theorem}


\begin{lemma}\label{lem: av assum reduction}
In the stochastic setting, \pref{assum: function approximation stochastic} implies \pref{assum: avg err} with $N=1$ estimation function
$\ell_{h}(\phi;o_h)=f_{\phi}(s_h,a_h)-r_h-f_{\phi}(s_{h+1})$.
In the hybrid setting, \pref{assum: function approximation hybrid}, \pref{assum: unique mapping} and \pref{assum: known feature} imply
\pref{assum: avg err}
with $N=d$ estimation functions 
$\ell_h(\phi;o_h)_j=f_\phi(s_h,a_h;\bm{e}_j)-\varphi(s_h,a_h)^\top\bm{e}_j - f_{\phi}(s_{h+1};\bm{e}_j)$, where $\bm{e}_j$ as a reward represents the reward function defined as $R(s,a)=\varphi(s,a)_j$. 
\end{lemma}

In order to minimize $\est$ in \pref{eq:minimax-guarantee}, we have to obtain an estimator of $\tilD^{\pi_t}_\bi(\phi\|M^\star)$ for all $\phi$. This can only be achieved via \emph{batching}, which results in the design of \pref{alg:general batched}: In each epoch $k=1,2,\ldots, T/\tau$, the learner uses the same policy $\pi_k$ to interact with the MDP for $\tau$ episodes. While similar epoching mechanism has been proposed in \cite{foster2024model}, our construction of the estimator improves their rate of $\est$ from $\sqrt{T}$ to $T^{\frac{1}{3}}$. To see the difference, consider the case $N=1$ in the stochastic setting, in which the goal is to approximate $\sum_{h=1}^H \left(\E^{\pi_k, M^\star}[\ellest_h(\phi; o_h)]\right)^2$. With observations $(o^1, \ldots, o^\tau)$ drawn from $M^\star(\cdot|\pi_k)$ in epoch $k$, we construct an \emph{unbiased} estimator as $L_k(\phi) =\sum_{h=1}^H\big(\frac{2}{\tau}\sum_{i=1}^{\tau/2} \ellest_h(\phi; o_{h}^i)\big)\big(\frac{2}{\tau}\sum_{i=\tau/2+1}^{\tau} \ellest_h(\phi; o_{h}^i)\big)$, while \cite{foster2024model} constructs a \emph{biased} estimator as $L_k(\phi) =\sum_{h=1}^H\big(\frac{1}{\tau}\sum_{i=1}^{\tau} \ellest_h(\phi; o_{h}^i)\big)^2$. The detail of this estimation procedure is provided in \pref{app: batching improve}.

\subsubsection{Squared Estimation Error}
Under stronger assumptions on the estimation function, we can improve the rate further. 
This is motivated by the class of Bellman-complete MDPs, given as followed. 
\begin{definition}[Bellman completeness for the stochastic setting]\label{def: sto bc} A $\Phi$ satisfying \pref{assum: function approximation stochastic} is Bellman complete under model $M=(P, R)$ if for any $\phi\in\Phi$, there exists an $\phi'\in\Phi$ such that for any $s,a$,  
\begin{align*}
    f_{\phi'}(s,a) = R(s,a)+\E_{s'\sim P(\cdot|s,a)} [f_\phi(s')].  
\end{align*}
A $\Phi$ is Bellman complete if it is Bellman complete under all model $M\in\calM$\footnote{In fact, it suffices to assume Bellman completeness only under the ground-truth model $M^\star$ (as in \cite{foster2024model}). However, it is without loss of generality to assume Bellman completeness under all $M\in\calM$, as one can preprocess the model set $\calM$ by eliminating models under which Bellman completeness does not hold. For simplicity, we assume the latter. Similar for \pref{def: adv bc}. }. 
\end{definition}
\begin{definition}[Bellman completeness for the hybrid setting]\label{def: adv bc} A $\Phi$ satisfying \pref{assum: unique mapping} is Bellman complete under transition $P$ if for any $\phi\in\Phi$, there exists an $\phi'\in\Phi$ such that $\pi_{\phi'}=\pi_\phi$ and for any $s,a,R$,  
\begin{align*}
    f_{\phi'}(s,a; R) = R(s,a) + \E_{s'\sim P(\cdot|s,a)}[f_\phi(s'; R)].  
\end{align*}
A $\Phi$ is Bellman complete if it is Bellman complete under all transition $P\in\calP$. 
\end{definition}

\begin{assumption}\label{assum: squared est err}
   There exists $\err_h: \Phi\times\Phi\times \calO \to [0, B^2]$ for every $h$ and $\calT_M: \Phi\to \Phi$ for every~$M$ such that for any $\phi$ and any $M\in\phi$, it holds that $\phi=\calT_M\phi$. Furthermore, for any $\phi', \phi\in\Phi$, any $M\in\calM$, and any $\pi\in\Pi$,  
   \begin{align*}
       4B^2\cdot \E^{\pi, M}\left[\err_h(\phi', \phi; o_h) - \err_h(\calT_{M}\phi, \phi; o_h)\right] &\geq \E^{\pi, M}\left[\left(\err_h(\phi', \phi; o_h) - \err_h(\calT_{M}\phi, \phi; o_h)\right)^2\right]. 
   \end{align*}
   Additionally, assume that the adversary is restricted such that $\calT_{M_t}\phi = \calT_{M_{t'}}\phi$ for all $\phi$ and all $t,t'\in[T]$. 
\end{assumption}
Similar to \pref{assum: avg err}, the function $\xi$ in \pref{assum: squared est err} will be instantiated as the square Bellman error in \pref{lem: sq assum reduction} for all concrete examples. In this case, \pref{assum: squared est err} corresponds to the standard realizability plus Bellman-completeness assumption.

\begin{theorem}
\label{thm: sq est}
    Assume \pref{assum: squared est err} holds. Then \pref{alg:general} with \pref{alg:MF-BiLevel-final} as \postupdate with $\tilD^\pi(\phi\|M) = \tilD_\sbe^\pi(\phi\|M) \triangleq \frac{1}{B^2H}\sum_{h=1}^H \E^{\pi, M}\left[\err_h(\phi, \phi; o_h) - \err_h(\calT_{M}\phi, \phi; o_h)\right]$ ensures
    \begin{align*}
        \E[\est] \lesssim \log^2|\Phi|. 
    \end{align*}
\end{theorem}

\begin{lemma}\label{lem: sq assum reduction}
In the stochastic setting, \pref{assum: function approximation stochastic} together with Bellman completeness (\pref{def: sto bc}) implies \pref{assum: squared est err} with the estimation function
$\err_{h}(\phi',\phi;o_h)=(f_{\phi'}(s_h,a_h)-r_h-f_{\phi}(s_{h+1}))^2$ and $B^2=1$.  
In the hybrid setting, \pref{assum: function approximation hybrid}, \pref{assum: unique mapping} and \pref{assum: known feature} together with Bellman completeness (\pref{def: adv bc}) imply
\pref{assum: squared est err}
with the estimation function 
$\xi_h(\phi', \phi;o_h)=\|(f_{\phi'}(s_h,a_h;\bm{e}_j)-\varphi(s_h,a_h)^\top\bm{e}_j - f_{\phi}(s_{h+1};\bm{e}_j))_{j\in[d]}\|^2$ and $B^2=d$, where $\bm{e}_j$ as a reward represents the reward function defined as $R(s,a)=\varphi(s,a)_j$. 
\end{lemma}

With \pref{assum: squared est err}, \postupdate no longer needs to rely on batching. We leverage a two-timescale \postupdate learning procedure similar to that of \cite{foster2024model}, which in turn builds on \cite{agarwal2022non}. We refine their approach so $\est$ can be bounded by a constant, improving over \cite{foster2024model}'s $T^{\frac{1}{3}}$ bound. In addition, our approach comes with a simpler regret analysis. Our \postupdate features a two-layer learning structure with a biased loss on the top layer. It is related to model selection algorithms with comparator-dependent second-order bounds (e.g.,  \cite{chen2021impossible}), but also has its special structure not seen in prior work. Thus, we believe it is of independent interest.  The detail of this estimation procedure is provided in \pref{app: bilevel}.  


\section{Applications}\label{sec: stochastic}

By \pref{thm: general thm}, the worst-case regret of \pref{alg:general} is $\sum_t \min_p \max_\nu \air^{\Phi, D}_\eta(p,\nu;\rho_t) + \est/\eta \leq T\mfdec^{\Phi, \tilD}_\eta + \est/\eta$. 
In \pref{sec: postupdate sec}, we provided bounds on $\est$ for two types of $\tilD$, i.e., $\tilD_\bi$ and $\tilD_\sbe$. Below, we provide upper bounds for $\mfdec_\eta^{\Phi, \tilD}$ in concrete settings associated with each $\tilD$.


\subsection{Stochastic Settings}
For the stochastic setting, we consider MDP class $\calM$ and its associated $\Phi$ with bounded bilinear rank \citep{du2021bilinear}, Bellman-eluder dimension \citep{jin2021bellman}, and coverability \citep{xie2022role}. The results are summarized in \pref{tab: summary stochastic}.The on-policy/off-policy in \pref{tab: summary stochastic} should not be confused with the standard on-policy/off-policy training in standard RL. Instead, they are two subclasses of the bilinear class \citep{du2021bilinear} and correspond to the $Q$-type/$V$-type Bellman eluder dimension in \citep{jin2021bellman}. The on-policy case has smaller regret because the executed policies provides sufficient exploration to notice a model missmatch, while in the off-policy case, the learner needs to execute an additional exploration policy for this purpose.




\begin{table}[h]
  \centering
  \caption{Summary of the applications in the stochastic settings. BE stands for MDPs with bounded Bellman-eluder dimensions. Dig-DEC bounds are provided in \pref{app: relate sto to bilinear} for bilinear classes, \pref{app: sto Be} for BE, and \pref{app: related sto coverability} for coverable MDPs. Bilinear classes marked with $\star$ are restricted to estimation function specified in \pref{lem: connecting two digdec},  under which it holds that $\mfdec_\eta^{\Phi, \tilD_{\sbe}}\leq \mfdec_\eta^{\Phi, \tilD_{\bi}}$. $B$ and $N$ are parameters specified in \pref{assum: avg err} or \pref{assum: squared est err}. The regret bound is given by $T\cdot \mfdec^{\Phi, \tilD}_\eta + \est/\eta$ with  $\est$ given in \pref{thm: avg est} or \pref{thm: sq est}, with the optimal $\eta$. } \label{tab: summary stochastic}
  \vspace*{-5pt}
  \renewcommand{\arraystretch}{1.3}
  \begin{tabular}{ |c|c|c|c|c|c|c|c| }
    \hline
    \multicolumn{3}{|c|}{Setting} & \multirow{2}{*}{$\mfdec^{\Phi, \tilD}_\eta$} & \multirow{2}{*}{$\tilD$} & \multirow{2}{*}{$B$} & \multirow{2}{*}{$N$} & \multirow{2}{*}{$\E[\Reg(\pi_{M^\star})]$} \\
    \cline{1-3}
    class & sub-class & completeness & & & & & \\ 
    \hline
    bilinear & on-policy &  & $H^2d\eta$ & $\tilD_{\bi}$ & $1$ & $1$ & $H\sqrt{d\log|\Phi|}T^{\frac{2}{3}}$ \\ \hline
    bilinear & off-policy &  & $\sqrt{H^3d |\calA|^2 \eta}$ & $\tilD_{\bi}$ & $|\calA|$ & $1$ & $H(d|\calA|^2 \log|\Phi|)^{\frac{1}{3}} T^{\frac{7}{9}}$ \\ \hline
    BE & $Q$-type &  & $H^2d\eta$ & $\tilD_{\bi}$ & $1$ & $1$ & $H\sqrt{d\log|\Phi|}T^{\frac{2}{3}}$ \\ \hline
    BE & $V$-type &  & $\sqrt{H^3d|\calA|\eta}$ & $\tilD_{\bi}$ & $1$ & $1$ & $H(d|\calA| \log|\Phi|)^{\frac{1}{3}} T^{\frac{7}{9}}$ \\ \hline
    bilinear$^\star$ & on-policy & \cmark & $H^2d\eta$ & $\tilD_{\sbe}$ & $1$ & -- & $H\sqrt{dT}\log|\Phi|$ \\ \hline
    bilinear$^\star$ & off-policy & \cmark & $\sqrt{H^3d |\calA|^2 \eta}$ & $\tilD_{\sbe}$ & $|\calA|$ & -- & $H(d|\calA|^2 \log^2|\Phi|)^{\frac{1}{3}}T^{\frac{2}{3}}$ \\ \hline
    BE & $Q$-type & \cmark & $H^2d\eta$ & $\tilD_{\sbe}$ & $1$ & -- &  $H\sqrt{dT}\log|\Phi|$\\ \hline
    BE & $V$-type & \cmark & $\sqrt{H^3d|\calA|\eta}$ & $\tilD_{\sbe}$ & $1$ & -- & $H(d|\calA| \log^2|\Phi|)^{\frac{1}{3}}T^{\frac{2}{3}}$ \\ \hline
    coverable & -- & \cmark & $H^2 d \eta$ & $\tilD_{\sbe}$ & $1$ & -- & $H\sqrt{dT}\log|\Phi|$ \\ \hline
  \end{tabular}
\end{table}

We remark without giving details that in the stochastic setting, we can achieve same results in \pref{tab: summary stochastic} \emph{with high-probability} if we replace the $\E_{M\sim \nu} \E_{o\sim M(\cdot|\pi)}[\KL(\nu_{\bo{\phi}}(\cdot|\pi, o), \rho_t)]$ term by $\KL(\nu_{\bo{\phi}}, \rho_t)$ in the definition of $D$ in \pref{eq: our divergence}. This variant, however, cannot handle the hybrid setting.  

\subsection{Hybrid Settings}
\label{sec:hybrid}
For the hybrid setting, with known linear reward feature, we consider transition structure including hybrid bilinear classes \citep{liu2025decision} and coverability \citep{xie2022role}. While it is possible to also extend Bellman-eluder dimension to the hybrid setting, we omit it for simplicity. 

\begin{table}[h]
  \centering
  \caption{Summary of the applications in the hybrid settings. Dig-DEC bounds are provided in \pref{app: hybrid bilinear } for hybrid bilinear classes and \pref{app: coverability hybrid} for coverable MDPs. Bilinear classes marked with $\star$ are restricted to estimation function specified in \pref{lem: connecting two digdec hybrid}, under which it holds that $\mfdec_\eta^{\Phi, \tilD_{\sbe}}\leq \mfdec_\eta^{\Phi, \tilD_{\bi}}$. } 
  \vspace*{-5pt}
  \renewcommand{\arraystretch}{1.3}
  \begin{tabular}{ |c|c|c|c|c|c|c|c| }
    \hline
    \multicolumn{3}{|c|}{Setting} & \multirow{2}{*}{$\mfdec^{\Phi, \tilD}_\eta$} & \multirow{2}{*}{$\tilD$} & \multirow{2}{*}{$B$} & \multirow{2}{*}{$N$} & \multirow{2}{*}{$\E[\Reg(\pi_{\phi^\star})]$} \\
    \cline{1-3}
    class & sub-class & completeness & & & & & \\ 
    \hline
    bilinear & on-policy &  & $(H^5d^3\eta)^{\frac{1}{3}}$ & $\tilD_{\bi}$ & $1$ & $d$ & $d(H^5\log|\Phi|)^{\frac{1}{4}}T^{\frac{5}{6}}$ \\ \hline
    bilinear & off-policy &  & $(H^6 d^3 |\calA|^2 \eta)^{\frac{1}{4}}$ & $\tilD_{\bi}$ & $|\calA|$ & $d$ & $( H^6  d^4 |\calA|^2 \log|\Phi|)^{\frac{1}{5}}T^{\frac{13}{15}}$ \\ \hline
    bilinear$^\star$ & on-policy & \cmark & $(H^5 d^4 \eta)^{\frac{1}{3}}$ & $\tilD_{\sbe}$ & $\sqrt{d}$ & -- & $d(H^5 \log^2|\Phi|)^{\frac{1}{4}}T^{\frac{3}{4}}$ \\ \hline
    bilinear$^\star$ & off-policy & \cmark & $(H^6 d^4 |\calA|^2 \eta)^{\frac{1}{4}}$ & $\tilD_{\sbe}$ & $\sqrt{d}|\calA|$ & -- & $(H^6 d^4 |\calA|^2  \log^2|\Phi|)^{\frac{1}{5}}T^{\frac{4}{5}}$ \\ \hline
    coverable & -- & \cmark & $(H^5 d^4 \eta)^{\frac{1}{3}}$ & $\tilD_{\sbe}$ & $\sqrt{d}$ & -- & $d(H^5 \log^2|\Phi|)^{\frac{1}{4}}T^{\frac{3}{4}}$ \\ \hline
  \end{tabular}
\end{table}

\section{Comparison with Prior Complexities in Stochastic MDPs}\label{sec: compare DEC} 
Compared with $\mfdec_\eta^{\Phi, \tilD}$ in \pref{eq: digdec} achieved by our algorithm, the complexity of optimistic E2D  \citep{foster2024model} defined for the stochastic setting is
\begin{align}
    \odec_\eta^{\Phi, \tilD} = \max_{\rho\in\Delta(\Phi)} \min_{p\in\Delta(\Pi)} \max_{\nu\in \Delta(\Psi)} \E_{\pi\sim p}\E_{M\sim \nu}\E_{\phi \sim \rho}\left[ V_\phi(\pi_\phi) - V_M(\pi) - \frac{1}{\eta} \tilD^\pi(\phi\|M)\right] \label{eq: odec high level}  
\end{align}
for the same choices of $\tilD$. 
Another model-free DEC in \cite{liu2025decision} is 
\begin{align*}
    \dec_\eta^\Phi = \max_{\rho\in\Delta(\Phi)} \min_{p\in\Delta(\Pi)} \max_{\nu\in \Delta(\Psi)}  \E_{\pi\sim p}\E_{(M,\pi^\star)\sim \nu}\left[ V_M(\pi^\star) - V_M(\pi) - \frac{1}{\eta}\E_{o\sim M(\cdot|\pi)}\left[\KL(\nu_{\bo{\phi}}(\cdot|\pi, o), \rho)\right] \right].   
\end{align*}
It is clear that $\mfdec_\eta^{\Phi,\tilD}\leq \dec_\eta^{\Phi}$ for any non-negative divergence $\tilD$. Furthermore, we have 

\begin{theorem}\label{thm: improvement}
    In the stochastic setting, $\mfdec_\eta^{\Phi,\tilD}\leq \odec_\eta^{\Phi,\tilD} + \eta$ for any $\tilD$.  
\end{theorem}
Since DECs with parameter $\eta$ is usually of order $(\eta d)^\alpha$ for some intrinsic dimension $d$ and exponent $\alpha\leq 1$, \pref{thm: improvement} implies that for any setting that can be handled by optimistic E2D with a certain $\tilD$, it can also be covered by our algorithm with the same $\tilD$. Compared to optimistic DEC (\pref{eq: odec high level}), \decname (\pref{eq: digdec}) has an extra KL term $\E_{\pi\sim p}\E_{M\sim \nu}\E_{o\sim M(\cdot|\pi)}[\KL(\nu_{\bo{\phi}}(\cdot|\pi,o), \rho)]$ that can be further decomposed into two terms $\KL(\nu_{\bo{\phi}}, \rho) + \E_{\pi\sim p}\E_{M\sim \nu}\E_{o\sim M(\cdot|\pi)}[\KL(\nu_{\bo{\phi}}(\cdot|\pi,o), \nu_{\bo{\phi}})]$. They have different purposes: The first term $\KL(\nu_{\bo{\phi}}, \rho)$ is for \emph{regularization}, which makes the marginal distribution of $\nu$ not overly distant from $\rho$. This is the key that allows us to avoid the optimism mechanism in \cite{foster2024model} (i.e., the $V_\phi(\pi_\phi)$ in \pref{eq: odec high level}). We remark that by \emph{regularization only}, we can \emph{recover} the bounds achieved by optimistic DEC in the stochastic setting (this can be seen from the proof of \pref{thm: improvement}), though it is unclear whether it can give \emph{strict improvement}. However, the removal of optimism turns out to be important in the hybrid setting (\pref{sec:hybrid}) as it avoids explicit construction of the reward estimator. The second term $\E_{\pi\sim p}\E_{M\sim \nu}\E_{o\sim M(\cdot|\pi)}[\KL(\nu_{\bo{\phi}}(\cdot|\pi,o), \nu_{\bo{\phi}})]$ is an \emph{information gain} that allows Dig-DEC to \emph{strictly improve} over optimistic DEC even in the stochastic setting. This is because all common choices of $\tilD$ such as bilinear divergence and squared Bellman error are mean-based and ignore distributional differences, and the KL information gain term can capture them. 
We give a toy example in the next theorem to show this, with a detailed proof provided in \pref{app:compare DEC}.  
\begin{theorem}\label{thm:lower bound}
There exists a $3$-armed bandit instance where for any $T\geq 1$ and $\eta \le 1$, the algorithm in \cite{foster2024model} suffers $\max_a \E[\Reg(a)] \ge \Omega(\sqrt{T})$, while our algorithm achieves $\max_a \E[\Reg(a)] \leq 1$. 
\end{theorem}

\section{Conclusion}
We introduced a new model-free DEC approach that removes optimism in prior work and incorporates two information-gain terms into the AIR objective for decision making. In addition, we refined the online function estimation procedure. Together, they yield improved regret bounds in the stochastic setting and establish the first regret bounds for model-free learning in hybrid MDPs with bandit feedback. Future directions include relaxing \pref{assum: unique mapping} and \pref{assum: known feature}, and investigating the fundamental limits of model-free learning.
\clearpage
\bibliographystyle{apalike}
\bibliography{ref.bib}

\clearpage
\appendix

\appendixpage

{
\startcontents[section]
\printcontents[section]{l}{1}{\setcounter{tocdepth}{2}}
}


\newpage

\section{Regret Bound Comparison with Previous Work}
\label{app: comparison tables}

\begin{table}[H]
  \centering
  \caption{Regret for model-free learning in stochastic MDPs
           (only showing $T$ dependence). ``Toy 3-arm'' is defined in
           \pref{thm:lower bound}. The two bounds in the same cell correspond to the cases with on-policy and off-policy estimation. }   \label{tab: sto}
  \vspace*{-5pt}
  \hspace*{-25pt}
  \begin{tabular}{ |c|c|c|c|c| }
    \hline
    Algorithm & Bilinear or BE & \makecell{\{Bilinear or BE or Coverable\}\\ + Bellman Complete + On-Policy}& Toy 3-arm & Exploration Mechanism \\ \hline
    \makecell{\cite{du2021bilinear} \\ \cite{jin2021bellman} \\ \cite{xie2022role}} & $T^{\frac{2}{3}}/T^{\frac{2}{3}}$
      & $\sqrt{T}$ & $\sqrt{T}$ & optimism \\ \hline
    \cite{foster2024model} & $T^{\frac{3}{4}}/T^{\frac{5}{6}}$
      & $T^{\frac{2}{3}}$ & $\sqrt{T}$ & information gain + optimism \\ \hline
    Ours & $T^{\frac{2}{3}}/T^{\frac{7}{9}}$ & $\sqrt{T}$ & 1
      & information gain \\ \hline
  \end{tabular}

  \vspace{1.5em}  

  \captionof{table}{Regret for learning in hybrid MDPs (stochastic transition and adversarial reward). The model-free learning guarantees in \cite{liu2025decision} and our work cannot handle general reward but rely on \pref{assum: known feature}. }
  \label{tab: adv}
  \vspace*{3pt}
  \hspace*{-23pt}
  \begin{tabular}{ |c|c|c|c|c|c| }
    \hline
    Algorithm & Bilinear & \makecell{ \{Bilinear or Coverable\}\\+ Bellman Complete + On-Policy} & Model-Free & Bandit Feedback & General Reward \\ \hline
    \cite{liu2025decision} & $\sqrt{T}/T^{\frac{2}{3}}$
      & $\sqrt{T}$ & \xmark & \cmark & \cmark \\ \hline
    \cite{liu2025decision} & $T^{\frac{3}{4}}/T^{\frac{5}{6}}$
      & -- &  \cmark & \xmark  & \xmark \\ \hline
    Ours & $T^{\frac{5}{6}}/T^{\frac{13}{15}}$
      & $T^{\frac{3}{4}}$ & \cmark & \cmark & \xmark  \\ \hline
  \end{tabular}
\end{table}

\section{Partitioning over $\calP\times \calR\times \Pi$ for hybrid MDPs}\label{app: illustrate}
\begin{figure}[H]
\centering
\includegraphics[width=0.3\textwidth]{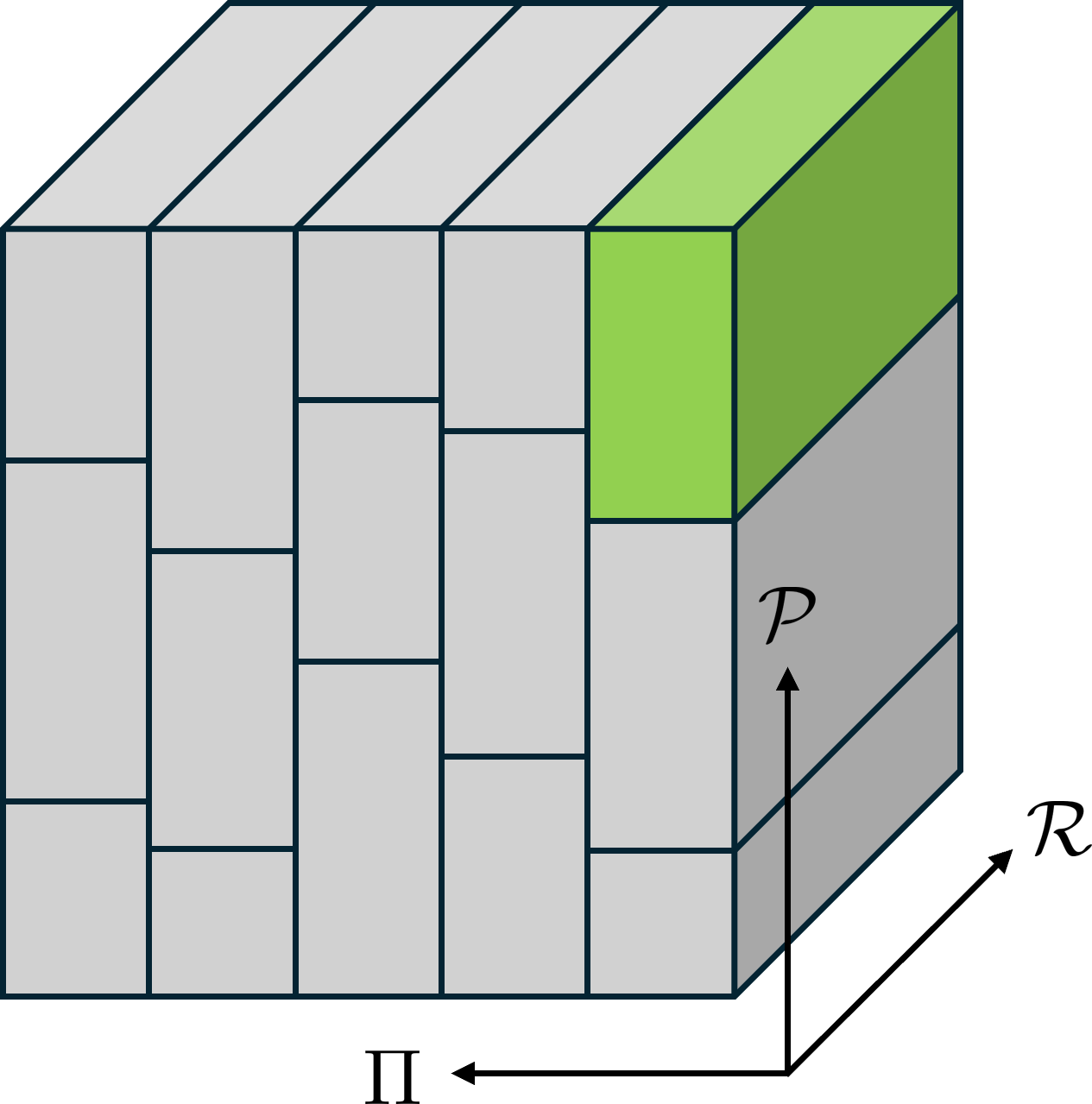}
\caption{Partitioning for hybrid MDPs}\label{fig: illustrate}
\end{figure}

\pref{fig: illustrate} illustrates the partition scheme over $\calM\times \Pi=\calP\times \calR\times \Pi$ described in \pref{assum: function approximation hybrid}. Each infoset $\phi$ (represented by the green block in \pref{fig: illustrate}) is associated with single policy $\pi_\phi$, a subset of transitions, and all reward functions. 
As shown in \pref{fig: illustrate}, the partition over the $\calP$ space could be different for different $\pi$.

\newpage
\section{Omitted Details in \pref{sec:pre}} \label{app: recovering}
In this section, we show that the algorithms in \cite{xu2023bayesian} and \cite{liu2025decision} are special cases of \pref{alg:general}. 

\subsection{Recovering \pref{thm: air phi}} The decision rule of \cite{liu2025decision}'s algorithm corresponds to \pref{eq:minmax} with $D^\pi(\nu\|\rho) = \E_{M\sim \nu}\E_{o\sim M(\cdot|\pi)}[\KL(\nu_{\bo{\phi}}(\cdot|\pi,o), \rho)]$. It can be shown that $\breg_{D^\pi(\cdot\|\rho)}(\nu, \nu')=\E_{M \sim \nu}\E_{o \sim M(\cdot|\pi)}\big[\KL(\nu_{\bo{\phi}}(\cdot|\pi, o), \nu'_{\bo{\phi}}(\cdot|\pi, o))\big]$ in this case. Furthermore, notice that when $\nu=\delta_{M_t, \picirc}$, we have $\nu_{\bo{\phi}}(\cdot|\pi, o)=\delta_{\phi^\star}$ according to \pref{def: restricted env}. Thus, the estimation error term in \pref{eq:minimax-guarantee} in \cite{liu2025decision}'s algorithm is 
\begin{align*}
    \E[\est] &= \E\left[\sum_{t=1}^T \left( \KL(\delta_{\phi^\star}, \rho_t) - \E_{o\sim M_t(\cdot|\pi_t)} \Big[\KL(\delta_{\phi^\star},  (\nu_t)_{\bo{\phi}}(\cdot|\pi_t,o))\Big]\right) \right] \\
    &= \E\left[\sum_{t=1}^T \left( \KL(\delta_{\phi^\star}, \rho_t) - \KL(\delta_{\phi^\star},  (\nu_t)_{\bo{\phi}}(\cdot|\pi_t,o_t))\right) \right] = \E\left[ \sum_{t=1}^T \log\frac{\nu_t(\phi^\star|\pi_t,o_t)}{\rho_t(\phi^\star)} \right],   
\end{align*}
where in the second equality we use that $o_t$ is drawn from $M_t(\cdot|\pi_t)$. 
Thus, by letting $\rho_{t+1}(\phi) = \nu_t(\phi|\pi_t,o_t)$, their algorithm achieves $\E[\est] = \E\left[\sum_{t=1}^T \log\frac{\rho_{t+1}(\phi^\star)}{\rho_t(\phi^\star)} \right]\leq \log\frac{1}{\rho_1(\phi^\star)}=\log|\Phi|$. 
Using this in \pref{eq:minimax-guarantee} proves \pref{thm: air phi}. The results of \cite{xu2023bayesian} can also be recovered as they are special cases of \cite{liu2025decision}. 

\subsection{Recovering results for adversarial MDP with full-information feedback \citep{liu2025decision}}  
For learning with full information feedback in the adversarial MDPs, the learner can observe the full reward function at the end of each episode. In other words, at episode $t$, the reward function $R_t: \calS\times \calA\to [0,1]$ is part of the observation $o_t$. In this setting, the $\log|\Pi|$ dependence in the regret bound can be improved to $\log|\calA|$. To achieve this, \cite{liu2025decision} designed a two-level algorithm and define a new notion called InfoAIR. We can recover this result by instantiating our \pref{alg:general} with $\Phi=\{\phi_{P,(a_s)_{s\in\calS}}:~ P\in\calP, a_s \in\calA, \forall s\in\calS\}$ where $\phi_{P,(a_s)_{s\in\calS}} = \{((P, R), \pi^\star): R\in\calR, \pi^\star=(a_s)_{s\in\calS}\}$, that is, partitioning $\calM\times \Pi$ according to the transition kernel and the actions taken by the policy on all states.  
Then define 
\begin{align*}
   D^\pi(\nu\| \rho) = \E_{(P,R,\pi^\star) \sim \nu}\E_{o \sim M_{P,R}(\cdot|\pi)}\E_{s \sim d^{\pi, P}}\left[\KL(\nu_{\bo{a}_s, \bo{P}}(\cdot|\pi, o), \rho_{\bo{a}_s, \bo{P}})\right], 
\end{align*}
where $M_{P, R}$ denotes the MDP model with transition kernel $P$ and reward function $R$, and $\rho_{\bo{a}_s, \bo{P}}$ denotes $\rho$'s marginal distribution over $(a_s, P)$ following our notational convention. Finally, update the posterior as 
$\rho_{t+1} = \argmin_{\rho}\sum_{s \in \calS}\KL\left(\rho_{\bo{a}_s, \bo{P}}, \nu_{\bo{a}_s, \bo{P}}(\cdot|\pi_t, o_t)\right)$. This recovers the same regret bound as in \cite{liu2025decision} without the need for the two-level design. 
We also note that the analysis for this result requires our new proof strategy in \pref{eq:general-decom}, as the $D^\pi(\nu\|\rho)$ here is not strictly convex in $\nu$ and the previous proof \cite{xu2023bayesian, liu2025decision} cannot be applied. 


\newpage
\section{Concentration Inequality}
\begin{lemma}[Freedman's inequality \citep{beygelzimer2011contextual}]\label{lem: strengthen freed}
    Let $X_1, X_2,\ldots$ be a martingale difference sequence with respect to a filtration $\mathfrak{F}_1 \subset \mathfrak{F}_2 \subset \cdots$ such that $\E[X_t|\mathfrak{F}_t]=0$ and assume $X_t\leq B$ almost surely. Then for any $\alpha\geq B$, with probability at least $1-\delta$, 
    \begin{align*}
        \sum_{t=1}^T X_t \leq \frac{1}{\alpha} \sum_{t=1}^T \E[X_t^2|\mathfrak{F}_t] + \alpha \log (1/\delta).  \numberthis \label{eq: to substitutate}
    \end{align*}
\end{lemma}

\begin{lemma}[Empirical Freedman's inequality]\label{lem: empirical freed}
    Let $X_1, X_2,\ldots$ be a sequence with respect to a filtration $\mathfrak{F}_1 \subset \mathfrak{F}_2 \subset \cdots$ such that $\E[X_t|\mathfrak{F}_t]=\mu_t$ and assume $\max\{X_t-\mu_t,X_t\}\leq B$ almost surely. Then for any $\alpha\geq 4B$, with probability at least $1-\delta$, 
    \begin{align*}
        \sum_{t=1}^T (\mu_t - X_t) \leq \frac{4}{\alpha} \sum_{t=1}^T X_t^2 + \alpha \log (1/\delta).  \numberthis \label{eq: to empirical substitutate}
    \end{align*}
\end{lemma}
\begin{proof}
    Denote $\E_t[\cdot]=\E[\cdot~|~\mathfrak{F}_t]$. We have at any time step
    \begin{align*}
    &\E_{t}\left[\exp\left(\frac{1}{\alpha}(\mu_t - X_t) - \frac{4}{\alpha^2}X_t^2\right)\right]
    \\&\leq \E_{t}\left[1+ \frac{1}{\alpha}(\mu_t - X_t) - \frac{4}{\alpha^2}X_t^2 +\left(\frac{1}{\alpha}(\mu_t - X_t) - \frac{4}{\alpha^2}X_t^2\right)^2\right]
    \\&\leq 1+\E_{t}\left[ - \frac{4}{\alpha^2}X_t^2 +\frac{2}{\alpha^2}((\mu_t - X_t)^2+X_t^2)\right]\leq 1. 
    \end{align*}
    Markov inequality finishes the proof.
\end{proof}

\begin{lemma}\label{lem: useful inequali}
    Let $(X_1, Y_1), (X_2, Y_2)\ldots$ be a sequence with respect to a filtration $\mathfrak{F}_1 \subset \mathfrak{F}_2 \subset \cdots$ such that $|X_t|\leq B$ and $0\leq Y_t\leq B$ almost surely. Furthermore, $\E[X_t|\mathfrak{F}_t]\geq \E[Y_t|\mathfrak{F}_t] $ and $B\E[X_t|\mathfrak{F}_t]\geq \E[X_t^2|\mathfrak{F}_t]$. Then with probability at least $1-\delta$, 
    \begin{align}
        \frac{1}{2}\sum_{t=1}^T \E[X_t|\mathfrak{F}_t] \leq \sum_{t=1}^T \left(X_t - \frac{1}{4}Y_t\right) + 9B\log(1/\delta).\label{eq: useful 11}  
    \end{align}
    Also, with probability at least $1-\delta$, 
    \begin{align}
        \frac{1}{2}\sum_{t=1}^T X_t \leq \sum_{t=1}^T \left(X_t - \frac{1}{4}Y_t\right) + 9B\log(1/\delta).\label{eq: useful 22}   
    \end{align}
    
\end{lemma}
\begin{proof}
    Denote $\E_t[\cdot]=\E[\cdot~|~\mathfrak{F}_t]$. Let $c\in[\frac{1}{2},1]$ be a fixed constant, and define $Z_t = cX_t - \frac{1}{4}Y_t$. Applying \pref{lem: strengthen freed} with $\alpha=9B$ gives 
    \begin{align*}
        \sum_{t=1}^T  \left(\E_t[Z_t] - Z_t\right) 
        &\leq \frac{1}{9B}\sum_{t=1}^T  \E_t[(\E_t[Z_t] - Z_t)^2] + 9B\log(1/\delta) \\
        &\leq \frac{1}{9B}\sum_{t=1}^T  \E_t[Z_t^2] + 9B\log(1/\delta) \\
        &\leq \frac{1}{9B} \sum_{t=1}^T \left(2c^2\E_t[X_t^2] + \frac{2}{16}\E_t[Y_t^2] \right) + 9B\log(1/\delta) \\
        &\leq \frac{1}{9} \sum_{t=1}^T \left(2c^2\E_t[X_t] + \frac{2}{16}\E_t[X_t] \right) + 9\log(1/\delta) \tag{$\E_t[Y_t^2]\leq B\E_t[Y_t]$ because $Y_t\in[0,B]$}
    \end{align*}
    Rearranging:  
    \begin{align}
       \sum_{t=1}^T \E_t\left[Z_t - \left(\frac{2c^2}{9} + \frac{1}{72}\right)X_t\right]\leq \sum_{t=1}^T Z_t + 9B\log(1/\delta). \label{eq: useful44}
    \end{align}
    To prove \pref{eq: useful 11}, let $c=1$, which gives 
    $\E_t\left[Z_t - \left(\frac{2c^2}{9} + \frac{1}{72}\right)X_t\right]= \E_t\left[X_t - \frac{1}{4}Y_t - \frac{17}{72} X_t\right] \geq \frac{1}{2}\E_t[X_t]  $. Combining this with \pref{eq: useful44} proves \pref{eq: useful 11}. To prove \pref{eq: useful 22}, let $c=\frac{1}{2}$. which gives  $\E_t\left[Z_t - \left(\frac{2c^2}{9} + \frac{1}{72}\right)X_t\right]= \E_t\left[\frac{1}{2}X_t - \frac{1}{4}Y_t - \frac{5}{72} X_t\right] \geq 0$. Combining this with \pref{eq: useful44} and rearranging proves \pref{eq: useful 22}. 
\end{proof}

\newpage
\section{Mirror Descent}

\begin{lemma}[First-order optimality condition]\label{lem: bregman}
   For any concave and differentiable function $F$, if $\nu'\in\argmax_{\nu\in\Omega} F(\nu)$ for some convex set $\Omega$, then $F(\nu)\leq F(\nu') - \breg_{(-F)}(\nu, \nu')$ for any $\nu\in\Omega$. 
\end{lemma}
\begin{proof}
Define $G = -F$. Then $G$ is convex and $\nu'\in\argmin_{\nu'} G(\nu')$. We have by the definition of Bregman divergence $\breg_G(\nu, \nu')= G(\nu) - G(\nu') - \langle \nabla G(\nu'), \nu-\nu'\rangle$, and first-order optimality condition $\langle \nabla G(\nu'), \nu-\nu'\rangle \ge 0$. Thus, $G(\nu) \ge G(\nu') + \breg_G(\nu, \nu')$, which is equivalent to $F(\nu) \le F(\nu') + \breg_{(-F)}(\nu, \nu')$.
\end{proof}

\begin{lemma}\label{lem:stability}
    Let $g: \Phi\to [-1,1]$ be any function and let $\nu, \rho\in\Delta(\Phi)$. Then for any $\eta>0$, 
    \begin{align*}
        \E_{\phi\sim \nu}[g(\phi)] - \E_{\phi\sim \rho}[g(\phi)] - \frac{1}{\eta} \KL(\nu, \rho) \leq \eta. 
    \end{align*}
\end{lemma}
\begin{proof}
    \begin{align*}
        \E_{\phi\sim \nu}[g(\phi)] - \E_{\phi\sim \rho}[g(\phi)] \leq 2D_{\text{TV}}(\nu, \rho) \leq 2\sqrt{\KL(\nu, \rho)}\leq \frac{1}{\eta}\KL(\nu, \rho) + \eta, 
    \end{align*}
    where we use Pinsker's inequality and AM-GM inequality. 
\end{proof}

\begin{lemma}[Mirror descent with auxiliary terms]\label{lem: mirror descent}
Let $F_t$ be a convex function over $\Delta_N$, and let $\ell_t, b_t\in\mathbb{R}^N$ with $\ell_t^2$ denoting $(\ell_t(1)^2, \ldots, \ell_t(N)^2)$. Then the update $p_1=\frac{1}{N}\mathbf{1}$ and 
    \begin{align*}
        p_{t+1} = \argmin_{p\in\Delta_N} \Big\{ 
  \inner{p, \ell_t + 4\gamma \ell_t^2 + b_t} + F_t(p) + \frac{1}{\gamma}\KL(p, p_t)\Big\}
    \end{align*}
    with $\gamma |\ell_t(i)|\leq \frac{1}{16}$ and $0\leq \gamma b_t(i)\leq \frac{1}{4}$ for all $i\in[N]$  ensures for any $p^\star\in\Delta_N$, 
    \begin{align*}
       &\sum_{t=1}^T \inner{p_t,\ell_t} \\
       &\leq \frac{\log N}{\gamma} + \sum_{t=1}^T \Big(\inner{p^\star, \ell_t + 4\gamma\ell_t^2} + \inner{p^\star, b_t} - \frac{1}{2}\inner{p_t,b_t} + F_t(p^\star) - F_t(p_{t+1}) - \breg_{F_t}(p^\star, p_{t+1})\Big). 
    \end{align*}
    
\end{lemma}
\begin{proof}
    By \pref{lem: bregman}, 
    \begin{align*}
        &\inner{p_{t+1}, \ell_t + 4\gamma\ell_t^2 + b_t} + F_t(p_{t+1}) + \frac{1}{\gamma}\KL(p_{t+1}, p_t) \\
        &\leq \inner{p^\star, \ell_t + 4\gamma\ell_t^2 + b_t} + F_t(p^\star) + \frac{1}{\gamma}\KL(p^\star, p_t) - \breg_{F_t}(p^\star, p_{t+1}) - \frac{1}{\gamma}\KL(p^\star, p_{t+1}). 
    \end{align*}
    Rearranging gives 
    \begin{align*}
        &\inner{p_t, \ell_t + 4\gamma\ell_t^2} \\
        &\leq \inner{p^\star, \ell_t + 4\gamma \ell_t^2} +  \inner{p_t-p_{t+1}, \ell_t + 4\gamma\ell_t^2 + b_t} - \frac{1}{\gamma} \KL(p_{t+1}, p_t) \\
        &\quad  + \inner{p^\star - p_{t}, b_t} + \frac{\KL(p^\star, p_t) - \KL(p^\star, p_{t+1})}{\gamma} + F_t(p^\star) - F_t(p_{t+1}) - \breg_{F_t}(p^\star, p_{t+1}).   \label{eq: apply here}  \numberthis
    \end{align*}
    Since $\gamma |\ell_t(i) + 4\gamma\ell_t(i)^2 + b_t(i)|\leq \frac{1}{16} + 4\times (\frac{1}{16})^2 + \frac{1}{4}\leq 1$, by \pref{lem:stability} we have 
    \begin{align*}
        &\inner{p_t-p_{t+1}, \ell_t + 4\gamma\ell_t^2 + b_t} - \frac{1}{\gamma} \KL(p_{t+1}, p_t) \\
        &\leq \gamma\inner{p_t, (\ell_t + 4\gamma\ell_t^2 + b_t)^2} \\
        &\leq 2\gamma \inner{p_t, ({\textstyle \frac{5}{4}}\ell_t)^2} + 2\gamma\inner{p_t, b_t^2} \\
        &\leq \inner{p_t, 4\gamma\ell_t^2} + \frac{1}{2}\inner{p_t, b_t}. 
    \end{align*}
    Using this in \pref{eq: apply here} we get 
    \begin{align*}
        \inner{p_t, \ell_t} 
        &\leq \inner{p^\star, \ell_t + 4\gamma\ell_t^2} + \inner{p^\star, b_t} - \frac{1}{2}\inner{p_t, b_t} \\
        &\qquad + \frac{\KL(p^\star, p_t) - \KL(p^\star, p_{t+1})}{\gamma} + F_t(p^\star) - F_t(p_{t+1}) - \breg_{F_t}(p^\star, p_{t+1}). 
    \end{align*}
    Summing over $t$ gives the desired inequality. 
\end{proof}

\newpage
\section{Estimation Procedures}
We present the choices of \postupdate as standalone online learning algorithms because they might be of independent interest.

\subsection{Average estimation error minimization via batching}\label{app: batching improve}

\begin{algorithm}[H]
    \caption{Epoch-based learning algorithm for average estimation error}
    \label{alg:MF-Epoch-final}
    \textbf{Input}: An estimation function $\ell_h: \Phi\times \calO\to [-B, B]^N$ satisfying \pref{assum: avg err}. \\
    \textbf{Parameter:} $\tau = T^{\frac{1}{3}}, \beta = 7\tau N\iota$, $\gamma = \frac{1}{2\beta}$, $\iota = \log(12NKH/\delta)$. \\
    \For{$k=1, 2, \ldots, K$}{
         Receive observations $o_t\sim M_t(\cdot|\pi_k)$ for all 
         $t\in\calI_k = \{(k-1)\tau +1, \ldots, k\tau\}$. \\ Split $\calI_k$ into two sub-intervals of equal size: 
         \begin{align*}
             \calI_k^{\leftt} = \{(k-1)\tau +1, \ldots, (k-1)\tau +\tfrac{\tau}{2}\}  \quad \text{and} \quad \calI_k^{\rightt} = \{(k-1)\tau + \tfrac{\tau}{2} +1 , \ldots, k\tau\}. 
         \end{align*}
         Define for all $j\in[N]$,  
         \begin{align*}
             L_k(\phi)_j &=\frac{\tau}{B^2 H}\sum_{h=1}^H \left(\frac{1}{|\calI_k^\leftt|}\sum_{t\in\calI_k^\leftt} \ell_{h}(\phi; o_{t,h})_j\right)\left(\frac{1}{|\calI_k^\rightt|}\sum_{t\in\calI_k^\rightt} \ell_h(\phi; o_{t,h})_j\right)\,,\quad L_k(\phi) = \sum_{j=1}^N L_k(\phi)_{j}. 
         \end{align*}

         Let $(F_t)_{t\in\calI_k}: \Delta(\Phi)\to \mathbb{R}$ be convex functions. Calculate 
        \begin{align*}
             &\rho_{k+1} = \argmin_{\rho\in \Delta(\Phi)}\left\{\inner{\rho,  L_k + (4\gamma+2\beta^{-1}) L_k^2} + \sum_{t\in\calI_k}F_t(\rho) + \frac{1}{\gamma}\KL(\rho, \rho_k) \right\}.   \label{eq: epoch level update} \numberthis   
        \end{align*}
    }
\end{algorithm}
\begin{lemma}\label{lem: est part1}
With probability at least $1-\delta/3$,  
\pref{alg:MF-Epoch-final} satisfies
\begin{align*}
    \frac{1}{B^2H}\sum_{k=1}^K \sum_\phi \rho_k(\phi) \sum_{t\in\calI_k} \max_{j\in[N]} \sum_{h=1}^H \left(\E^{\pi_k, M_t}[\ell_h(\phi; o_{h})_j]  \right)^2 \leq \sum_{k=1}^K\sum_{\phi}\rho_{k}(\phi)\left(L_k(\phi) + \frac{1}{\beta}L_k(\phi)^2\right)+4\beta\log(3/\delta)\,.
\end{align*}
\end{lemma}
\begin{proof}
By \pref{assum: avg err}, for any $t,t'\in\calI_k$ it holds that
\begin{align*}
    \E^{\pi_k, M_t}\left[ \ell_h(\phi; o_{h}) \right] = \E^{\pi_k, M_{t'}}\left[ \ell_h(\phi; o_{h}) \right]. 
\end{align*}
We denote $\bar\ell_{k,h}(\phi) = \E^{\pi_k, M_t}[\ell_h(\phi; o_{h})]$ for any $t\in\calI_k$. 

Clearly, the left-hand side of the desired inequality is upper bounded by  
\begin{align*}
    \frac{1}{B^2H}\sum_{k=1}^K \sum_\phi \rho_k(\phi) \sum_{t\in\calI_k} \sum_{j=1}^N \sum_{h=1}^H \left(\E^{\pi_k, M_t}[\ell_h(\phi; o_{h})_j]  \right)^2 = \frac{\tau}{B^2H} \sum_{k=1}^K\sum_{\phi}\rho_{k}(\phi)\sum_{j=1}^N \sum_{h=1}^H\bar\ell_{k,h}(\phi)^2_j 
\end{align*}
By construction, $\E_{k}[L_k(\phi)]=\frac{\tau}{B^2H}\sum_{j=1 }^N\sum_{h=1}^H\bar\ell_{k,h}(\phi)^2_j$ due to the conditional independence of the observations. Furthermore, we have $L_k(\phi)\in[-\tau N,\tau N]$. Therefore, we can use \pref{lem: empirical freed} on the sequence $X_k=-\sum_{\phi}\rho_k(\phi)L_k(\phi)$ with $\beta \geq 7\tau N$: 
    \begin{align*}
        \frac{\tau}{B^2H}\sum_{k=1}^K\sum_{\phi}\rho_{k}(\phi)\sum_{j=1}^N\sum_{h=1}^H\bar\ell_{k,h}(\phi)^2_j\leq\sum_{k=1}^K\sum_{\phi}\rho_{k}(\phi)\left(L_k(\phi) + \frac{1}{\beta}L_k(\phi)^2\right)+4\beta\log(3/\delta)\,.
    \end{align*}
\end{proof}
\begin{lemma}\label{lem: est part2}
    With probability at least $1-\delta/3$,
    \begin{align*}
        \sum_{k=1}^KL_k(\phi^\star)^2 \leq KN^2\log^2(12NKH/\delta). 
    \end{align*}
\end{lemma}
\begin{proof}
By \pref{assum: avg err} and \pref{lem: strengthen freed}, for any $j,k,h$, we have with probability $1-\delta$, 
\begin{align*}
&\left|\sum_{t\in\calI_k^\leftt} \ell_{h}(\phi^\star, o_{t,h})_j\right| = \left|\sum_{t\in\calI_k^\leftt} \ell_{h}(\phi^\star, o_{t,h})_j - \sum_{t\in\calI_k^\leftt} \E^{\pi_k, M_t}[\ell_h(\phi^\star; o_h)] \right| \leq B\sqrt{\tau \log(12/\delta)}\\
&\left|\sum_{t\in\calI_k^\rightt} \ell_{h}(\phi^\star, o_{t,h})_j\right| = \left|\sum_{t\in\calI_k^\rightt} \ell_{h}(\phi^\star, o_{t,h})_j - \sum_{t\in\calI_k^\rightt} \E^{\pi_k, M_t}[\ell_h(\phi^\star; o_h)] \right| \leq B\sqrt{\tau \log(12/\delta)}\,.
\end{align*}
Via a union bound over all these events, this holds simultaneously for all $j,k,h$. Hence with probability $1-\delta$, we have $|L_k(\phi^\star)_j|\leq \frac{\tau}{B^2H} H\left(\frac{1}{\tau}B\sqrt{\tau \log(12NKH/\delta)}\right)^2=\log(12NKH/\delta) $ for all $j,k$ simultaneously.
Summing over $j$ and $k$ finishes the proof.
\end{proof}
\begin{lemma}\label{lem: est part3}
    With probability at least $1-\delta/3$, we have
    \begin{align*}
        \sum_{k=1}^KL_k(\phi^\star) \leq 
            \frac{1}{\beta}\sum_{k=1}^KL_k(\phi^\star)^2+4\beta \log(6/\delta)
    \end{align*}
\end{lemma}
\begin{proof}
    Define the random variable $X_k = \min\{L_k(\phi^\star),N\log(12NKH/\delta) \}$. 
    By \pref{lem: empirical freed} we have with probability at least $1-\delta/6$, 
    \begin{align*}
        \sum_{k=1}^KX_t \leq \frac{1}{\beta}\sum_{k=1}^KL_k(\phi^\star)^2 + 4\beta\log(6/\delta)\,, 
    \end{align*}
    where we use that $\E_k[X_k]\leq \E_k[L_k(\phi^\star)]=0$.
    Finally note that with probability $1-\delta/6$ we have $L_k(\phi^\star)=X_k$ for all $k$ by the proof of \pref{lem: est part2}. Combining both events finishes the proof.
\end{proof}
\begin{lemma}\label{lem: online esitmation epoch general}
    With probability at least $1-\delta$, \pref{alg:MF-Epoch-final} satisfies
    \begin{align*}
        &\frac{1}{B^2H}\sum_{k=1}^K\sum_{\phi}\rho_{k}(\phi)\sum_{t\in\calI_k}\max_{j\in[N]}\sum_{h=1}^H \left(\E^{\pi_k, M_t}\left[\ell_h(\phi;o_{h})_j\right]\right)^2 \le O\left(NT^{\frac{1}{3}}\log|\Phi|\right) \\
        &\quad +  \sum_{k=1}^K \sum_{t\in\calI_k}(  F_t(\delta_{\phi^\star}) - F_t(\rho_{k+1}) - \breg_{F_t}(\delta_{\phi^\star}, \rho_{k+1})).  
    \end{align*}
\end{lemma}

\begin{proof}[Proof of \pref{lem: online esitmation epoch general}]
By union bound, the events of \pref{lem: est part1}, \pref{lem: est part2}, and \pref{lem: est part3} hold simultaneously with probability $1-\delta$.
Observe that the update of $\rho_k$ (\pref{eq: epoch level update}) is in the form specified in \pref{lem: mirror descent}. 
Invoking \pref{lem: mirror descent} with $b_k=\frac{2}{\beta}L_k^2$, we get 
\begin{align*}
    &\sum_{k=1}^K  \inner{\rho_k, L_k+\frac{1}{\beta}L_k^2} 
    \leq \frac{\log|\Phi|}{\gamma} \numberthis \label{eq: miror descent bound 33}\\
    &\quad + \sum_{k=1}^K \left(L_k(\phi^\star) + \left(4\gamma+\frac{2}{\beta}\right) L_k(\phi^\star)^2 +\sum_{t\in\calI_k}(F_t(\delta_{\phi^\star}) - F_t(\rho_{k+1}) - \breg_{F_t}(\delta_{\phi^\star}, \rho_{k+1})) \right). \\
\end{align*}
    Chaining \pref{lem: est part2} and \pref{lem: est part3},  
    \begin{align*}
        &\sum_{k=1}^K \left(L_k(\phi^\star) + \left(4\gamma+\frac{2}{\beta}\right) L_k(\phi^\star)^2\right)
        \\ &\leq 4\beta\log(6/\delta)+\left(4\gamma + \frac{3}{\beta}\right)KN^2\log^2(12NKH/\delta). 
    \end{align*}
Using \pref{lem: est part1} and substituting $\beta = 7\tau N\iota$, $\gamma =\frac{1}{2\beta}$ yields
\begin{align*}
     \frac{1}{B^2H}\sum_{k=1}^K \sum_\phi \rho_k(\phi) \sum_{t\in\calI_k} \max_{j\in[N]} \sum_{h=1}^H \left(\E^{\pi_k, M_t}[\ell_h(\phi; o_{h})_j]  \right)^2 \leq 35 \tau N\iota + 20\frac{K N\iota}{\tau}
\end{align*}
Using $K=T/\tau$ and tuning $\tau = T^\frac{1}{3}$ yields $O(T^\frac{1}{3}N\iota)$.
\end{proof}

\subsection{Squared estimation error minimization via bi-level learning}\label{app: bilevel}

        

\begin{algorithm}[H]
    \caption{Bi-level learning algorithm for squared estimation error}
    \label{alg:MF-BiLevel-final}
    \textbf{Input:} An estimation function $\err_h: \Phi\times\Phi\times \calO\to [0, B^2]$ satisfying \pref{assum: squared est err}. \\
    \textbf{Parameter:} $\iota = 64\log|\Phi|$, $\gamma=\frac{1}{4\iota}$. \\
   $\rho_1(\phi) = 1/|\Phi|,\, \forall \phi \in \Phi$ and  $q_1(\phi'|\phi) = 1/|\Phi|$, $\forall \phi', \phi\in \Phi$. \\
    \For{$t=1, 2, \ldots, T$}{
         Receive observation $o_t\sim M_t(\cdot|\pi_t)$. \\
         Define 
         \begin{align*}
             \Delta_t(\phi', \phi) &= \frac{1}{B^2H}\sum_{h=1}^H \err_h(\phi', \phi, o_{t,h}), \\
             L_t(\phi) &= \Delta_t(\phi, \phi) - \E_{\phi'\sim q_t(\cdot|\phi)}\left[\Delta_t(\phi', \phi)\right], 
             \\ 
             b_t(\phi) &= \frac{[\rho_t(\phi) - \max_{s<t}\rho_s(\phi)]_+}{\rho_t(\phi)} \iota. 
         \end{align*}
         Let $F_t: \Delta(\Phi)\to \mathbb{R}$ be a convex function. Calculate
        \begin{align*}
             &\rho_{t+1} = \argmin_{\rho\in \Delta(\Phi)}\left\{\inner{\rho,  L_t + 4\gamma L_t^2 +  b_t} + F_t(\rho) + \frac{1}{\gamma}\KL(\rho, \rho_t) \right\},  \label{eq: top level update} \numberthis \\
              & q_{t+1}(\phi'|\phi) \propto \exp\left(- \alpha_{t}(\phi) \sum_{s=1}^t \rho_s(\phi)\Delta_s(\phi',\phi)\right) \quad \text{where \ \ }  \alpha_{t}(\phi) = \frac{1}{16 \max_{s \le t} \rho_s(\phi)}.  
        \end{align*}
    }
\end{algorithm}

\begin{lemma}\label{lem: online esitmation bilinear outer general}
    With probability at least $1-\delta$, 
    \begin{align*}
        &\sum_{t=1}^T \inner{\rho_{t}, L_t} \le \frac{\log|\Phi|}{\gamma} \\
        &\quad +  \sum_{t=1}^T \left(- \frac{1}{2} \inner{\rho_{t}, b_t} + b_t(\phi^\star) +  F_t(\delta_{\phi^\star}) - F_t(\rho_{t+1}) - \breg_{F_t}(\delta_{\phi^\star}, \rho_{t+1}) \right) + O\left(\log(1/\delta)\right).  
    \end{align*}
\end{lemma}

\begin{proof}[Proof of \pref{lem: online esitmation bilinear outer general}]
Observe that the update of $\rho_t$ (\pref{eq: top level update}) is in the form specified in \pref{lem: mirror descent}. 
Invoking \pref{lem: mirror descent}, we get 
\begin{align*}
    &\sum_{t=1}^T  \inner{\rho_t, L_t} 
    \leq \frac{\log|\Phi|}{\gamma} \numberthis \label{eq: miror descnet bound 33}\\
    &\quad + \sum_{t=1}^T \left(L_t(\phi^\star) + 4\gamma L_t(\phi^\star)^2 +  b_t(\phi^\star)-\frac{1}{2}\inner{\rho_{t}, b_t}+F_t(\delta_{\phi^\star}) - F_t(\rho_{t+1}) - \breg_{F_t}(\delta_{\phi^\star}, \rho_{t+1}) \right). \\
\end{align*}
By \pref{assum: squared est err} we have 
\begin{align*}
    0\leq \E_t[L_t(\phi^\star)^2] 
    &= \E_t\left[\left(\Delta_t(\phi^\star, \phi^\star) - \E_{\phi'\sim q_t(\cdot|\phi^\star)}\left[\Delta_t(\phi', \phi^\star)\right]\right)^2\right] \\
    &\leq \E_{\phi'\sim q_t(\cdot|\phi^\star)}\left[\E_t\left[\left(\Delta_t(\phi^\star, \phi^\star) - \Delta_t(\phi', \phi^\star)\right)^2\right] \right] \tag{Jensen's inequality} \\
    &\leq \E_{\phi'\sim q_t(\cdot|\phi^\star)}\left[\E_t\left[\left(\Delta_t(\calT_{M_t}\phi^\star, \phi^\star) - \Delta_t(\phi', \phi^\star)\right)^2\right] \right] \tag{$M_t\in\phi^\star$ and thus $\calT_{M_t}\phi^\star=\phi^\star$} \\
    &\leq 4\E_{\phi'\sim q_t(\cdot|\phi^\star)}\left[\E_t\left[\Delta_t(\phi', \phi^\star) - \Delta_t(\calT_{M_t}\phi^\star, \phi^\star) \right] \right]  \tag{by \pref{assum: squared est err}} \\
    &= 4\E_{\phi'\sim q_t(\cdot|\phi^\star)}\left[\E_t\left[\Delta_t(\phi', \phi^\star) - \Delta_t(\phi^\star, \phi^\star) \right] \right] \\
    &= - 4\E_t[L_t(\phi^\star)]. 
\end{align*}
This allows us to apply \pref{lem: useful inequali} with $X_t = -L_t(\phi^\star)$ and $Y_t=\frac{1}{4}X_t^2$, which gives
\begin{align*}
    \sum_{t=1}^T \left(L_t(\phi^\star) + 4\gamma L_t(\phi^\star)^2\right) 
    &\leq \sum_{t=1}^T \left(L_t(\phi^\star) + \frac{1}{16} L_t(\phi^\star)^2\right) \\
    &\leq \frac{1}{2}\sum_{t=1}^T \E_t[L_t(\phi^\star)] + 36 \log(1/\delta)\leq 36\log(1/\delta).  
\end{align*}
Combining this with \pref{eq: miror descnet bound 33} finishes the proof. 
\end{proof}

\begin{lemma}\label{lem: max rho lemma 33}
    With probability at least $1-\delta$, 
    \begin{align*}
        \sum_{t=1}^T \E_{\phi\sim \rho_t} \E_{\phi'\sim q_t(\cdot|\phi)}[\Delta_t(\phi', \phi) - \Delta_t(\calT_{M_t}\phi, \phi)] \leq 32 \sum_\phi \max_{t\leq T} \rho_t(\phi) \log|\Phi| + 72\log(1/\delta). 
    \end{align*}
\end{lemma}
\begin{proof}
By \pref{assum: squared est err}, we have $\calT_{M_t}\phi = \calT_{M_{t'}}\phi$ for all $\phi$ and all $t,t'\in[T]$. We denote $\calT \phi = \calT_{M_{t}}\phi$ for any $t$. 
By the exponential weight update, for any $\phi$, 
\begin{align*}
    &\sum_{t=1}^T \sum_{\phi'} q_t(\phi'|\phi)\rho_t(\phi)\left(\Delta_t(\phi', \phi) - \Delta_t(\calT_{M_t}\phi, \phi)\right) 
 \\
 &= \sum_{t=1}^T \sum_{\phi'} q_t(\phi'|\phi)\rho_t(\phi)\left(\Delta_t(\phi', \phi) - \Delta_t(\calT\phi, \phi)\right) \\ 
    &\leq \frac{\log |\Phi|}{\alpha_T(\phi)} + \sum_{t=1}^T \sum_{\phi'}\alpha_t(\phi)q_t(\phi'|\phi)\rho_t(\phi)^2  \left(\Delta_t(\phi', \phi) - \Delta_t(\calT\phi, \phi)\right)^2 \\
    &\leq 16\max_{t\leq T}\rho_t(\phi) \log |\Phi| + \frac{1}{16}\sum_{t=1}^T \sum_{\phi'}q_t(\phi'|\phi)\rho_t(\phi)  \left(\Delta_t(\phi', \phi) - \Delta_t(\calT\phi, \phi)\right)^2.  
\end{align*}
Rearranging and summing over $\phi$:  
\begin{align*}
   &\sum_{t=1}^T \E_{\phi\sim \rho_t} \E_{\phi'\sim q_t(\cdot|\phi)}\left[\Delta_t(\phi', \phi) - \Delta_t(\calT\phi, \phi) - \frac{1}{16}(\Delta_t(\phi', \phi) - \Delta_t(\calT\phi, \phi))^2\right] \\
   &\qquad \qquad \leq 16 \sum_\phi \max_{t\leq T} \rho_t(\phi) \log|\Phi|.   \numberthis \label{eq: some bound 34}
\end{align*}
Define
\begin{align*}
    X_t &= \E_{\phi\sim \rho_t} \E_{\phi'\sim q_t(\cdot|\phi)}\left[\Delta_t(\phi', \phi) - \Delta_t(\calT\phi, \phi)\right], \\
    Y_t &= \frac{1}{4}\E_{\phi\sim \rho_t} \E_{\phi'\sim q_t(\cdot|\phi)}\left[\left(\Delta_t(\phi', \phi) - \Delta_t(\calT\phi, \phi)\right)^2\right].
\end{align*}
By \pref{assum: squared est err} we have $\E_t[Y_t]\leq \E_t[X_t]$. By Jensen's inequality, $\E_t[X_t^2]\leq 4B^2 H\E_t[Y_t]\leq 4B^2 H\E_t[X_t]$. Invoking \pref{lem: useful inequali} and using \pref{eq: some bound 34} give
\begin{align*}
    \frac{1}{2}\sum_{t=1}^T X_t &\leq \sum_{t=1}^T \left(X_t - \frac{1}{4}Y_t\right) + 36\log(1/\delta) 
    \leq 16 \sum_\phi \max_{t\leq T} \rho_t(\phi) \log|\Phi| + 36\log(1/\delta),  
\end{align*}
proving the desired inequality. 

\end{proof}

\begin{lemma}\label{lem: bilievel final bound}
   With probability at least $1-\delta$, 
  \begin{align*}
    &\sum_{t=1}^T \E_{\phi\sim \rho_t} \left[\Delta_t(\phi, \phi) - \Delta_t(\calT_{M_t}\phi, \phi)\right] \\
    &\leq \sum_{t=1}^T  \left(F_t(\delta_{\phi^\star}) - F_t(\rho_{t+1}) - \breg_{F_t}(\delta_{\phi^\star}, \rho_{t+1})\right) + O\left(\log^2(|\Phi|/\delta)\right). 
\end{align*}
\end{lemma}
\begin{proof}
By \pref{assum: squared est err}, we have $\calT_{M_t}\phi = \calT_{M_{t'}}\phi$ for all $\phi$ and all $t,t'\in[T]$. We denote $\calT \phi = \calT_{M_{t}}\phi$ for any $t$.

\begin{align*}
    \E_{\phi\sim \rho_t}[L_t(\phi)] 
    &= \E_{\phi\sim \rho_t}[\Delta_t(\phi, \phi) - \E_{\phi'\sim q_t(\cdot|\phi)}[\Delta_t(\phi', \phi)]] \\
    &= \E_{\phi\sim \rho_t}\left[\Delta_t(\phi, \phi) - \Delta_t(\calT\phi, \phi) - \left(\E_{\phi'\sim q_t(\cdot|\phi)}[\Delta_t(\phi', \phi)] - \Delta_t(\calT\phi, \phi)\right)\right].
\end{align*}
Combining this with \pref{lem: online esitmation bilinear outer general}, we get 
\begin{align*}
   &\sum_{t=1}^T \E_{\phi\sim \rho_t} \left[\Delta_t(\phi, \phi) - \Delta_t(\calT\phi, \phi)\right] \\
   &\leq \frac{\log|\Phi|}{\gamma} + \sum_{t=1}^T \left(- \frac{1}{2} \inner{\rho_{t}, b_t} + b_t(\phi^\star) +  F_t(\delta_{\phi^\star}) - F_t(\rho_{t+1}) - \breg_{F_t}(\delta_{\phi^\star}, \rho_{t+1}) \right) \\
   &\qquad + O\left(\log(1/\delta)\right)   + \sum_{t=1}^T \E_{\phi\sim \rho_t}\E_{\phi'\sim q_t(\cdot|\phi)}\left[\Delta_t(\phi', \phi) - \Delta_t(\calT\phi, \phi)\right] \\
   &\leq \sum_{t=1}^T \left(- \frac{1}{2} \inner{\rho_{t}, b_t} + b_t(\phi^\star) +  F_t(\delta_{\phi^\star}) - F_t(\rho_{t+1}) - \breg_{F_t}(\delta_{\phi^\star}, \rho_{t+1})\right) \\
   &\qquad + O\left(\log^2(|\Phi|/\delta)\right)  + 32\sum_\phi \max_{t\leq T} \rho_t(\phi) \log|\Phi|.   \tag{by \pref{lem: max rho lemma 33} and the value of $\gamma$} 
\end{align*}

Note that 
\begin{align*}
    32\log|\Phi|\sum_{\phi}\max_{t \le T} \rho_{t}(\phi)  
    &= 32\log|\Phi|  \sum_{\phi}\left(\rho_1(\phi) +\sum_{t=2}^T [\rho_t(\phi)-\max_{s<t}\rho_s(\phi)]_+\right)
    \\&= 32\log|\Phi|   \sum_{\phi}\left(\rho_1(\phi) +\sum_{t=2}^T \rho_t(\phi)\frac{[\rho_t(\phi)-\max_{s<t}\rho_s(\phi)]_+}{\rho_t(\phi)}\right)
    \\&= \frac{1}{2} \sum_{t=1}^T \inner{\rho_t, b_t}  
\end{align*}
and 
\begin{align*}
    \sum_{t=1}^T b_t(\phi^\star) 
    &= O(\log|\Phi|)\times \sum_{t=1}^T \frac{\max_{s\leq t}\rho_s(\phi^\star) - \max_{s\leq t-1}\rho_s(\phi^\star)}{\max_{s\leq t}\rho_s(\phi^\star)} \\
    &\leq O(\log|\Phi|)\times \left(1+\sum_{t=2}^T  \ln \frac{\max_{s\leq t}\rho_s(\phi^\star)}{\max_{s\leq t-1} \rho_s(\phi^\star)}\right) \tag{$1-x\leq \ln \frac{1}{x}$}\\
    &\leq O\left(\log^2|\Phi|\right). 
\end{align*}
Combining inequalities above proves the lemma. 

\end{proof}


\newpage
\section{Omitted Details in \pref{sec: general framework}}
\label{app:model-free stoc}
We define a batched version of \pref{alg:general} in \pref{alg:general batched}. When the batch size $\tau=1$, it is exactly \pref{alg:general}. One can also think of \pref{alg:general batched} as a special case of \pref{alg:general} where $\postupdate$ makes a real update only when $t=k\tau$ for $k=1,2,\ldots$, and keeps $\rho_{t+1}=\rho_t$ otherwise. 

\begin{algorithm}[H]
    \caption{General Batched Framework}
     \label{alg:general batched}
    \textbf{Input:} Partition set $\Phi$ and its union $\Psi$ (defined in \pref{sec: hybrid setting}). Batch size $\tau$. \\
    $\rho_1(\phi) = 1/|\Phi|,\, \forall \phi \in \Phi$. \\
    \For{$k=1, 2, \ldots, K$}{
       Set $p_k, \nu_k$ as the solution of the following minimax optimization (defined in \pref{eq: new AIR}): 
        \begin{align}
            \min_{p \in \Delta(\Pi)}\max_{\nu \in \Delta(\Psi)} \air^{\Phi, D}_\eta(p,\nu; \rho_k). 
            \label{eq:minmax batch}
        \end{align}
        Execute $\pi_k$ in rounds $t\in \{(k-1)\tau+1, \ldots, k\tau\}=\calI_k$ and receive observations $(o_t)_{t\in\calI_k}$. 
        \vspace{-5pt}
        \begin{flalign}
        &\text{Update\ } \rho_{k+1}= \postupdate(\nu_k, \rho_k, \pi_k, (o_t)_{t\in \calI_k}). && \label{eq:rho batch}
        \end{flalign}
        
        }
        \vspace{-3pt}
\end{algorithm}
\subsection{Assumption reductions}
\begin{proof}[Proof of \pref{lem: av assum reduction}]
In the stochastic setting, by \pref{assum: function approximation stochastic} we have $f_{\phi}(s,a)=Q^{\star}(s,a;M)$ and $f_{\phi}(s)=V^{\star}(s;M)$ for any $M\in\phi$. Hence
\begin{align*}
        \E^{\pi, M}[\ellest_h(\phi; o_h)] &= \E^{\pi, M}[f_{\phi}(s_h,a_h)-r_h-f_{\phi}(s_{h+1})]\\
        &=\E^{\pi, M}[Q^{\star}(s_h,a_h;M)-r_h-V^\star(s_{h+1};M)]=0.    
    \end{align*}
In the hybrid setting, we have by \pref{assum: function approximation hybrid} and \pref{assum: unique mapping} that 
$f_{\phi}(s,a;R)=Q^{\pi_{\phi}}(s,a;(P,R))$ and $f_{\phi}(s;R)=V^{\pi_{\phi}}(s;(P,R))$ for any $P\in\phi$. Hence, for any $j\in[d]$, defining $R'$ such that $R'(s,a)=\varphi(s,a)_j$, we have for $(P,R)\in\phi$, 
\begin{align*}
    \E^{\pi, (P,R)}[\ellest_h(\phi; o_h)_j] &= \E^{\pi, P}[f_{\phi}(s_h,a_h;R')-R'(s_h,a_h)-f_{\phi}(s_{h+1};R')]\\
        &=\E^{\pi, P}[Q^{\pi_{\phi}}(s_h,a_h;(P,R'))-R'(s,a)-V^{\pi_{\phi}}(s_{h+1};(P,R'))]=0. 
\end{align*}
Finally, note that in the stochastic setting $M_t=M^\star$, and in the hybrid setting $P_t=P^\star$, so the additional assumption always holds.
\end{proof}
\begin{proof}[Proof of \pref{lem: sq assum reduction}]

In the stochastic setting, 
with \pref{assum: function approximation stochastic} and the Bellman completeness assumption (\pref{def: sto bc}), for any $M=(P,R)$, we define $\calT_M \phi\in\Phi$ as the $\phi'$ such that 
\begin{align*}
f_{\phi'}(s,a) = R(s,a) + \E_{s'\sim P(\cdot|s,a)}[f_\phi(s')]. 
\end{align*}
By \pref{def: sto bc}, such $\phi'$ always exists. 

In the hybrid setting, with \pref{assum: function approximation hybrid}, \pref{assum: unique mapping} and \pref{assum: known feature} and the Bellman completeness assumption (\pref{def: adv bc}), for any $M=(P,R)$, we define $\calT_M\phi\in\Phi$ to be the $\phi'$ such that $\pi_{\phi'}=\pi_\phi$ and for all $\tilde{R}$, 
\begin{align*}
    f_{\phi'}(s,a; \tilde{R}) = \tilde{R}(s,a) + \E_{s'\sim P(\cdot|s,a)}[f_\phi(s';\tilde{R})]. 
\end{align*}
By \pref{def: adv bc}, such $\phi'$ always exists. 

Below, with a slight overload of notation, we denote
in the hybrid setting $f_{\phi}(s_h,a_h)\in\bbR^{d}$ as the vector $(f_{\phi}(s_h,a_h;\bm{e}_j))_{j\in[d]}$ and $f_{\phi}(s_{h+1})\in\bbR^{d}$ as the vector $\E_{a\sim\pi_{\phi}(\cdot|s_{h+1})}[(f_{\phi}(s_{h+1},a;\bm{e}_j))_{j\in[d]}]$. Furthermore, we use the notation $y_h$ to denote $r_h\in\mathbb{R}$ in the stochastic setting, and $\varphi(s_h,a_h)\in\mathbb{R}^d$ in the hybrid setting. 

Then we have by our choice of $\err_h$: 
\begin{align*}
&\E^{\pi, M}\left[\err_h(\phi', \phi; o_h) - \err_h(\calT_{M}\phi, \phi; o_h)\right]
\\&= \E^{\pi, M}\left[\|f_{\phi'}(s_h,a_h)-y_h-f_{\phi}(s_{h+1})\|^2 - \|f_{\calT_M\phi}(s_h,a_h)-y_h-f_{\phi}(s_{h+1})\|^2\right]
\\&= \E^{\pi, M}\left[\left\|f_{\phi'}(s_h,a_h)-f_{\calT_M\phi}(s_h,a_h)\right\|^2\right]
\\&\qquad +2\cdot \E^{\pi, M}\left[\inner{f_{\phi'}(s_h,a_h)-f_{\calT_M\phi}(s_h,a_h),f_{\calT_M\phi}(s_h,a_h)-y_h-f_{\phi}(s_{h+1})}\right]
\\&= \E^{\pi, M}\left[\left\|f_{\phi'}(s_h,a_h)-f_{\calT_M\phi}(s_h,a_h)\right\|^2\right]\,, \numberthis \label{eq: second calcul}
\end{align*}
where the last line follows from $\E^{\pi, M}[y_h+f_{\phi}(s_{h+1})]=f_{\calT_M\phi}(s_h,a_h)$ by definition of $\calT_M\phi$.
On the other hand, 
\begin{align*}
    &\E^{\pi, M}\left[\left(\err_h(\phi', \phi; o_h) - \err_h(\calT_{M}\phi, \phi; o_h)\right)^2\right]
    \\&=\E^{\pi, M}\left[\left(\|f_{\phi'}(s_h,a_h)-y_h-f_{\phi}(s_{h+1})\|^2 - \|f_{\calT_M\phi}(s_h,a_h)-y_h-f_{\phi}(s_{h+1})\|^2\right)^2\right]
    \\&=\E^{\pi, M}\left[\inner{f_{\phi'}(s_h,a_h)-f_{\calT_M\phi}(s_h,a_h), \ \ f_{\calT_M\phi}(s_h,a_h)+f_{\phi'}(s_h,a_h)-2y_h-2f_{\phi}(s_{h+1})}^2\right]
    \\&\leq 4B^2 \E^{\pi, M}\left[\left\|f_{\phi'}(s_h,a_h)-f_{\calT_M\phi}(s_h,a_h)\right\|^2\right], 
\end{align*}
where $B^2=1$ in the stochastic setting and $B^2=d$ in the hybrid setting. Combining both finishes the proof. 

\end{proof}

\subsection{Bounds on $\est$}
With the specific form of divergence
\begin{align}
    &D^\pi(\nu\|\rho) = \E_{M \sim \nu}\E_{o \sim M(\cdot|\pi)}\left[\KL\left(\nu_{\bo{\phi}}(\cdot|\pi, o), \rho\right) + \E_{\phi\sim \rho}\left[\tilD^\pi(\phi\|M)\right]\right], \label{eq: div general}
\end{align}
the estimation term in \pref{eq:minimax-guarantee}
for an epoch algorithm with epoch length $\tau'$ and $K$ epochs is given by 

\begin{lemma}\label{lem: general est bound lemm}
    $\est$ in \pref{eq:minimax-guarantee} can be written as 
    \begin{align}
        \est = \sum_{t=1}^T \E_{\pi \sim p_t}\E_{o \sim M_t(\cdot|\pi)}\left[\log\left(\frac{\nu_t(\phi^\star|\pi, o)}{\rho_t(\phi^\star)}\right)\right] +  \sum_{t=1}^T\E_{\pi \sim p_t}\E_{\phi\sim \rho_t}\left[\tilD^{\pi}(\phi\|M_t)\right].    \label{eq: est in batch}
    \end{align}
\end{lemma}
\begin{proof}[Proof of \pref{lem: general est bound lemm}]

From the definition of divergence in \pref{eq: div general} and \pref{eq: est in batch}, let $\delta_{\phi^\star}\in\Delta(\Phi)$ be the Kronecker delta function centered at  $\phi^\star$.  Then  
\begin{align}
   \est&= \sum_{t=1}^T \Bigg(\log\left(\frac{1}{\rho_t(\phi^\star)}\right) +  \E_{\pi \sim p_t}\E_{\phi\sim \rho_t}\left[\tilD^{\pi}(\phi\|M_t)  \right] \nonumber \\
    &\qquad \qquad \qquad \qquad - \E_{\pi \sim p_t}\E_{o \sim M_t(\cdot|\pi)}\left[\KL\left(\delta_{\phi^\star}, (\nu_t)_{\bo{\phi}}(\cdot|\pi, o)\right) \right]\Bigg)\nonumber
    \\&= \sum_{t=1}^T \E_{\pi \sim p_t}\E_{o \sim M_t(\cdot|\pi)}\left[\log\left(\frac{\nu_t(\phi^\star|\pi, o)}{\rho_t(\phi^\star)}\right)\right] +  \sum_{t=1}^T\E_{\pi \sim p_t}\E_{\phi\sim \rho_t}\left[\tilD^{\pi}(\phi\|M_t)\right] \label{eq:est_gen_fir}
\end{align}
where the first equality uses the fact that for any $\rho$, 
\begin{align*}
    \breg_{D^{\pi}(\cdot\|\rho)}(\nu, \nu') = \E_{M \sim \nu}\E_{o \sim M(\cdot|\pi)}\left[\KL\left(\nu_{\bo{\phi}}(\cdot|\pi, o), \nu_{\bo{\phi}}'(\cdot|\pi, o)\right)\right].
\end{align*}
\end{proof}

\begin{proof}[Proof of \pref{thm: avg est}]
With abuse of notation, we use $p_t, \nu_t, \rho_t$ to denote the $p_k, \nu_k, \rho_k$ where $k$ is the epoch where episode $t$ lies. We start from the estimation term in \pref{eq: est in batch} using the definition of $\tilD$:
\begin{align*}
    \est 
    &= \sum_{t=1}^T \E_{\pi \sim p_t}\E_{o \sim M_t(\cdot|\pi)}\left[\log\left(\frac{\nu_t(\phi^\star|\pi, o)}{\rho_t(\phi^\star)}\right)\right] +  \frac{1}{B^2H}\sum_{t=1}^T\E_{\pi \sim p_t}\E_{\phi\sim \rho_t}\left[\max_{j\in[N]}\sum_{h=1}^H  \left(\E^{\pi, M_t}\left[\ellest_h(\phi; o_h)_j\right]\right)^2\right] \\
    &= \sum_{k=1}^K \E_{\pi \sim p_k} \sum_{t\in\calI_k} \E_{o \sim M_t(\cdot|\pi)}\left[\log\left(\frac{\nu_k(\phi^\star|\pi, o)}{\rho_k(\phi^\star)}\right)\right] \\
    &\qquad \qquad +  \frac{1}{B^2H}\sum_{k=1}^K\E_{\pi \sim p_k}\E_{\phi\sim \rho_k}\left[\sum_{t\in\calI_k} \max_{j\in[N]}\sum_{h=1}^H  \left(\E^{\pi, M_t}\left[\ellest_h(\phi; o_h)_j\right]\right)^2\right].  
\end{align*}
Applying \pref{lem: online esitmation epoch general} with $F_t(\rho) = \KL(\rho, (\nu_k)_{\bo{\phi}} (\cdot|\pi_k, o_t))$ for $t\in\calI_k$, we get 
\begin{align*}
   \E[\est] &\leq \E\left[\sum_{k=1}^K \E_{\pi \sim p_k} \sum_{t\in\calI_k} \E_{o \sim M_t(\cdot|\pi)}\left[\log\left(\frac{\nu_k(\phi^\star|\pi, o)}{\rho_k(\phi^\star)}\right)\right] \right] + O\left( N \log(|\Phi|)T^{\frac{1}{3}} \right) \\
   &\qquad \quad + \E\left[\sum_{k=1}^K \sum_{t\in\calI_k} \left(\log\left(\frac{1}{\nu_k(\phi^\star|\pi_k, o_t)}\right) - \KL(\rho_{k+1}, (\nu_k)_{\bo{\phi}} (\pi_k, o_t)) - \log\left(\frac{1}{\rho_{k+1}(\phi^\star)}\right)\right)\right] \\
   &\leq \E\left[\sum_{k=1}^K  \sum_{t\in\calI_k} \left(\log\left(\frac{\nu_k(\phi^\star|\pi_k, o_t)}{\rho_k(\phi^\star)}\right) + \log\left(\frac{\rho_{k+1}(\phi^\star)}{\nu_k(\phi^\star|\pi_k, o_t)}\right)\right) \right]  + O\left( N\log(|\Phi|)T^{\frac{1}{3}} \right) \\
    &\leq \tau \log\left(\frac{1}{\rho_1(\phi^\star)}\right)  + O\left( N \log(|\Phi|)T^{\frac{1}{3}} \right) \\
    &= O\left( N\log(|\Phi|)T^{\frac{1}{3}} \right). 
\end{align*}
\end{proof}
\begin{proof}[Proof of \pref{thm: sq est}]
We start from the estimation term in \pref{eq: est in batch}, using the definition of $\tilD$:   
\begin{align*}
    \est &= \sum_{t=1}^T \E_{\pi \sim p_t}\E_{o \sim M_t(\cdot|\pi)}\left[\log\left(\frac{\nu_t(\phi^\star|\pi, o)}{\rho_t(\phi^\star)}\right)\right] \\
    &\qquad + \frac{1}{B^2H} \sum_{t=1}^T\E_{\pi \sim p_t}\E_{\phi\sim \rho_t}\left[\sum_{h=1}^H \E^{\pi, M_t}\left[\err_h(\phi, \phi; o_h) - \err_h(\calT_{M_t}\phi, \phi; o_h)\right]\right]. 
\end{align*}
Applying \pref{lem: bilievel final bound} with $F_t(\rho) = \KL(\rho, (\nu_t)_{\bo{\phi}} (\cdot|\pi_t, o_t) )$, we get 
\begin{align*}
   \E[\est] &\leq \E\left[\sum_{t=1}^T \E_{\pi \sim p_t}\E_{o \sim M_t(\cdot|\pi)}\left[\log\left(\frac{\nu_t(\phi^\star|\pi, o)}{\rho_t(\phi^\star)}\right)\right] \right]  + O\left(\log^2|\Phi|\right) \\
   &\qquad \quad + \E\left[\sum_{t=1}^T \left(\log\left(\frac{1}{\nu_t(\phi^\star|\pi_t, o_t)}\right) - \KL(\rho_{t+1}, (\nu_t)_{\bo{\phi}} (\pi_t, o_t)) - \log\left(\frac{1}{\rho_{t+1}(\phi^\star)}\right)\right)\right] \\
   &\leq \E\left[\sum_{t=1}^T \log\left(\frac{\rho_{t+1}(\phi^\star)}{\rho_t(\phi^\star)}\right) \right]  + O\left( \log^2|\Phi|\right) = O\left( \log^2|\Phi|\right). 
\end{align*}
\end{proof}

\newpage

\section{Relating $\mfdec$ to Existing Complexities in the Stochastic Setting}


\subsection{Supporting lemmas}

\begin{lemma}\label{lem: connecting two digdec}
    Suppose that $(\calM,\Phi)$ satisfy \pref{assum: avg err} with estimation function $\ellest_h(\phi; o_h) = f_\phi(s_h,a_h) - r_h - f_\phi(s_{h+1})$. Furthermore, assume that $(\calM, \Phi)$ is Bellman complete (\pref{def: sto bc}). Then \pref{assum: squared est err} holds with $\err_h(\phi', \phi; o_h) = (f_{\phi'}(s_h,a_h) - r_h - f_{\phi}(s_{h+1}))^2$ and 
    \begin{align*} 
        \mfdec_\eta^{\Phi, \tilD_{\sbe}} \leq \mfdec_\eta^{\Phi, \tilD_{\bi}}. 
    \end{align*}
\end{lemma}
\begin{proof}
   It suffices to show that $\tilD_\bi^\pi(\phi\|M)\leq \tilD^\pi_\sbe(\phi\|M)$ for any $\pi, \phi, M$:  
   \begin{align*}
       \tilD_\sbe^\pi(\phi\|M) &= \frac{1}{B^2H}\sum_{h=1}^H  \E^{\pi, M}\left[\err_h(\phi, \phi; o_h) - \err_h(\calT_M\phi, \phi; o_h)\right] \\
       &= \frac{1}{B^2H}\sum_{h=1}^H  \E^{\pi, M}\left[\left(f_{\phi}(s_{h}, a_{h}) - f_{\calT_M\phi}(s_{h}, a_{h})\right)^2\right] \tag{by the same calculation as \pref{eq: second calcul}}\\
       &\geq \frac{1}{B^2H}\sum_{h=1}^H  \left(\E^{\pi, M}\left[f_{\phi}(s_{h}, a_{h}) - f_{\calT_M\phi}(s_{h}, a_{h})\right]\right)^2  \tag{Jensen's inequality} \\
       &= \frac{1}{B^2H}\sum_{h=1}^H  \left(\E^{\pi, M}\left[f_{\phi}(s_{h}, a_{h}) - r_h - f_\phi(s_{h+1})\right]\right)^2 \\
       &= \tilD_{\bi}(\phi\|M). 
   \end{align*}
\end{proof}

\subsection{Relating $\mfdec$ to $\odec$}
\begin{proof}[Proof of \pref{thm: improvement}]
In the stochastic setting, by definition, 
\begin{align*}
    &\mfdec_\eta^{\Phi,\tilD}
    =\max_{\rho\in\Delta(\Phi)} \min_{p\in\Delta(\Pi)} \max_{\nu\in \Delta(\calM)} \\
    &\qquad \E_{\pi\sim p}\E_{M\sim \nu} \left[ V_M(\pi_M) - V_M(\pi) - \frac{1}{\eta}\E_{o\sim M(\cdot|\pi)}\left[\KL(\nu_{\bo{\phi}}(\cdot|\pi, o), \rho)\right] - \frac{1}{\eta} \E_{\phi\sim \rho}\left[\tilD^\pi(\phi\|M)\right]  \right]
    \end{align*}
    and 
    \begin{align*}
    \odec_\eta^{\Phi,\tilD} &= \max_{\rho\in\Delta(\Phi)} \min_{p\in\Delta(\Pi)} \max_{\nu\in \Delta(\calM)} \E_{\pi\sim p}\E_{M\sim \nu}\E_{\phi \sim \rho}\left[ V_\phi(\pi_\phi) - V_M(\pi) - \frac{1}{\eta} \tilD^\pi(\phi\|M)\right].   
\end{align*}
For any $\rho, p, \nu$, we have
\begin{align*}
&\E_{\pi\sim p}\E_{M\sim \nu}\left[ V_M(\pi_M) - V_M(\pi) - \frac{1}{\eta}\E_{o\sim M(\cdot|\pi)}\left[\KL(\nu_{\bo{\phi}}(\cdot|\pi, o), \rho)\right] - \frac{1}{\eta} \E_{\phi\sim \rho}\left[\tilD^\pi(\phi\|M)\right]  \right]   
\\&= \underbrace{\E_{M\sim \nu}\E_{\phi \sim \rho}\left[ V_M(\pi_M) - V_{\phi}(\pi_{\phi})   \right]- \frac{1}{\eta}\KL(\nu_{\bo{\phi}}, \rho) }_{\textbf{term1}} - \frac{1}{\eta} \E_{\pi\sim p} \E_{M\sim \nu} \E_{o\sim M(\cdot|\pi)} [\KL(\nu_{\bo{\phi}}(\cdot|\pi, o), \nu_{\bo{\phi}})]  
\\&\quad + \E_{\pi\sim p}\E_{M\sim \nu}\E_{\phi \sim \rho}\left[ V_{\phi}(\pi_\phi) - V_M(\pi)- \frac{1}{\eta} \tilD^\pi(\phi\|M)  \right].   
\end{align*}
To bound \textbf{term1}, observe that
\begin{align*}
    \E_{M\sim \nu}\left[ V_M(\pi_M)\right] &= \E_{\phi \sim \nu}\left[ V_{\phi}(\pi_{\phi})\right].  
\end{align*}

Thus, 
\begin{align*}
    &\textbf{term1}
    = \E_{\phi\sim \nu}[V_\phi(\pi_\phi)] - \E_{\phi\sim \rho}[V_\phi(\pi_\phi)] - \frac{1}{\eta} \KL(\nu_{\bo{\phi}}, \rho)\leq \eta.  \tag{\pref{lem:stability}}
\end{align*}
This implies
\begin{align*}
    &\mfdec_\eta^{\Phi,\tilD} \\
    &\le \eta + \max_{\rho\in\Delta(\Phi)} \min_{p\in\Delta(\Pi)} \max_{\nu\in \Delta(\calM)} \E_{\pi\sim p}\E_{M\sim \nu}\E_{\phi \sim \rho}\left[ V_{\phi}(\pi_\phi) - V_M(\pi)- \frac{1}{\eta} \tilD^\pi(\phi\|M) - \frac{1}{\eta} \E_{o\sim M(\cdot|\pi)} [\KL(\nu_{\bo{\phi}}(\cdot|\pi, o), \nu_{\bo{\phi}})] \right] \\
    &\leq \eta +\max_{\rho\in\Delta(\Phi)} \min_{p\in\Delta(\Pi)} \max_{\nu\in \Delta(\calM)} \E_{\pi\sim p}\E_{M\sim \nu}\E_{\phi \sim \rho}\left[ V_{\phi}(\pi_\phi) - V_M(\pi)- \frac{1}{\eta} \tilD^\pi(\phi\|M) \right] \\
    &= \eta + \odec^{\Phi,\tilD}_\eta. 
\end{align*}
\end{proof}

\subsection{Relating $\mfdec$ to bilinear rank}\label{app: relate sto to bilinear}
Bilinear rank is a complexity measure proposed in \cite{du2021bilinear}. It is defined as the following.

\begin{assumption}[Bilinear class \citep{du2021bilinear}]
\label{assum:sto bilinear}
A model class $\calM$ and its associated $\Phi$ satisfying \pref{assum: function approximation stochastic} is a bilinear class with rank $d$ if there exists functions $X_h: \Phi \times \mathcal{M} \rightarrow \mathbb{R}^d$ and  $W_h: \Phi \times \mathcal{M} \rightarrow \mathbb{R}^d$ for all $h\in[H]$ such that 
\begin{enumerate}[leftmargin=1.5em, nosep]
 \setlength{\itemsep}{0pt} 
    \setlength{\parskip}{0pt} 
    \setlength{\topsep}{0pt}
    \item For $M\in\phi$, it holds that $W_h(\phi; M)=0$. 
    \item For any $\phi \in \Phi$ and any $M \in \calM$, 
    \begin{align*}
        \left|V_\phi(\pi_\phi) - V_{M}(\pi_\phi)\right| \leq \sum_{h=1}^H \left|\langle X_h(\phi; M), W_h(\phi; M)\rangle\right|.
    \end{align*}
    \item For every policy $\pi$, there exists an estimation policy $\pi^{\esttt}$. Also, there exists a discrepancy function $\ellest_h: \Phi \times \calO  \rightarrow \mathbb{R}$ such that for any $\phi', \phi \in \Phi$ and any $M\in\calM$,  
    \begin{align*}
        \left|\langle X_h(\phi'; M), W_h(\phi; M) \rangle\right| = \left|\E^{\pi_{\phi'} \,\circ_h\, \pi_{\phi'}^{\esttt},\, M} \left[\ellest_h(\phi; o_h)\right]\right|
    \end{align*}
    where $o_h = (s_h,a_h, r_h, s_{h+1})$ and $\pi \circ_h \pi^{\esttt}$ denotes a policy that plays $\pi$ for the first $h-1$ steps and plays policy $\pi^{\esttt}$\, at the $h$-th step. 
\end{enumerate}
We call it an on-policy bilinear class if $\pi^{\esttt}=\pi$ for all $\pi\in\Pi$, and otherwise an off-policy bilinear class. As in prior work \citep{du2021bilinear, foster2021statistical}, for the off-policy case, we assume $|\calA|$ is finite, and $\pi^{\esttt}$ is always $\unif(\calA)$. We denote by $\pi^\alpha$ the policy that in every step $h=1,\ldots, H$ chooses $\pi$ with probability $1-\frac{\alpha}{H}$ and chooses $\pi^{\esttt}$ with probability $\frac{\alpha}{H}$. 
\end{assumption}

\begin{lemma}\label{lem: bilinear satisfy M* property}
   Bilinear classes (\pref{assum:sto bilinear}) satisfy \pref{assum: avg err}. 
\end{lemma}
\begin{proof}[Proof of \pref{lem: bilinear satisfy M* property}]
For any $\phi'\in\Phi$ and any $(M,\phi)$ such that $M\in\phi$, 
\begin{align*}
    \left|\E^{\pi_{\phi'}\circ_h \pi_{\phi'}^{\esttt}, M} \left[  \ellest_h(\phi; o_h) \right] \right| 
    &= \left| \inner{X_h(\phi'; M),  W_h(\phi, M)}\right|   \tag{by \pref{assum:sto bilinear}.3} \\
    &= 0.   \tag{by \pref{assum:sto bilinear}.1 and that $M\in\phi$}
\end{align*}
\end{proof}

\begin{lemma}\label{lem: digdec < d bi}
    Let $(\calM, \Phi)$ be a bilinear class (\pref{assum:sto bilinear}). Then 
    \begin{itemize}[leftmargin=1em, nosep]
 \setlength{\itemsep}{0pt} 
    \setlength{\parskip}{0pt} 
    \setlength{\topsep}{0pt}
    \item $\mfdec_\eta^{\Phi,\tilD_{\bi}}  \le O(B^2H^2 d\eta)$ in the on-policy case. 
    \item $\mfdec_\eta^{\Phi,\tilD_{\bi}}  \le O(\sqrt{B^2 H^3 d \eta})$ in the off-policy case.
\end{itemize}
\end{lemma}
\begin{proof}[Proof of \pref{lem: digdec < d bi}]
    We first use \pref{thm: improvement} to bound $\mfdec_\eta^{\Phi, \tilD_\bi}$ by $\odec_\eta^{\Phi, \tilD_\bi}+\eta$, and then use \pref{lem: bilinear complexity sto} to relate $\odec_\eta^{\Phi, \tilD_\bi}$ to bilinear rank. 
\end{proof}

\begin{lemma}[Proposition~2.2 of \cite{foster2024model}] \label{lem: bilinear complexity sto}
Let $(\calM, \Phi)$ be a bilinear class (\pref{assum:sto bilinear}). Then 
\begin{itemize}[leftmargin=1em, nosep]
 \setlength{\itemsep}{0pt} 
    \setlength{\parskip}{0pt} 
    \setlength{\topsep}{0pt}
    \item $\odec_\eta^{\Phi,\tilD_{\bi}}  \le O(B^2H^2 d\eta)$ in the on-policy case; 
    \item $\odec_\eta^{\Phi,\tilD_{\bi}}  \le O(\sqrt{B^2 H^3 d|\calA|\eta})$ in the off-policy case.\footnote{\label{foot: foot}In \cite{foster2024model}, the bounds on $\odec_\eta^{\Phi, \tilD_\bi}$ have different scaling of $B,H$ than ours. This is because their average estimation error does not involve the normalization factor $\frac{1}{B^2H}$ like ours (\pref{thm: avg est}). We normalize $\tilD_\bi$ to keep the two information gain terms in Dig-DEC of the same unit. Equivalently, one can view our $\eta$ as a scaled version of theirs.    
    }
\end{itemize}
\end{lemma}

\subsection{Relating $\mfdec$ to Bellman-eluder dimension}\label{app: sto Be}
\begin{lemma}\label{lem: bound Q type BE} 
    Let $\ellest_h(\phi; o_h) = f_\phi(s_h,a_h) - r_h - f_\phi(s_{h+1})$, and let $\tilD_\bi$ be defined with respect to this $\ellest_h$. Then 
    \begin{itemize}[leftmargin=1em, nosep]
 \setlength{\itemsep}{0pt} 
    \setlength{\parskip}{0pt} 
    \setlength{\topsep}{0pt}
        \item If the $Q$-type Bellman-eluder dimension of $(\calM,\Phi)$ is bounded by $d$, then $\mfdec_\eta^{\Phi, \tilD_{\bi}}\leq O(Hd\eta)$. 
        \item \mbox{If the $V$-type Bellman-eluder dimension of $(\calM,\Phi)$ is bounded by $d$, then $\mfdec_\eta^{\Phi, \tilD_{\bi}}\leq O(H\sqrt{d|\calA|\eta })$.} 
    \end{itemize}
     
\end{lemma}
\begin{proof}
    We first consider the $Q$-type setting.  Define $g_h(\phi', \phi; M) = \E^{\pi_{\phi'}, M}[\ellest_h(\phi; o_h)]$. 
    For a fixed $M$, we have by the AM-GM inequality  
    \begin{align*}
        \E_{\phi\sim \rho} \left[g_h(\phi, \phi; M)\right] 
        &\leq  \frac{\lambda}{4}\cdot \E_{\phi\sim \rho} \left[\frac{g_h(\phi, \phi; M)^2}{\E_{\phi'\sim \rho} \left[g_h(\phi', \phi; M)^2\right]}\right] + \frac{1}{\lambda}\E_{\phi\sim \rho}\E_{\phi'\sim \rho} \left[g_h(\phi', \phi; M)^2\right]  
    \end{align*}
    for any $\lambda>0$, 
    implying that 
    \begin{align*}
        &\odec_\eta^{\Phi, \tilD_\bi} \\
        &=\max_\rho \min_p \max_\nu \E_{\pi\sim p}\E_{\phi\sim \rho} \E_{M\sim \nu} \left[V_\phi(\pi_\phi) - V_M(\pi) - \frac{1}{\eta B^2 H} \sum_{h=1}^H \left(\E^{\pi, M}\left[\ellest_h(\phi; o_h)\right]\right)^2\right]  \\ 
        &\leq \max_\rho \max_\nu \E_{\phi'\sim \rho}\E_{\phi\sim \rho} \E_{M\sim \nu} \left[V_\phi(\pi_\phi) - V_M(\pi_\phi) - \frac{1}{\eta B^2 H} \sum_{h=1}^H \left(\E^{\pi_{\phi'}, M}\left[\ellest_h(\phi; o_h)\right]\right)^2\right]  \\
        &=\max_\rho \max_\nu \E_{\phi'\sim \rho} \E_{\phi\sim \rho} \E_{M\sim \nu}\left[\sum_{h=1}^H g_h(\phi, \phi; M) - \frac{1}{\eta B^2 H} \sum_{h=1}^H  g_\phi(\phi', \phi, M)^2\right] \\
        &\leq \frac{\eta B^2 H}{4} \max_{\rho}\max_\nu \sum_{h=1}^H  \E_{\phi\sim \rho} \left[\frac{g_h(\phi, \phi; M)^2}{\E_{\phi'\sim \rho} \left[g_h(\phi', \phi; M)^2\right]}\right]. 
    \end{align*}
    The rest of the proof goes through standard steps. First, bound $\E_{\phi\sim \rho} \left[\frac{g_h(\phi, \phi; M)^2}{\E_{\phi'\sim \rho} \left[g_h(\phi', \phi; M)^2\right]}\right]$ by the \emph{disagreement coefficient} of the function class $\calF_M = \left\{f_\phi - \calT_M f_\phi: \phi\in\Phi \right\}$ where $(\calT_M f)(s,a) \triangleq R(s,a) + \E_{s'\sim P(\cdot|s,a)}[f(s')]$ under the probability measure $\E_{\phi\sim \rho}[d_h^{\pi_\phi, M}]$ (Lemma E.2 of \cite{foster2021statistical}). Taking a maximum over $\rho$, this can be further bounded by the \emph{distributional eluder dimension} of $\calF_M$ over the probability measure space $\calD_{\Phi, M} = \{d_h^{\pi_\phi, M}: \phi\in\Phi\}$ (Lemma 6.1 of \cite{foster2021statistical} and Theorem~2.10 of \cite{foster2021instance}), which is equivalent to the \emph{$Q$-type Bellman-eluder dimension} in $M$ defined in \cite{jin2021bellman}. This then allows us to bound $\odec_\eta^{\Phi, \tilD_{\bi}}\leq \eta dB^2H^2$, where $d$ is the maximum $Q$-type Bellman-eluder dimension over all possible~$M$. 

    Next, we consider the $V$-type setting. Define $g_h(\phi', \phi; M)=\E^{\pi_{\phi'}\circ_h \pi_\phi, M}[\ellest_h(\phi; o_h)]$.  For a fixed $M$, we have by the AM-GM inequality  
    \begin{align*}
        \E_{\phi\sim \rho} \left[g_h(\phi, \phi; M)\right] 
        &\leq  \frac{\lambda}{4}\cdot \E_{\phi\sim \rho} \left[\frac{g_h(\phi, \phi; M)^2}{\E_{\phi'\sim \rho} \left[g_h(\phi', \phi; M)^2\right]}\right] + \frac{1}{\lambda}\E_{\phi\sim \rho}\E_{\phi'\sim \rho} \left[g_h(\phi', \phi; M)^2\right]  
    \end{align*}
    for any $\lambda>0$. 
    Below, let $\pi^\alpha$ be the policy that in every step $h$, with probability $1-\frac{\alpha}{H}$ executes policy $\pi$, and with probability $\frac{\alpha}{H}$ executes $\unif(\calA)$. Then we have 
    \begin{align*}
        &\odec_\eta^{\Phi, \tilD_\bi} \\
        &=\max_\rho \min_p \max_\nu \E_{\pi\sim p}\E_{\phi\sim \rho} \E_{M\sim \nu} \left[V_\phi(\pi_\phi) - V_M(\pi) - \frac{1}{\eta B^2H} \sum_{h=1}^H \left(\E^{\pi, M}\left[\ellest_h(\phi; o_h)\right]\right)^2\right]  \\ 
        &\leq \max_\rho \max_\nu \E_{\phi'\sim \rho}\E_{\phi\sim \rho} \E_{M\sim \nu} \left[V_\phi(\pi_\phi) - V_M(\pi_\phi^\alpha) - \frac{1}{\eta B^2H} \sum_{h=1}^H \left(\E^{\pi_{\phi'}^\alpha, M}\left[\ellest_h(\phi; o_h)\right]\right)^2\right]  \\
        &\leq \alpha + \max_\rho \max_\nu \E_{\phi'\sim \rho}\E_{\phi\sim \rho} \E_{M\sim \nu} \left[V_\phi(\pi_\phi) - V_M(\pi_\phi) - \frac{1}{\eta B^2H}\cdot\frac{\alpha}{3 H|\calA|} \sum_{h=1}^H \left(\E^{\pi_{\phi'}\circ_h \pi_\phi, M}\left[\ellest_h(\phi; o_h)\right]\right)^2\right]  \\
        &=\alpha + \max_\rho \max_\nu \E_{\phi'\sim \rho} \E_{\phi\sim \rho} \E_{M\sim \nu}\left[\sum_{h=1}^H g_h(\phi, \phi; M) - \frac{\alpha}{3\eta B^2 H^2|\calA|} \sum_{h=1}^H  g_\phi(\phi', \phi, M)^2\right] \\
        &\leq \alpha + \frac{3\eta B^2H^2|\calA|}{4\alpha} \max_{\rho}\max_\nu \sum_{h=1}^H  \E_{\phi\sim \rho} \left[\frac{g_h(\phi, \phi; M)^2}{\E_{\phi'\sim \rho} \left[g_h(\phi', \phi; M)^2\right]}\right]. 
    \end{align*}
    where the second inequality is because with probability at least $\left(1-\frac{\alpha}{H}\right)^{h-1}\frac{\alpha}{H|\calA|}\geq \frac{\alpha}{3H|\calA|}$, the policy $\pi^\alpha_{\phi'}$ chooses the same actions in steps $1,\ldots, h$ as the policy $\pi_{\phi'}\circ_h \pi_\phi$. 
    Similar to the $Q$-type analysis, the last expression can be related to $V$-type Bellman-eluder dimension (notice that the definition of $g_h$ is different for $Q$-type and $V$-type). This gives $\odec_\eta^{\Phi, \tilD_\bi}\lesssim \alpha + \frac{B^2 H^3 d|\calA|\eta }{\alpha}=O\big(\sqrt{B^2 H^3 d|\calA|\eta}\big)$ by choosing the optimal $\alpha$. 

    Finally, using \pref{thm: improvement} finishes the proof. 
\end{proof}

\subsection{Relating $\mfdec$ to coverability under Bellman completeness}\label{app: related sto coverability}

\begin{lemma}\label{lem: bound coverability bound}
    Let $(\calM, \Phi)$ be Bellman complete (\pref{def: sto bc}), and suppose the coverability of every model in $\calM$ is bounded by $d$. Then it holds that $\odec_\eta^{\Phi, \tilD_{\sbe}}\leq \eta dH$ where $\tilD_\sbe$ is defined with 
    \begin{align*}
        \err_h(\phi', \phi; o_h) = (f_{\phi'}(s_h,a_h) - r_h - f_{\phi}(s_{h+1}))^2. 
    \end{align*}
\end{lemma}
\begin{proof}
    For $M=(P,R)$, define 
    \begin{align*}
        g_h(s,a, \phi; M)  &= f_{\phi}(s,a) - R(s,a) - \E_{s'\sim P(\cdot|s,a)}[f_{\phi}(s')] =  f_{\phi}(s,a) - f_{\calT_M\phi}(s,a), \\
        d_h^{\rho, M}(s,a) &= \E_{\phi\sim \rho}\left[d_h^{\pi_\phi, M}(s,a) \right]. 
    \end{align*}
    By the AM-GM inequality, for any $\lambda>0$, 
     \begin{align*}
        &\E_{\phi\sim \rho} \E^{\pi_\phi, M} \left[g_h(s_h,a_h, \phi; M)\right] 
        \\&= \E_{\phi\sim \rho} \E_{(s,a)\sim d_h^{\pi_\phi, M}} \left[g_h(s,a, \phi; M)\right] \\
        &= \E_{\phi\sim \rho} \E_{(s,a)\sim d_h^{\rho, M}} \left[ \frac{d_h^{\pi_\phi, M}(s,a)}{d_h^{\rho, M}(s,a)} g_h(s,a, \phi; M)\right] \\
        &\leq  \E_{\phi\sim \rho} \E_{(s,a)\sim d_h^{\rho, M}} \left[ \frac{\lambda}{4}\frac{d_h^{\pi_\phi, M}(s,a)^2}{d_h^{\rho, M}(s,a)^2} + \frac{1}{\lambda}g_h(s,a, \phi; M)^2\right] \\
        &= \frac{\lambda}{4} \E_{\phi\sim \rho}  \left[\sum_{s,a} \frac{d_h^{\pi_\phi, M}(s,a)^2}{d_h^{\rho, M}(s,a)}\right] + \frac{1}{\lambda}\E_{\phi\sim \rho} \E_{\phi'\sim \rho}\E^{\pi_{\phi'}, M}\left[g_h(s_h,a_h, \phi, M)^2\right].  \numberthis \label{eq: tmptmptmp}
    \end{align*}
    Note that 
    \begin{align*}
        \sum_{h=1}^H \E_{\phi\sim \rho} \E^{\pi_\phi, M} [g_h(s_h,a_h, \phi; M)] = \E_{\phi\sim \rho} \left[ V_\phi(\pi_\phi) - V_M(\pi_\phi) \right],  
    \end{align*}
    and by the same calculation as \pref{eq: second calcul}, we have 
    \begin{align*}
        \frac{1}{B^2H}\sum_{h=1}^H \E^{\pi_{\phi'}, M}\left[g_h(s_h,a_h, \phi, M)^2\right] = \frac{1}{B^2H}\sum_{h=1}^H\E^{\pi_{\phi'}, M}\left[ \err_h(\phi', \phi; o_h) - \err_h(\calT_M \phi, \phi; o_h)\right] = \tilD_\sbe^{\pi_{\phi'}}(\phi\|M). 
    \end{align*}
 By the definition of $\odec$ and combining the inequalities above, 
 \begin{align*}
        &\odec_\eta^{\Phi, \tilD_\sbe} \\
        &=\max_\rho \min_p \max_\nu \E_{\pi\sim p}\E_{\phi\sim \rho} \E_{M\sim \nu} \left[V_\phi(\pi_\phi) - V_M(\pi) - \frac{1}{\eta} \tilD_\sbe^{\pi}(\phi\|M)\right]  \\ 
        &\leq \max_\rho \max_\nu \E_{\phi'\sim \rho}\E_{\phi\sim \rho} \E_{M\sim \nu} \left[V_\phi(\pi_\phi) - V_M(\pi_\phi) - \frac{1}{\eta} \tilD_\sbe^{\pi_{\phi'}}(\phi\|M)\right]  \\
        &=\max_\rho \max_\nu \E_{\phi'\sim \rho} \E_{\phi\sim \rho} \E_{M\sim \nu}\left[\sum_{h=1}^H \E^{\pi_\phi, M}[g_h(s_h,a_h, \phi; M)] - \frac{1}{\eta B^2H} \sum_{h=1}^H \E^{\pi_{\phi'}, M}\left[g_h(s_h,a_h, \phi, M)^2\right]\right] \\
        &\leq \frac{\eta B^2H}{4} \max_\rho \max_\nu \E_{M\sim \nu}\E_{\phi\sim \rho}  \left[\sum_{h=1}^H \sum_{s,a} \frac{d_h^{\pi_\phi, P}(s,a)^2}{d_h^{\rho, P}(s,a)}\right].  \tag{by \pref{eq: tmptmptmp}}
    \end{align*}

    Let $\mu_h^P$ be any occupancy measure over layer $h$ that depends on $P$. Then 
    \begin{align*}
        \E_{\phi\sim \rho}  \left[\sum_{s,a} \frac{d_h^{\pi_\phi, P}(s,a)^2}{d_h^{\rho, P}(s,a)}\right] 
        &= \E_{\phi\sim \rho}  \left[\sum_{s,a} \frac{d_h^{\pi_\phi, P}(s,a) \mu_h^P(s,a)}{d_h^{\rho, P}(s,a)}\cdot \frac{d^{\pi_\phi, P}(s,a)}{\mu_h^P(s,a)}\right] \\
        &\leq \E_{\phi\sim \rho}  \left[\sum_{s,a} \frac{d_h^{\pi_\phi, P}(s,a) \mu_h^P(s,a)}{d_h^{\rho, P}(s,a)} \right] \cdot\max_{s,a, \pi}\frac{d_h^{\pi, P}(s,a)}{\mu_h^P(s,a)} \\
        &= \sum_{s,a} \mu_h^P(s,a) \cdot \max_{s,a, \pi}\frac{d_h^{\pi, P}(s,a)}{\mu_h^P(s,a)} \\
        &= \max_{s,a, \pi}\frac{d_h^{\pi, P}(s,a)}{\mu_h^P(s,a)}. 
    \end{align*}
    We let $\mu_h^P$ be the minimizer of $\max_{s,a, \pi}\frac{d_h^{\pi, P}(s,a)}{\mu_h^P(s,a)}$. The coverability in MDP $M$ is defined as $\min_\mu \max_{s,a, \pi, h}\frac{d_h^{\pi, P}(s,a)}{\mu_h^P(s,a)}$ \citep{xie2022role}. Combining the inequalities proves $\odec_\eta^{\Phi, \tilD_{\sbe}}\leq \eta dB^2H^2$. 
\end{proof}

\newpage

\section{Relating $\mfdec$ to Existing Complexities in the Hybrid Setting}

\subsection{Supporting lemmas}

\begin{lemma}\label{lem: triangular discri}
   Let $g: \Phi\to [0,G]$. For $\nu, \rho\in\Delta(\Phi)$, we have  
   \begin{align*}
       \E_{\phi\sim \rho}[g(\phi)] \leq 3\E_{\phi\sim \nu}[g(\phi)] + 2G\cdot D_{\textup{H}}^2(\nu, \rho),   
   \end{align*}
   where $D_{\textup{H}}^2$ is the Hellinger distance. 
\end{lemma}
\begin{proof}
    \begin{align*}
        \left|\E_{\phi\sim \rho}[g(\phi)] - \E_{\phi\sim \nu}[g(\phi)]\right| 
        &= \left|\sum_\phi (\rho(\phi) - \nu(\phi)) g(\phi)\right| \\
        &\leq \sqrt{\sum_\phi (\rho(\phi) + \nu(\phi)) g(\phi)^2 } \sqrt{\sum_\phi \frac{(\rho(\phi) - \nu(\phi))^2}{\rho(\phi) + \nu(\phi)}} \\
        &\leq \frac{1}{2} \E_{\phi\sim \rho}[g(\phi)] + \frac{1}{2} \E_{\phi\sim \nu}[g(\phi)] + \frac{G}{2} D_\Delta(\nu, \rho),    \numberthis \label{eq: triangle inequ}
    \end{align*}
    where 
    \begin{align*}
        D_\Delta(\nu, \rho) = \sum_\phi \frac{(\rho(\phi) - \nu(\phi))^2}{\rho(\phi) + \nu(\phi)}
    \end{align*}
    is the triangular discrimination. We can further bound it as 
    \begin{align*}
        D_\Delta(\nu, \rho) = \sum_\phi \frac{(\rho(\phi) - \nu(\phi))^2}{\rho(\phi) + \nu(\phi)} = \sum_\phi  \frac{(\sqrt{\rho(\phi)} - \sqrt{\nu(\phi)})^2(\sqrt{\rho(\phi)} +  \sqrt{\nu(\phi)})^2}{\rho(\phi) + \nu(\phi)} \leq 2D_{\textup{H}}^2(\nu,\rho).  
    \end{align*}
    Using this in \pref{eq: triangle inequ} and rearranging gives the desired inequality. 
\end{proof}

\begin{lemma}\label{lem: connecting two digdec hybrid}
    Suppose that $(\calM,\Phi)$ satisfy \pref{assum: avg err} with estimation function $\ellest_h(\phi; o_h)_j = f_\phi(s_h,a_h; \bm{e}_j) - \varphi(s_h,a_h)^\top \bm{e}_j - f_\phi(s_{h+1}; \bm{e}_j)$. Furthermore, assume that $(\calM, \Phi)$ is Bellman complete (\pref{def: adv bc}). Then \pref{assum: squared est err} holds with $\err_h(\phi', \phi; o_h) = \sum_{j=1}^d (f_{\phi'}(s_h,a_h; \bm{e}_j) - \varphi(s_h,a_h)^\top \bm{e}_j - f_{\phi}(s_{h+1}; \bm{e}_j))^2$ and 
    \begin{align*} 
        \mfdec_\eta^{\Phi, \tilD_{\sbe}} \leq \mfdec_\eta^{\Phi, \tilD_{\bi}}. 
    \end{align*}
\end{lemma}
\begin{proof}
   The proof is similar to that in the stochastic setting (\pref{lem: connecting two digdec}). 
\end{proof}

\begin{lemma}\label{lem: reward same}
Under \pref{assum: unique mapping} and \pref{assum: known feature}, if $P, P'\in\phi$, then they share the same $d\times H$ dimensional vector: 
\begin{align*}
    \Big(\E^{\pi_\phi, P}\left[\varphi(s_h,a_h)\right]\Big)_{h\in[H]} = \Big(\E^{\pi_\phi, P'}\left[\varphi(s_h,a_h)\right]\Big)_{h\in[H]}
\end{align*}
  
\end{lemma}
\begin{proof}
    Given a linear reward with known feature (\pref{assum: known feature}), we have $R(s_h,a_h) = \varphi(s_h,a_h)^\top \theta_h(R)$ where $\varphi$ is a known feature. For any $P, R, \pi$, we have
\begin{align*}
    V_{P, R}(\pi) = \sum_{h=1}^H \E^{\pi, P}\left[\varphi(s_h,a_h)^\top\theta_h(R)\right].  
\end{align*}
Fix a $\phi$ and consider $P, P'\in\phi$. By \pref{assum: known feature}, $V_{P,R}(\pi_\phi) = V_{P', R}(\pi_\phi)$ for any $R$. For each $h$, by instantiating $\theta_h(R)$ as all basis vectors in the $d$ dimensional space, we prove that  $\E^{\pi_\phi, P}\left[\varphi(s_h,a_h)_j\right] = \E^{\pi_\phi, P'}\left[\varphi(s_h,a_h)_j\right]$ for any $h\in[H]$ and any $j\in[d]$. 
\end{proof}

\begin{definition}
    We define several quantities that will be reused in \pref{app: hybrid bilinear } for hybrid bilinear classes and \pref{app: coverability hybrid} for coverable MDPs. 
    We fix $\alpha\in[0,1]$, and define $\pi^\alpha$ as the policy that in every step $h=1,2,\ldots, H$ chooses $\pi$ with probability $1-\frac{\alpha}{H}$ and chooses $\unif(\calA)$ with probability $\frac{\alpha}{H}$. We also fix $\tilD$, which will be instantiated as $\tilD_\hbi$ and $\tilD_{\hsbe}$ in later subsections. 

    With them, we define (with $M=(P,R)$)
    \begin{align*}
        \tra_\eta^{\Phi, \tilD}(\nu) 
        &= \alpha + \E_{M \sim \nu}\E_{\phi \sim \nu} \E_{\phi'\sim \nu}\left[  V_{\phi, R}(\pi_{\phi}) - V_M(\pi_\phi) - \frac{1}{9\eta} \tilD^{\pi_{\phi'}^\alpha}(\phi\|M)\right]
\\
\qcom_\eta^{\Phi, \tilD}(\nu) &= 6\sqrt{dH}\sqrt{3\E_{\phi \sim \nu}\E_{(P,R) \sim \nu}\left[\left(V_{\phi, R}(\pi_{\phi}) - V_{P,R}(\pi_{\phi}) \right)^2\right]} - \frac{2}{9\eta}\E_{\phi' \sim \nu}\E_{M \sim \nu}\E_{\phi \sim \nu}\left[\tilD^{\pi_{\phi'}^\alpha}(\phi\|M)\right]
\\
\rew_\eta^{\Phi, \tilD}(\nu) 
        &= \E_{(M, \pi^\star) \sim \nu}\E_{\phi \sim \nu} \E_{\phi'\sim \nu} \\
        &\qquad\qquad  \left[ V_M(\pi^\star) - V_{\phi, R}(\pi_{\phi}) - \frac{1}{\eta}\E_{o\sim M(\cdot|\pi_{\phi'}^\alpha)}\left[\KL(\nu_{\bo{\phi}}(\cdot|\pi_{\phi'}^\alpha, o), \nu_{\pmb{\phi}})\right] - \frac{2}{9\eta} \tilD^{\pi_{\phi'}^\alpha}(\phi\|M)\right]
    \end{align*}
\end{definition}

\begin{lemma}\label{lem: air by comp}
\begin{align*}
    \min_p \max_\nu \air_\eta^{\Phi, D}(p, \nu; \rho) \leq \max_\nu \tra_\eta^{\Phi, \tilD}(\nu) + \max_\nu \rew_\eta^{\Phi, \tilD}(\nu).  
\end{align*}
\end{lemma}
\begin{proof}
    \begin{align*}
    & \air_\eta^{\Phi, D}(p, \nu; \rho)\\
    &= \E_{\pi\sim p}\E_{(M, \pi^\star) \sim \nu}\left[ V_{M}(\pi^\star) - V_M(\pi) - \frac{1}{\eta}\E_{o\sim M(\cdot|\pi)}\left[\KL(\nu_{\bo{\phi}}(\cdot|\pi, o), \rho)\right] - \frac{1}{\eta} \E_{\phi\sim \rho}\left[\tilD^\pi(\phi\|M)\right]  \right]
    \\&= \E_{\pi\sim p}\E_{(M, \pi^\star) \sim \nu}\left[ V_{M}(\pi^\star) - V_M(\pi) - \frac{1}{\eta}\E_{o\sim M(\cdot|\pi)}\left[\KL\left(\nu_{\bo{\phi}}(\cdot|\pi, o), \nu_{\bo{\phi}}\right)\right] - \frac{1}{\eta} \E_{\phi\sim \rho}\left[\tilD^\pi(\phi\|M)\right]  - \frac{1}{\eta}\KL\left(\nu_{\pmb{\phi}}, \rho\right) \right]
    \\&\le \E_{\pi\sim p}\E_{(M, \pi^\star) \sim \nu}\left[ V_{M}(\pi^\star) - V_M(\pi) - \frac{1}{\eta}\E_{o\sim M(\cdot|\pi)}\left[\KL\left(\nu_{\bo{\phi}}(\cdot|\pi, o), \nu_{\bo{\phi}}\right)\right]\right. 
    \\& \qquad \qquad \qquad  \left. - \frac{1}{3\eta} \E_{\phi\sim \nu}\left[\tilD^\pi(\phi\|M)\right]  + \frac{2}{3\eta}D_{\textup{H}}^2(\nu_{\pmb{\phi}}, \rho) - \frac{1}{\eta}\KL\left(\nu_{\pmb{\phi}}, \rho\right) \right] \tag{\pref{lem: triangular discri}}
    \\&\le \E_{\pi\sim p}\E_{M \sim \nu}\E_{\phi \sim \nu}\left[  V_{\phi, R}(\pi_{\phi}) - V_M(\pi) - \frac{1}{9\eta} \tilD^\pi(\phi\|M)\right]
    \\& \qquad + \E_{\pi\sim p}\E_{(M, \pi^\star) \sim \nu}\E_{\phi \sim \nu}\left[ V_M(\pi^\star) - V_{\phi, R}(\pi_{\phi}) - \frac{1}{\eta}\E_{o\sim M(\cdot|\pi)}\left[\KL(\nu_{\bo{\phi}}(\cdot|\pi, o), \nu_{\pmb{\phi}})\right] - \frac{2}{9\eta} \tilD^\pi(\phi\|M)\right]. 
\end{align*}    
    We have $\min_p \max_\nu \air_\eta^{\Phi, D}(p, \nu; \rho) = \max_\nu \min_p \air_\eta^{\Phi, D}(p, \nu; \rho)$ because $\air$ is convex in $p$ and concave in~$\nu$. After the min-max swap, for each $\nu$, we choose $p$ to be such that $\pi\sim p$ is equivalent to first sampling $\phi'\sim \nu$ and then setting $\pi=\pi_{\phi'}^\alpha$. This gives 
    \begin{align*}   
        &\min_p \max_\nu \air_\eta^{\Phi, D}(p, \nu; \rho) 
        \\&\le \max_\nu \E_{\phi'\sim \nu}\E_{M \sim \nu}\E_{\phi \sim \nu}\left[  V_{\phi, R}(\pi_{\phi}) - V_M(\pi_{\phi'}^\alpha) - \frac{1}{9\eta} \tilD^{\pi_{\phi'}^\alpha}(\phi\|M)\right]
    \\& \qquad + \E_{\phi'\sim \nu}\E_{(M, \pi^\star) \sim \nu}\E_{\phi \sim \nu}\left[ V_M(\pi^\star) - V_{\phi, R}(\pi_{\phi}) - \frac{1}{\eta}\E_{o\sim M(\cdot|\pi_{\phi'}^\alpha)}\left[\KL(\nu_{\bo{\phi}}(\cdot|\pi_{\phi'}^\alpha, o), \nu_{\pmb{\phi}})\right] - \frac{2}{9\eta} \tilD^{\pi_{\phi'}^\alpha}(\phi\|M)\right] 
    \\&\leq \max_\nu \tra_\eta^{\Phi, \tilD}(\nu) + \max_\nu \rew_\eta^{\Phi, \tilD}(\nu).  
    \end{align*}
\end{proof}

\begin{lemma}\label{lem: bound Rcompe}
    \begin{align*}
       \rew_\eta^{\Phi, \tilD}(\nu) \leq O(\eta dH+ \alpha) + \qcom_\eta^{\Phi, \tilD}(\nu). 
    \end{align*}

\end{lemma}
\begin{proof}
By \pref{lem: reward same} we can define with any $P\in\phi$, 
\begin{align*}
    X_h(\phi) = \E^{\pi_\phi, P}\left[\varphi(s_h,a_h)\right]. 
\end{align*}
Furthermore, define 
\begin{align*}
    X(\phi) &= (X_h(\phi))_{h\in[H]}\in\mathbb{R}^{dH}, \\
    \theta(R) &= (\theta_h(R))_{h\in[H]}\in\mathbb{R}^{dH}. 
\end{align*}
With this, we have 
\begin{align*}
    &\E_{(M, \pi^\star)\sim \nu} \E_{\phi\sim \nu}[V_M(\pi^\star) - V_{\phi, R}(\pi_\phi)]   \\ 
    &=\E_{\phi \sim \nu}\E_{R \sim \nu(\cdot|\phi)}\left[ V_{\phi, R}(\pi_{\phi})\right] - \E_{\phi \sim \nu}\E_{R \sim \nu}\left[V_{\phi, R}(\pi_{\phi})\right]
    \\&= \E_{\phi \sim \nu}\left[X(\phi)^\top\left(\E_{R \sim \nu(\cdot|\phi)}\left[\theta(R)\right] - \E_{R \sim \nu}\left[\theta(R)\right]\right)\right]
    \\&\le  \E_{\phi \sim \nu}\left[\left\|X(\phi)\right\|_{\Sigma_{\nu}^{-1}}\left\|\E_{R \sim \nu(\cdot|\phi)}\left[\theta(R)\right] - \E_{R \sim \nu}\left[\theta(R)\right]\right\|_{\Sigma_{\nu}}\right] \tag{$\Sigma_{\nu} = \E_{\phi \sim \nu}\left[X(\phi)X(\phi)^\top\right]$}
    \\&\le \sqrt{ \E_{\phi \sim \nu}\left[\left\|X(\phi)\right\|_{\Sigma_{\nu}^{-1}}^2\right]} \sqrt{\E_{\phi \sim \nu}\left[\left\|\E_{R \sim \nu(\cdot|\phi)}\left[\theta(R)\right] - \E_{R \sim \nu}\left[\theta(R)\right]\right\|_{\Sigma_{\nu}}^2\right]}
    \\&= \sqrt{dH}\sqrt{\E_{\phi' \sim \nu}\E_{\phi \sim \nu}\left[\left(X(\phi')^\top\E_{R \sim \nu(\cdot|\phi)}\left[\theta(R)\right] - X(\phi')^\top\E_{R \sim \nu}\left[\theta(R)\right]\right)^2\right]}
    \\&\le 3\sqrt{dH}\sqrt{\underbrace{\E_{\phi' \sim \nu}\E_{\phi \sim \nu}\left[\left(\E_{(P,R) \sim \nu(\cdot|\phi)}\left[V_{P,R}(\pi_{\phi'})\right] - \E_{(P,R) \sim \nu}\left[V_{P,R}(\pi_{\phi'})\right]\right)^2\right]}_{\textbf{Div1}}} 
    \\&\quad +  3\sqrt{dH}\sqrt{\underbrace{\E_{\phi' \sim \nu}\E_{\phi \sim \nu}\left[\left(X(\phi')^\top\E_{R \sim \nu(\cdot|\phi)}\left[\theta(R)\right] - \E_{(P,R) \sim \nu(\cdot|\phi)}\left[V_{P,R}(\pi_{\phi'})\right]\right)^2\right]}_{\textbf{Div2}}} 
    \\&\qquad +  3\sqrt{dH}\sqrt{\underbrace{\E_{\phi' \sim \nu}\E_{\phi \sim \nu}\left[\left(X(\phi')^\top\E_{R \sim \nu}\left[\theta(R)\right] - \E_{(P,R) \sim \nu}\left[V_{P,R}(\pi_{\phi'})\right]\right)^2\right]}_{\textbf{Div3}}}.  \numberthis  \label{eq: plug bkac} 
\end{align*}

For any observation $o  = (s_1, a_1, r_1, \cdots, s_H, a_H, r_H)$, let $r(o) = \sum_{h=1}^H r_h$, we have
\begin{align*}
    \textbf{Div1} &= \E_{\phi' \sim \nu}\E_{\phi \sim \nu}\left[\left(\E_{(P,R) \sim \nu(\cdot|\phi)}\left[V_{P,R}(\pi_{\phi'})\right] - \E_{(P,R) \sim \nu}\left[V_{P,R}(\pi_{\phi'})\right]\right)^2\right]
    \\&\le 2\E_{\phi' \sim \nu}\E_{\phi \sim \nu}\left[\left(\E_{(P,R) \sim \nu(\cdot|\phi)}\left[V_{P,R}(\pi_{\phi'}^\alpha)\right] - \E_{(P,R) \sim \nu}\left[V_{P,R}(\pi_{\phi'}^\alpha)\right]\right)^2\right] + 8\alpha^2
    \\&= 2\E_{\phi' \sim \nu}\E_{\phi \sim \nu}\left[\left(\E_{(P,R) \sim \nu(\cdot|\phi)}\left[\E_{o \sim M_{P,R}(\cdot|\pi_{\phi'}^\alpha)}\left[r(o)\right]\right] - \E_{(P,R) \sim \nu}\left[\E_{o \sim M_{P,R}(\cdot|\pi_{\phi'}^\alpha)}\left[r(o)\right]\right]\right)^2\right] + 8\alpha^2
    \\&= 2\E_{\phi' \sim \nu}\E_{\phi \sim \nu}\left[\left(\E_{o \sim \nu(\cdot|\phi, \pi_{\phi'}^\alpha)}\left[r(o)\right] - \E_{o \sim \nu(\cdot|\pi_{\phi'}^\alpha)}\left[r(o)\right] \right)^2\right] + 8\alpha^2
    \\&\le 2\E_{\phi' \sim \nu}\E_{\phi \sim \nu}\left[\left(\sum_{o}\left|\nu(o|\phi, \pi_{\phi'}^\alpha) - \nu(o|\pi_{\phi'}^\alpha) \right|\right)^2\right] + 8\alpha^2
    \\&= 8\E_{\phi' \sim \nu}\E_{\phi \sim \nu}\left[D_{\text{TV}}^2\left(\nu_{\pmb{o}}(\cdot|\phi, \pi_{\phi'}^\alpha), \nu_{\pmb{o}}(\cdot|\pi_{\phi'}^\alpha) \right)\right]  + 8\alpha^2
    \\&\le 8\E_{\phi' \sim \nu}\E_{\phi \sim \nu}\left[\KL\left(\nu_{\pmb{o}}(\cdot|\phi, \pi_{\phi'}^\alpha), \nu_{\pmb{o}}(\cdot|\pi_{\phi'}^\alpha) \right)\right] + 8\alpha^2
    \\&= 8\E_{\phi' \sim \nu}\E_{M\sim \nu} \E_{o\sim M(\cdot|\pi^\alpha_{\phi'})}\left[\KL\left(\nu_{\pmb{\phi}}(\cdot|\pi_{\phi'}^\alpha, o), \nu_{\pmb{\phi}} \right)\right] + 8\alpha^2.  
\end{align*}
On the other hand
\begin{align*}
     \textbf{Div2} &= \E_{\phi' \sim \nu}\E_{\phi \sim \nu}\left[\left(X(\phi')^\top\E_{R \sim \nu(\cdot|\phi)}\left[\theta(R)\right] - \E_{(P,R) \sim \nu(\cdot|\phi)}\left[V_{P,R}(\pi_{\phi'})\right] \right)^2\right]
     \\&= \E_{\phi' \sim \nu}\E_{\phi \sim \nu}\left[\left(\E_{R \sim \nu(\cdot|\phi)}\left[V_{\phi', R}(\pi_{\phi'})\right] - \E_{(P,R) \sim \nu(\cdot|\phi)}\left[V_{P,R}(\pi_{\phi'})\right] \right)^2\right]
     \\&\le \E_{\phi' \sim \nu}\E_{\phi \sim \nu}\E_{(P,R) \sim \nu(\cdot|\phi)}\left[\left(V_{\phi', R}(\pi_{\phi'}) - V_{P,R}(\pi_{\phi'}) \right)^2\right]
     \\&= \E_{\phi' \sim \nu}\E_{(P,R) \sim \nu}\left[\left(V_{\phi', R}(\pi_{\phi'}) - V_{P,R}(\pi_{\phi'}) \right)^2\right]
\end{align*}
Similarly,
\begin{align*}
     \textbf{Div3} &= \E_{\phi' \sim \nu}\E_{\phi \sim \nu}\left[\left(X(\phi')^\top\E_{R \sim \nu}\left[\theta(R)\right] - \E_{(P,R) \sim \nu}\left[V_{P,R}(\pi_{\phi'})\right] \right)^2\right]
     \\&= \E_{\phi' \sim \nu}\E_{\phi \sim \nu}\left[\left(\E_{R \sim \nu}\left[V_{\phi', R}(\pi_{\phi'})\right] - \E_{(P,R) \sim \nu}\left[V_{P,R}(\pi_{\phi'})\right] \right)^2\right]
     \\&\le \E_{\phi' \sim \nu}\E_{(P,R) \sim \nu}\left[\left(V_{\phi', R}(\pi_{\phi'}) - V_{P,R}(\pi_{\phi'}) \right)^2\right]
\end{align*}
Combining these equations back to \pref{eq: plug bkac} and using the definition of $\rew_\eta^{\Phi, \tilD}(\nu)$, we have
\begin{align*}
&\rew_\eta^{\Phi, \tilD}(\nu) 
\\&\le 3\sqrt{8dH\E_{\phi' \sim \nu}\E_{M\sim \nu} \E_{o\sim M(\cdot|\pi^\alpha_{\phi'})}\left[\KL\left(\nu_{\pmb{\phi}}(\cdot|\pi_{\phi'}^\alpha, o), \nu_{\pmb{\phi}} \right)\right] + 8\alpha^2} 
\\&\qquad + 6\sqrt{dH}\sqrt{3 \E_{\phi \sim \nu}\E_{(P,R) \sim \nu}\left[\left(V_{\phi, R}(\pi_{\phi}) - V_{P,R}(\pi_{\phi}) \right)^2\right]}
\\& \qquad  - \frac{1}{\eta}\E_{\phi' \sim \nu}\E_{M\sim \nu} \E_{o\sim M(\cdot|\pi^\alpha_{\phi'})}\left[\KL\left(\nu_{\pmb{\phi}}(\cdot|\pi_{\phi'}^\alpha, o), \nu_{\pmb{\phi}} \right)\right] - \frac{2}{9\eta}\E_{\phi' \sim \nu}\E_{M \sim \nu}\E_{\phi \sim \nu}\left[\tilD^{\pi_{\phi'}^\alpha}(\phi\|M)\right]
\\&\le \order\left(\eta dH + \alpha\right) + 6\sqrt{dH}\sqrt{3\E_{\phi \sim \nu}\E_{(P,R) \sim \nu}\left[\left(V_{\phi, R}(\pi_{\phi}) - V_{P,R}(\pi_{\phi}) \right)^2\right]} - \frac{2}{9\eta}\E_{\phi' \sim \nu}\E_{M \sim \nu}\E_{\phi \sim \nu}\left[\tilD^{\pi_{\phi'}^\alpha}(\phi\|M)\right] 
\\&= \order\left(\eta dH + \alpha\right) +  \qcom_\eta^{\Phi, \tilD}(\nu). 
\end{align*}


\end{proof}



    





\subsection{Relating $\mfdec$ to hybrid bilinear rank}\label{app: hybrid bilinear }

\begin{assumption}[Hybrid bilinear class \citep{liu2025decision}]
\label{assum:hybrid bilinear}
A model class $\calM$ and its associated $\Phi$ satisfying \pref{assum: unique mapping} is a hybrid bilinear class with rank $d$ if there exists functions $X_h: \Phi \times \mathcal{P} \rightarrow \mathbb{R}^d$ and  $W_h: \Phi \times \mathcal{R} \times \mathcal{P} \rightarrow \mathbb{R}^d$ for all $h\in[H]$ such that 
\begin{enumerate}[leftmargin=1.5em, nosep]
 \setlength{\itemsep}{0pt} 
    \setlength{\parskip}{0pt} 
    \setlength{\topsep}{0pt}
    \item For any $M=(P,R)\in\phi$, it holds that $W_h(\phi, \tilde{R}; P)=0$ for any $\tilde{R}\in\calR$. 
    \item For any $\phi \in \Phi$ and any $(P,R) \in \calM$, 
    \begin{align*}
        \left|V_{\phi, R}(\pi_\phi) - V_{P,R}(\pi_\phi)\right| \leq \sum_{h=1}^H \left|\langle X_h(\phi; P), W_h(\phi, R; P)\rangle\right|.
    \end{align*}
    \item For every policy $\pi$, there exists an estimation policy $\pi^{\esttt}$. Also, there exists a discrepancy function $\ellest_h: \Phi \times \calR\times \calO  \rightarrow \mathbb{R}$ such that for any $\phi', \phi \in \Phi$ and any $M=(P,R)\in\calM$,  
    \begin{align*}
        \left|\langle X_h(\phi'; P), W_h(\phi, R; P) \rangle\right| = \left|\E^{\pi_{\phi'} \,\circ_h\, \pi_{\phi'}^{\esttt},\, P} \left[\ellest_h(\phi, R; o_h)\right]\right| 
    \end{align*}
    where $o_h = (s_h,a_h, r_h, s_{h+1})$ and $\pi \circ_h \pi^{\esttt}$ denotes a policy that plays $\pi$ for the first $h-1$ steps and plays policy $\pi^{\esttt}$\, at the $h$-th step. 
\end{enumerate}
We call it an on-policy bilinear class if $\pi^{\esttt}=\pi$ for all $\pi\in\Pi$, and otherwise an off-policy bilinear class. We denote by $\pi^\alpha$ the policy that in every step $h=1,\ldots, H$ chooses $\pi$ with probability $1-\frac{\alpha}{H}$ and chooses $\pi^{\esttt}$ with probability $\frac{\alpha}{H}$. 
\end{assumption}

\begin{lemma}
   Hybrid bilinear classes (\pref{assum:hybrid bilinear}) with known-feature linear reward (\pref{assum: known feature}) satisfy \pref{assum: avg err} with $N=d$. 
\end{lemma}
\begin{proof}
   With the estimation function $\ell_h(\phi, R; o_h)$ defined in \pref{assum:hybrid bilinear}, we define for $j\in[d]$, 
   \begin{align*}
       \ell_h(\phi; o_h)_j = \ell_h(\phi, \bm{e}_j; o_h), 
   \end{align*}
   where $\bm{e}_j$ as a reward represents the reward function defined as $R(s,a)=\varphi(s,a)^\top \bm{e}_j = \varphi(s,a)_j$. 

   For any $\phi'\in\Phi$ and any $M=(P,R)\in\phi$, 
   \begin{align*}
       &\left|\E^{\pi_{\phi'} \circ_h \pi_{\phi'}^{\esttt}, P } \left[\ell_h (\phi; o_h)_j \right]\right| \\
       &= \left|\E^{\pi_{\phi'} \circ_h \pi_{\phi'}^{\esttt}, P } \left[\ell_h (\phi, \bm{e}_j; o_h) \right]\right| \\
       &= \left|\inner{X_h(\phi'; P), W_h(\phi, \bm{e}_j; P)}\right|   \tag{by \pref{assum:hybrid bilinear}.3} \\
       &= 0.  \tag{by \pref{assum:hybrid bilinear}.1}
   \end{align*}
\end{proof}

\begin{lemma}[Lemma~20 of \cite{liu2025decision}]\label{lem: bound Pcomp}
   Let $(\calM, \Phi)$ be a hybrid bilinear class (\pref{assum:hybrid bilinear}). Then 
\begin{itemize}
    \item $\max_\nu \tra^{\Phi, \tilD_\hbi}_\eta(\nu) \le O(B^2H^2d\eta)$ in the on-policy case. 
    \item  
    $\max_\nu \tra^{\Phi, \tilD_\hbi}_\eta(\nu) \le O(\alpha + B^2 H^3 d\eta/\alpha)$ in the off-policy case.\footnote{As in \pref{foot: foot}, the bounds are different from \cite{liu2025decision}'s as we adopt a different scaling. } 
\end{itemize}
\end{lemma}

\begin{lemma}\label{lem: bound Qcomp}

Let $(\calM, \Phi)$ be a hybrid bilinear class (\pref{assum:hybrid bilinear}).  Then 
\begin{itemize}
    \item $\max_\nu \qcom^{\Phi, \tilD_\hbi}_\eta(\nu) \le \order\left(\left( B^2H^5d^3 \eta \right)^{\frac{1}{3}}\right)$ in the on-policy case. 
    \item  
    $\max_\nu \qcom^{\Phi, \tilD_\hbi}_\eta(\nu) \le \order\left(\left( B^2H^6 d^3\eta/\alpha\right)^{\frac{1}{3}}\right)$ in the off-policy case.  
\end{itemize}



\label{lem:square_bilinear}
\end{lemma}

\begin{proof}
From the definition of hybrid bilinear class in \pref{assum:hybrid bilinear}, we have
\begin{align*}
     &\E_{\phi \sim \nu}\E_{(P,R) \sim \nu}\left[\left(V_{\phi, R}(\pi_{\phi}) - V_{P,R}(\pi_{\phi}) \right)^2\right] 
     \\&\le \E_{\phi \sim \nu}\E_{(P,R) \sim \nu}\left[\left(\sum_{h=1}^H \left|\langle X_h(\phi; P), W_h(\phi, R; P)\rangle\right| \right)^2\right] 
     \\&\le H\sum_{h=1}^H \E_{\phi \sim \nu}\E_{(P,R) \sim \nu}\left[\left|\langle X_h(\phi; P), W_h(\phi, R; P)\rangle\right|^2\right]. 
\end{align*}
Define $\Sigma_{h,P} = \E_{\phi \sim \nu}\left[X_h(\phi; P)X_h(\phi; P)^\top\right]$. We have
\begin{align*}
    &\E_{\phi \sim \nu} \left[
\left|\langle X_h(\phi; P), W_h(\phi, R; P)\rangle\right|^2
\right] 
\\& \le \E_{\phi \sim \nu} \left[
\left|\langle X_h(\phi; P), W_h(\phi, R; P)\rangle\right|
\right]
\\&\le \sqrt{\E_{\phi \sim \nu}\left[\left\|X_h(\phi; P)\right\|_{\Sigma_{h,P}^{-1}}^2\right]} \sqrt{\E_{\phi \sim \nu}\left[\left\|W_h(\phi, R; P)\right\|_{\Sigma_{h,P}}^2\right]}
\\&= \sqrt{d\E_{\phi \sim \nu}\E_{\phi' \sim \nu}\left[\left(  \E^{\pi_{\phi'}\,\circ_h \, \pi_{\phi'}^{\esttt}, \,P}\left[\ellest_h(\phi, R; o_h)\right]\right)^2\right]}. \tag{\pref{assum:hybrid bilinear}}
\end{align*}
Thus, 
\begin{align*}
    &\sqrt{\E_{\phi \sim \nu}\E_{(P,R) \sim \nu}\left[\left(V_{\phi, R}(\pi_{\phi}) - V_{P,R}(\pi_{\phi}) \right)^2\right]} \\
    &\qquad \quad\le  \sqrt{H\sum_{h=1}^H\E_{(P, R) \sim \nu}\left[\sqrt{d\E_{\phi \sim \nu}\E_{\phi' \sim \nu}\left[\left(  \E^{\pi_{\phi'}\,\circ_h \, \pi_{\phi'}^{\esttt}, \,P}\left[\ellest_h(\phi, R; o_h)\right]\right)^2\right]}\right]}. 
\end{align*}
\textbf{(1)} In the on-policy case, we have $\alpha = 0$ and 
\begin{align*}
      &6\sqrt{3dH\E_{\phi \sim \nu}\E_{(P,R) \sim \nu}\left[\left(V_{\phi, R}(\pi_{\phi}) - V_{P,R}(\pi_{\phi}) \right)^2\right]}   - \frac{2}{9\eta}\E_{\phi' \sim \nu}\E_{M \sim \nu}\E_{\phi \sim \nu}\left[D_{\hbi}^{\pi_{\phi'}}(\phi\|M)\right] 
      \\&\le  6\sqrt{3d^{\frac{3}{2}}H^2\sum_{h=1}^H\E_{(P,R) \sim \nu}\left[\sqrt{\E_{\phi \sim \nu}\E_{\phi' \sim \nu}\left[\left(  \E^{\pi_{\phi'}, P}\left[\ellest_h(\phi, R; o_h)\right]\right)^2\right]}\right]}
      \\&\qquad - \frac{2}{9\eta B^2H}\sum_{h=1}^H \E_{\phi \sim \nu}\E_{\phi' \sim \nu}\E_{(P,R) \sim \nu} \left[\sum_{j=1}^d\left( \E^{\pi_{\phi'}, P}\left[ \ellest_h(\phi; o_h)_j\right]\right)^2\right]
      \\&\le O\left(d^{\frac{3}{2}}H^2\beta\right) + \frac{1}{4\beta}\sum_{h=1}^H\E_{(P,R) \sim \nu}\left[\sqrt{\E_{\phi \sim \nu}\E_{\phi' \sim \nu}\left[\left(  \E^{\pi_{\phi'}, P}\left[\ellest_h(\phi, R; o_h)\right]\right)^2\right]}\right] 
      \\&\qquad - \frac{2}{9\eta B^2 H}\sum_{h=1}^H \E_{\phi \sim \nu}\E_{\phi' \sim \nu}\E_{(P,R) \sim \nu} \left[\left( \E^{\pi_{\phi'}, P}\left[\ellest_h(\phi, R ; o_h)\right]\right)^2\right]
      \\&\le \order\left(d^{\frac{3}{2}}H^2\beta + \frac{\eta B^2H }{\beta^2}\right) = \order\left(\left( B^2H^5d^3 \eta \right)^{\frac{1}{3}}\right).  \tag{choosing optimal $\beta$}
\end{align*}
\textbf{(2)} For the off-policy case, we have 
\begin{align*}
    &6\sqrt{3dH\E_{\phi \sim \nu}\E_{(P,R) \sim \nu}\left[\left(V_{\phi, R}(\pi_{\phi}) - V_{P,R}(\pi_{\phi}) \right)^2\right]}   - \frac{2}{9\eta}\E_{\phi'\sim \nu}\E_{M \sim \nu}\E_{\phi \sim \nu}\left[D_{\hbi}^{\pi_{\phi'}^\alpha}(\phi\|M)\right] 
    \\&\le 6\sqrt{d^{\frac{3}{2}}H^2\sum_{h=1}^H\E_{(P,R) \sim \nu}\left[\sqrt{\E_{\phi \sim \nu}\E_{\phi' \sim \nu}\left[\left(  \E^{\pi_{\phi'} \circ_h \pi_{\phi'}^{\esttt}, P}\left[\ellest_h(\phi, R; o_h)\right]\right)^2\right]}\right]} 
    \\&\qquad \qquad - \frac{2}{9\eta B^2H}\sum_{h=1}^H \E_{\phi \sim \nu}\E_{\phi' \sim \nu}\E_{(P,R) \sim \nu} \left[\sum_{j=1}^d\left( \E^{\pi_{\phi'}^\alpha, P}\left[ \ellest_h(\phi; o_h)_j\right]\right)^2\right]
    \\&\le O\left(d^{\frac{3}{2}}H^2\beta\right) + \frac{1}{4\beta}\sum_{h=1}^H\E_{(P,R) \sim \nu}\left[\sqrt{\E_{\phi \sim \nu}\E_{\phi' \sim \nu}\left[\left(  \E^{\pi_{\phi'}\circ_h \pi_{\phi'}^{\esttt}, P}\left[\ellest_h(\phi, R; o_h)\right]\right)^2\right]}\right] 
      \\&\qquad \qquad - \frac{\alpha}{3H} \cdot \frac{2}{9\eta B^2 H}\sum_{h=1}^H \E_{\phi \sim \nu}\E_{\phi' \sim \nu}\E_{(P,R) \sim \nu} \left[\left( \E^{\pi_{\phi'} \circ_h \pi_{\phi'}^{\esttt}, P}\left[\ellest_h(\phi, R ; o_h)\right]\right)^2\right]
    \\&\le \order\left(d^{\frac{3}{2}}H^2\beta + \frac{\eta B^2 H^2}{\alpha \beta^2}\right) = \order\left(\left( B^2H^6 d^3\eta/\alpha\right)^{\frac{1}{3}}\right),  \tag{with the optimal $\beta$} 
\end{align*}
where the second-to-last inequality is because with probability $(1-\frac{\alpha}{H})^{h-1}\frac{\alpha}{H}\geq \frac{\alpha}{3 H}$, policy $\pi^\alpha_{\phi'}$ chooses the policy $\pi_{\phi'}\circ_h \pi_{\phi'}^{\esttt}$. 
\end{proof}

\begin{lemma}\label{lem: digdec_adv < d hbi}
    Let $(\calM, \Phi)$ be a hybrid bilinear class (\pref{assum:hybrid bilinear}). Then 
    \begin{itemize}[leftmargin=1em, nosep]
 \setlength{\itemsep}{0pt} 
    \setlength{\parskip}{0pt} 
    \setlength{\topsep}{0pt}
    \item $\mfdec_\eta^{\Phi,\tilD_{\hbi}}  \le O\left( B^2H^2d\eta + \left( B^2H^5d^3 \eta \right)^{\frac{1}{3}} \right)$ in the on-policy case; 
    \item $\mfdec_\eta^{\Phi,\tilD_{\hbi}}  \le O\left( \sqrt{B^2H^3 d \eta } + \left(B^2H^6 d^3 \eta\right)^{\frac{1}{4}} \right)$ in the off-policy case. 
\end{itemize}
\end{lemma}

\begin{proof}
   This can be obtained by directly combining \pref{lem: air by comp}, \pref{lem: bound Rcompe}, \pref{lem: bound Pcomp}, \pref{lem: bound Qcomp}.  In the on-policy case, 
    \begin{align*}
        \mfdec_\eta^{\Phi,\tilD_{\hbi}}  = O\left( B^2H^2d\eta + \left( B^2H^5d^3 \eta \right)^{\frac{1}{3}} \right). 
    \end{align*}
    In the off-policy case, 
    \begin{align*}
        \mfdec_\eta^{\Phi,\tilD_{\hbi}} 
        &= O\left( \alpha + B^2 H^3 d\eta/\alpha + \left(B^2H^6 d^3\eta/\alpha\right)^{\frac{1}{3}}  \right) \\
        &= O\left( \sqrt{B^2H^3 d \eta } + \left(B^2H^6 d^3 \eta\right)^{\frac{1}{4}} \right).   \tag{with optimal $\alpha$}
    \end{align*}
\end{proof}

\subsection{Relating $\mfdec$ to coverability under Bellman completeness}\label{app: coverability hybrid}

\begin{lemma}\label{lem: bound Pcomp cover}
For hybrid MDPs with Bellman completeness and coverability bounded by $d$, it holds that 
\begin{align*}
    \max_\nu \tra^{\Phi, \tilD_\hsbe}_\eta(\nu) \le O\left(\eta d B^2 H^2\right). 
\end{align*}
\end{lemma}

\begin{proof}
    For $M=(P,R)$, define 
    \begin{align*}
        g_h(s,a, \phi; R, P)  &= f_{\phi}(s,a; R) - R(s,a) - \E_{s'\sim P(\cdot|s,a)}[f_{\phi}(s'; R)], \\
        d_h^{\nu, P}(s,a) &= \E_{\phi\sim \nu}\left[d_h^{\pi_\phi, P}(s,a) \right]. 
    \end{align*}
    By the AM-GM inequality, for any $\lambda>0$, 
     \begin{align}
        &\E_{\phi\sim \nu} \E^{\pi_\phi, P} \left[g_h(s_h,a_h, \phi; R,P)\right] \nonumber
        \\&= \E_{\phi\sim \nu} \E_{(s,a)\sim d_h^{\pi_\phi, P}} \left[g_h(s,a, \phi; R,P)\right]  \nonumber \\
        &= \E_{\phi\sim \nu} \E_{(s,a)\sim d_h^{\nu, P}} \left[ \frac{d_h^{\pi_\phi, P}(s,a)}{d_h^{\nu, P}(s,a)} g_h(s,a, \phi; R,P)\right]  \nonumber \\
        &\leq  \E_{\phi\sim \nu} \E_{(s,a)\sim d_h^{\nu, P}} \left[ \frac{\lambda}{4}\frac{d_h^{\pi_\phi, P}(s,a)^2}{d_h^{\nu, P}(s,a)^2} + \frac{1}{\lambda}g_h(s,a, \phi; R,P)^2\right]  \nonumber \\
        &= \frac{\lambda}{4} \E_{\phi\sim \nu}  \left[\sum_{s,a} \frac{d_h^{\pi_\phi, P}(s,a)^2}{d_h^{\nu, P}(s,a)}\right] + \frac{1}{\lambda}\E_{\phi\sim \nu} \E_{\phi'\sim \nu}\E^{\pi_{\phi'}, M}\left[g_h(s_h,a_h, \phi, R,P)^2\right].  \label{eq:cov_start}
    \end{align}
    Note that 
    \begin{align*}
        \sum_{h=1}^H \E_{\phi\sim \nu} \E^{\pi_\phi, P} [g_h(s_h,a_h, \phi; R,P)] = \E_{\phi\sim \nu} \left[ V_{\phi, R}(\pi_\phi) - V_M(\pi_\phi) \right],  
    \end{align*}
    and 
    \begin{align*}
       &\sum_{h=1}^H \E^{\pi_{\phi'}, P}\left[g_h(s_h,a_h, \phi; R,P)^2\right]
       \\&\leq \sum_{h=1}^H \sum_{j=1}^d \E^{\pi_{\phi'}, P}\left[g_h(s_h,a_h, \phi; \bm{e}_j,P)^2\right] 
       \\&= \sum_{h=1}^H \sum_{j=1}^d  \E^{\pi_{\phi'}, P}\left[ \left(f_{\phi}(s_h,a_h; \bm{e}_j) - \varphi(s_h,a_h)^\top \bm{e}_j - \E_{s'\sim P(\cdot|s,a)}[f_{\phi}(s'; \bm{e}_j)]\right)^2\right], \\&= \sum_{h=1}^H \sum_{j=1}^d \E^{\pi_{\phi'}, P} \left[ \left(f_{\phi}(s_h,a_h; \bm{e}_j) - f_{\calT_M \phi}(s_h,a_h; \bm{e}_j)\right)^2\right]
       \\&= \sum_{h=1}^H \E^{\pi_{\phi'}, P} \left[ \left\|f_{\phi}(s_h,a_h) - f_{\calT_M \phi}(s_h,a_h)\right\|^2\right]
       \\&= \sum_{h=1}^H \E^{\pi_{\phi'}, P} \left[\err_h(\phi, \phi; o_h) - \err_h(\calT_M \phi, \phi; o_h)\right] \tag{by \pref{eq: second calcul}}
       \\&= B^2 H \tilD_{\hsbe}^{\pi_{\phi'}} (\phi\|M).  \numberthis \label{eq:BC-equal}
    \end{align*}
Thus, 
\begin{align*}
 &\tra_\eta^{\Phi, \tilD_{\hsbe}}(\nu) 
      \\&= \E_{M \sim \nu}\E_{\phi' \sim \nu}\E_{\phi \sim \nu}\left[V_{\phi, R}(\pi_{\phi}) - V_M(\pi_{\phi}) - \frac{1}{\eta}\tilD^{\pi_{\phi'}}_{\hsbe}\left(\phi \|M\right)\right] 
      \\&\le   \E_{\phi \sim \nu}\E_{\phi' \sim \nu}\E_{M \sim \nu}\left[ \sum_{h=1}^H  \E^{\pi_\phi, P} [g_h(s_h,a_h, \phi; R,P)] - \frac{1}{\eta B^2 H} \sum_{h=1}^H \E^{\pi_{\phi'}, P}\left[g_h(s_h,a_h, \phi, R,P)^2\right] \right]
      \\&\le \frac{\eta B^2 H}{4} \E_{M \sim \nu}\E_{\phi\sim \nu}  \left[\sum_{h=1}^H\sum_{s,a} \frac{d_h^{\pi_\phi, P}(s,a)^2}{d_h^{\nu, P}(s,a)}\right].  \tag{by \pref{eq:cov_start}}
\end{align*}
    Let $\mu_h^P$ be any occupancy measure over layer $h$ that depends on $P$. Then 
    \begin{align}
        \E_{\phi\sim \nu}  \left[\sum_{s,a} \frac{d_h^{\pi_\phi, P}(s,a)^2}{d_h^{\nu, P}(s,a)}\right] \nonumber
        &= \E_{\phi\sim \nu}  \left[\sum_{s,a} \frac{d_h^{\pi_\phi, P}(s,a) \mu_h^P(s,a)}{d_h^{\nu, P}(s,a)}\cdot \frac{d^{\pi_\phi, P}(s,a)}{\mu_h^P(s,a)}\right] \nonumber \\
        &\leq \E_{\phi\sim \nu}  \left[\sum_{s,a} \frac{d_h^{\pi_\phi, P}(s,a) \mu_h^P(s,a)}{d_h^{\nu, P}(s,a)} \right] \cdot\max_{s,a, \pi}\frac{d_h^{\pi, P}(s,a)}{\mu_h^P(s,a)} \nonumber \\
        &= \left(\sum_{s,a} \mu_h^P(s,a)\right) \cdot \max_{s,a, \pi}\frac{d_h^{\pi, P}(s,a)}{\mu_h^P(s,a)} \nonumber \\
        &= \max_{s,a, \pi}\frac{d_h^{\pi, P}(s,a)}{\mu_h^P(s,a)}. \label{eq:cov_bound} 
    \end{align}
    We let $\mu_h^P$ be the minimizer of $\max_{s,a, \pi}\frac{d_h^{\pi, P}(s,a)}{\mu_h^P(s,a)}$. The coverability in MDP $M$ is defined as $\min_\mu \max_{s,a, \pi, h}\frac{d_h^{\pi, P}(s,a)}{\mu_h^P(s,a)}$ \citep{xie2022role}. Combining the inequalities proves $\tra_\eta^{\Phi, \tilD_{\hsbe}}(\nu) \leq O\left(\eta d B^2 H^2\right)$. 
\end{proof}

\begin{lemma}\label{lem: bound Qcomp cov}

For hybrid MDPs with Bellman completeness and coverability bounded by $d$, it holds that 
\begin{align*}
    \max_\nu \qcom^{\Phi, \tilD_\hsbe}_\eta(\nu) \le \order\left(\left(B^2 H^5 d^3 \eta\right)^{\frac{1}{3}}\right).  
\end{align*}
\end{lemma}

\begin{proof}
By definition, 
\begin{align*}
    \qcom_\eta^{\Phi, \tilD_{\hsbe}}(\nu) &= 6\sqrt{dH}\sqrt{3\E_{\phi \sim \nu}\E_{(P,R) \sim \nu}\left[\left(V_{\phi, R}(\pi_{\phi}) - V_{P,R}(\pi_{\phi}) \right)^2\right]} - \frac{2}{9\eta}\E_{\phi' \sim \nu}\E_{M \sim \nu}\E_{\phi \sim \nu}\left[\tilD_{\hsbe}^{\pi_{\phi'}}(\phi\|M)\right]
\end{align*}
Define 
    \begin{align*}
        g_h(s,a, \phi; R, P)  &= f_{\phi}(s,a; R) - R(s,a) - \E_{s'\sim P(\cdot|s,a)}[f_{\phi}(s'; R)], \\
        d_h^{\nu, P}(s,a) &= \E_{\phi\sim \nu}\left[d_h^{\pi_\phi, P}(s,a) \right]. 
    \end{align*}
We have
\begin{align*}
    &\E_{\phi \sim \nu}\left[\left(V_{\phi, R}(\pi_{\phi}) - V_{P,R}(\pi_{\phi}) \right)^2\right]
    \\&= H \sum_{h=1}^H \E_{\phi \sim \nu}\E_{(s,a) \sim d_h^{\pi_{\phi}, P}}\left[g_h(s,a, \phi; R, P)^2\right]
    \\&\le H \sum_{h=1}^H \E_{\phi \sim \nu}\E_{(s,a) \sim d_h^{\pi_{\phi}, P}}\left[\left|g_h(s,a, \phi; R, P)\right|\right]
    \\&= H \sum_{h=1}^H\E_{\phi\sim \nu} \E_{(s,a)\sim d_h^{\nu, P}} \left[ \frac{d_h^{\pi_\phi, P}(s,a)}{d_h^{\nu, P}(s,a)} \left|g_h(s,a, \phi; R,P)\right|\right] 
    \\&\le H\sum_{h=1}^H\sqrt{\E_{\phi\sim \nu} \E_{(s,a)\sim d_h^{\nu, P}} \left[ \frac{d_h^{\pi_\phi, P}(s,a)^2}{d_h^{\nu, P}(s,a)^2}\right] }\sqrt{\E_{\phi\sim \nu} \E_{(s,a)\sim d_h^{\nu, P}} \left[ \left(g_h(s,a, \phi; R,P)\right)^2\right]}
    \\&\le H\sum_{h=1}^H\sqrt{d\E_{\phi\sim \nu} \E_{(s,a)\sim d_h^{\nu, P}} \left[ g_h(s,a, \phi; R,P)^2\right] }.\tag{by \pref{eq:cov_bound} and that coverability $\leq d$}
\end{align*}
Thus,
\begin{align*}
    &6\sqrt{dH}\sqrt{3\E_{\phi \sim \nu}\E_{(P,R) \sim \nu}\left[\left(V_{\phi, R}(\pi_{\phi}) - V_{P,R}(\pi_{\phi}) \right)^2\right]} - \frac{2}{9\eta}\E_{\phi' \sim \nu}\E_{M \sim \nu}\E_{\phi \sim \nu}\left[\tilD_{\hsbe}^{\pi_{\phi'}}(\phi\|M)\right]
    \\&\le  \sqrt{d^{\frac{3}{2}}H^2\sum_{h=1}^H\E_{(P,R) \sim \nu}\left[\sqrt{\E_{\phi\sim \nu} \E_{(s,a)\sim d_h^{\nu, P}} \left[ g_h(s,a, \phi; R,P)^2\right] } \right]} - \frac{2}{9\eta}\E_{\phi' \sim \nu}\E_{M \sim \nu}\E_{\phi \sim \nu}\left[\tilD_{\hsbe}^{\pi_{\phi'}}(\phi\|M)\right]
    \\&\le d^{\frac{3}{2}}H^2\beta + \frac{1}{4\beta}\sum_{h=1}^H\E_{(P,R) \sim \nu}\left[\sqrt{\E_{\phi\sim \nu} \E_{(s,a)\sim d_h^{\nu, P}} \left[ g_h(s,a, \phi; R,P)^2\right] } \right] 
    \\&\qquad - \frac{2}{9\eta B^2 H}\sum_{h=1}^H\E_{\phi' \sim \nu}\E_{M \sim \nu}\E_{\phi \sim \nu}\E_{(s,a)\sim d_h^{\pi_{\phi'}, P}} \left[ g_h(s,a, \phi; R,P)^2\right] \tag{\pref{eq:BC-equal}}
    \\&\le \order\left(d^{\frac{3}{2}}H^2\beta + \frac{\eta B^2 H}{\beta^2}\right) = \order\left(\left(B^2 H^5 d^3 \eta\right)^{\frac{1}{3}}\right).   
\end{align*}
\end{proof}

\begin{lemma}\label{lem: digdec_adv < d hsq}
  For hybrid MDPs with Bellman completeness and coverability bounded by $d$, it holds that 
  \begin{align*}
      \mfdec_\eta^{\Phi,\tilD_{\hbi}}  = O\left( B^2H^2d\eta + \left( B^2H^5d^3 \eta \right)^{\frac{1}{3}} \right). 
  \end{align*}
\end{lemma}
\begin{proof}
   This can be obtained by directly combining \pref{lem: air by comp}, \pref{lem: bound Rcompe}, \pref{lem: bound Pcomp cover}, \pref{lem: bound Qcomp cov}.
\end{proof}

\newpage
\section{Omitted Details in \pref{sec: compare DEC}}\label{app:compare DEC}

\subsection{Proof of \pref{thm:lower bound}}
In this section, we will use $\text{Ber}(p)$ to denote Bernoulli distribution with success probability $p$. We consider parameters $\epsilon$ and $\Delta$ with $\epsilon < \Delta = \frac{1}{16\sqrt{T}} \leq \frac{1}{16}$. Define $p^+ = \frac{1}{2} + \Delta$ and $p^- = \frac{1}{2} - \Delta$. Let $\mathbb{H}(\nu)$ denote the entropy of distribution $\nu$. We assume learning rate $\eta \leq 1$.  

Consider a three-arm bandit environment with model class $\mathcal{M} = \{M_1, M_2\}$ where
\begin{itemize}
    \item $M_1 = \left(\text{Ber}\left(p^-\right), \text{Ber}\left(p^+\right), \epsilon\text{Ber}(0.5)\right)$. The reward distribution is $\text{Ber}\left(p^-\right)$ for arm $a_1$ and   $\text{Ber}\left(p^+\right)$ for arm $a_2$. Arm $a_3$'s reward is $0$ and $\epsilon$ with equal probability. 
    \item $M_2 = \left(\text{Ber}\left(p^+\right), \text{Ber}\left(p^-\right), 0.5\epsilon\right)$. The reward distribution is $\text{Ber}\left(p^+\right)$ for arm $a_1$ and  $\text{Ber}\left(p^-\right)$ for arm $a_2$. Arm $a_3$'s reward is $0.5\epsilon$ deterministically. 
\end{itemize}
In this setting, $\Phi$ contains two infosets (based on \pref{assum: function approximation stochastic}): 
\begin{align*}
    \phi_1 = \left\{\left(M_1, \pi_{M_1}\right)\right\},\quad \phi_2=\left\{\left(M_2, \pi_{M_2}\right)\right\}. 
\end{align*}
In the rest of this proof, we compare the optimistic E2D algorithm \citep{foster2024model} and our algorithm in this environment. 

\paragraph{Optimistic DEC algorithm \citep{foster2024model}}
Given $\rho_t \in \Delta(\Phi)$, the algorithm chooses action distribution via
\begin{align}
    p_t = \argmin_{p \in \Delta(\Pi)} \max_{\nu \in \Delta(\Psi)} \E_{a\sim p} \E_{\phi\sim \rho_t} \E_{M\sim \nu} \left\{V_{\phi}(a_{\phi}) - V_M(a) - \frac{1}{\eta} D^{a}(\phi \| M) \right\}
\label{eq:dec-three}
\end{align}
where $a_{\phi}$ is the optimal action of infoset $\phi$. In this simple bandit setting, the bilinear divergence and the squared Bellman error coincide with
\begin{align*} 
 D^a(\phi\| M) = \left(\E^{a, M} [V_\phi(a) - r]\right)^2 = (V_\phi(a) - V_M(a))^2. 
\end{align*}
We first consider the divergence term, for action $a \in \{a_1, a_2\}$, we have
\begin{align}
    \E_{\phi \sim \rho_t}\E_{M \sim \nu}\left[D^a(\phi \| M)\right] &= \rho_t(\phi_1)\nu(M_2)(V_{\phi_1}(a) - V_{M_2}(a))^2 +  \rho_t(\phi_2)\nu(M_1)(V_{\phi_2}(a) - V_{M_1}(a))^2 \nonumber
    \\&= 4\left(\rho_t(\phi_1)\nu(M_2) + \rho_t(\phi_2)\nu(M_1)\right)\Delta^2 \label{eq:D12}
\end{align}
For action $a = a_3$, we have
\begin{align}
    \E_{\phi \sim \rho_t}\E_{M \sim \nu}\left[D^a(\phi\| M)\right] &= \rho_t(\phi_1)\nu(M_2)(V_{\phi_1}(a) - V_{M_2}(a))^2 +  \rho_t(\phi_2)\nu(M_1)(V_{\phi_2}(a) - V_{M_1}(a))^2 \nonumber
    \\&= 0 
    \label{eq:D3}
\end{align}
Thus, for any $\rho_t$ and $\nu$, we have
\begin{align*}
\E_{a \sim p}\E_{\phi\sim \rho_t} \E_{M\sim \nu}\left[ - \frac{1}{\eta} D^{a}(\phi\|M)\right]  &= -\frac{4(1-p(a_3))\Delta^2}{\eta}\left(\rho_t(\phi_1)\nu(M_2) + \rho_t(\phi_2)\nu(M_1)\right) 
\end{align*}
which is monotonically increasing in $p(a_3)$. 

We then consider the regret term. For any $p \in \Delta\left(\Pi\right)$, define $\Tilde{p} = \left(\frac{p(a_1)}{1-p(a_3)}, \frac{p(a_2)}{1-p(a_3)}, 0\right)$ if $p(a_3)<1$, and $\Tilde{p}=(\frac{1}{2}, \frac{1}{2}, 0)$ otherwise. For any $M \in \mathcal{M}$, when $p(a_3)<1$ we have 
\begin{align*}
    \E_{a \sim p}\left[V_{M}(a)\right] - \E_{a \sim \Tilde{p}}\left[V_M(a)\right] &=  \sum_{a \in \{a_1, a_2\}}\left(p(a) - \Tilde{p}(a)\right)V_{M}(a) + p(a_3)V_M(a_3) 
    \\&= \frac{-p(a_3)}{1-p(a_3)}\sum_{a \in \{a_1, a_2\}} p(a) V_M(a)  + p(a_3)V_M(a_3) 
    \\&\le \frac{-p(a_3)}{1-p(a_3)}\left(p(a_1) + p(a_2)\right)p^-  + p(a_3)V_M(a_3) \tag{$V_M(a) \ge p^-$ for any $M$ and $a \in \{a_1, a_2\}$, and $p(a_3) < 1$}
    \\&= p(a_3)\left(V_M(a_3) - \frac{1}{2} + \Delta\right)
    \\&\le p(a_3)\left(0.5\epsilon + \Delta -\frac{1}{2}\right) \le 0 \tag{$\epsilon < \Delta \leq \frac{1}{16}$}, 
\end{align*}
and when $p(a_3)=1$ we also have $\E_{a \sim p}\left[V_{M}(a)\right] - \E_{a \sim \Tilde{p}}\left[V_M(a)\right]\leq 0$.  
Thus, for any $\rho_t$, $\nu$, and $p$, 
\begin{align*}
    \E_{a\sim \Tilde{p}} \E_{\phi\sim \rho_t} \E_{M\sim \nu} \left\{V_\phi(a_\phi) - V_M (a)\right\} \le  \E_{a\sim p} \E_{\phi\sim \rho_t} \E_{M\sim \nu} \left\{V_\phi(a_\phi) - V_M(a)\right\}.
\end{align*}
Combining the discussion of the above two terms, for any $\rho_t, \nu$ and $p$, we have
\begin{align}
    \E_{a\sim \Tilde{p}} \E_{\phi\sim \rho_t} \E_{M\sim \nu} \left\{V_\phi(a_\phi) - V_M (a) - \frac{1}{\eta} D^{a}(\phi\|M) \right\} \le \E_{a\sim p} \E_{\phi\sim \rho_t} \E_{M\sim \nu} \left\{V_\phi(a_\phi) - V_M (a) - \frac{1}{\eta} D^{a}(\phi\| M) \right\}. 
\label{eq:minmax_start}
\end{align}


Given \pref{eq:minmax_start}, the minimax solution of \pref{eq:dec-three} must have $p_t(3) = 0$ for any $\rho_t$ and any $t$. This implies that the optimistic DEC algorithm will never choose $a_3$ and the problem degenerate to standard two-arm bandit, so the policy derived from optimistic DEC objective \pref{eq:dec-three} must suffer standard regret lower bound $\E\left[\Reg(\pi_{M^\star})\right] \ge \Omega(\sqrt{T})$ given $\Delta = \Theta\left(\frac{1}{\sqrt{T}}\right)$.

\paragraph{Our algorithm}


Given $\rho_1$ is a uniform distribution, we consider our first step optimization where
\begin{align}
    p_1 = \argmin_{p \in \Delta(\Pi)} \max_{\nu \in \Delta(\Psi)} \E_{a\sim p} \E_{\phi\sim \rho_1} \E_{M\sim \nu} \left\{V_M(a_M) - V_M (a) -\frac{1}{\eta} \E_{o\sim M(\cdot|a)}\left[\KL(\nu_{\bo{\phi}}(\cdot|a, o), \rho_1)\right] - \frac{1}{\eta} D^{a}(\phi\| M) \right\}. 
\label{eq:dec-our}
\end{align}
Below, we discuss the four terms in \pref{eq:dec-our}. \\

\noindent\underline{The $V_M(a_M)$ term}\ \ \ For any $\nu$, we have $\E_{M\sim \nu}[V_M(a_M)] = p^+$, which is a constant. Therefore, this term can be ignored in the objective. \\

\noindent\underline{The $V_M(a)$ term}\ \ \ 
By direct calculation, we have 
\begin{align}
    &\E_{a\sim p} \E_{M\sim \nu} \left[V_M (a) \right]= \frac{p(a_1) + p(a_2)}{2} + \left(p(a_1) - p(a_2)\right)\left(\nu(M_2) - \nu(M_1)\right)\Delta + 0.5p(a_3)\epsilon. \label{eq: temp 111} 
\end{align}
For any $p = (p(a_1), p(a_2), p(a_3))$, consider $\hat{p} = (\frac{p(a_1) + p(a_2)}{2}, \frac{p(a_1) + p(a_2)}{2}, p(a_3))$. By \pref{eq: temp 111} we have
\begin{align}
    \max_{\nu \in \Delta(\Psi)}\E_{a\sim \hat{p}} \E_{M\sim \nu} \left[-V_M (a)\right] \leq \max_{\nu \in \Delta(\Psi)}\E_{a\sim p} \E_{M\sim \nu} \left[-V_M (a)\right]. 
\label{eq:modify_p}
\end{align}

\noindent\underline{The $D^a(\phi\|M)$ term} \ \ \ Given $\rho_1$ is a uniform distribution, for action $a \in \{1,2\}$, from \pref{eq:D12}, for any $\nu$  we have $\E_{\phi \sim \rho_1}\E_{M \sim \nu}\left[D^a(\phi\|M)\right] = 2\Delta^2$. For action $a = 3$, from \pref{eq:D3}, for any $\nu$,  we have $\E_{\phi \sim \rho_1}\E_{M \sim \nu}\left[D^a(\phi\|M)\right] = 0$. Hence, $\E_{a\sim p}\E_{\phi \sim \rho_1}\E_{M \sim \nu}\left[D^a(\phi\|M)\right] = 2(1-p(a_3))\Delta^2$.  Note that now this is independent of  $\nu$, and only related to $p(a_3)$ or $p(a_1) + p(a_2)$ but not $p(a_1)$ or $p(a_2)$ individually. \\

\noindent\underline{The $\KL$ term}\ \ \ Notice that 
\begin{align*}
    &\nu_{\bo{o}}(\cdot|a_1, \phi_1) = \text{Ber}\left(p^-\right),  \quad  \nu_{\bo{o}}(\cdot|a_2, \phi_1) = \text{Ber}\left(p^+\right), \quad \nu_{\bo{o}}(\cdot|a_1, \phi_2) = \text{Ber}\left(p^+\right),  \quad \nu_{\bo{o}}(\cdot|a_2, \phi_2) = \text{Ber}\left(p^-\right),
    \\&\nu_{\bo{o}}(\cdot|a_1) =\text{Ber}\left(m_1\right),  \qquad \nu_{\bo{o}}(\cdot|a_2) =\text{Ber}\left(m_2\right),
\end{align*}
where $m_1 = \nu(\phi_1)p^- +  \nu(\phi_2)p^+$ and $m_2 =\nu(\phi_1)p^+ +  \nu(\phi_2)p^-$ and it holds that $m_1 + m_2 = 1$.  Given that $\KL\left(\text{Ber}\left(p\right), \text{Ber}\left(q\right)\right) = \KL\left(\text{Ber}\left(1-p\right), \text{Ber}\left(1-q\right)\right)$, we have
\begin{align*}
&\E_{a\sim p} \E_{M\sim \nu} \left[\E_{o\sim M(a)}\left[\KL(\nu_{\pmb{\phi}}(\cdot|a, o), \rho_1)\right]\right]
\\&= \E_{a\sim p} \E_{\phi \sim \nu}\left[\KL(\nu_{\pmb{o}}(\cdot|a, \phi), \nu_{\pmb{o}}(\cdot|a))\right]+ \KL(\nu_{\pmb{\phi}}, \rho_1)
\\&= p(a_1)\nu(\phi_1)\KL\left(\text{Ber}\left(p^-\right),  \text{Ber}\left(m_1\right)\right) +  p(a_2)\nu(\phi_1)\KL\left(\text{Ber}\left(p^+\right),  \text{Ber}\left(m_2\right)\right) + \KL(\nu_{\pmb{\phi}}, \rho_1)
    \\&\qquad + p(a_1)\nu(\phi_2)\KL\left(\text{Ber}\left(p^+\right),  \text{Ber}\left(m_1\right)\right) +  p(a_2)\nu(\phi_2)\KL\left(\text{Ber}\left(p^-\right),  \text{Ber}\left(m_2\right)\right) 
    \\&\qquad + p(a_3)\E_{\phi \sim \nu}\left[\KL(\nu_{\pmb{o}}(\cdot|a_3, \phi), \nu_{\pmb{o}}(\cdot|a_3))\right]
    \\&=  \left(p(a_1) + p(a_2)\right)\left(\nu(\phi_1)\KL\left(\text{Ber}\left(p^-\right),  \text{Ber}\left(m_1\right)\right) + \nu(\phi_2)\KL\left(\text{Ber}\left(p^+\right),  \text{Ber}\left(m_1\right)\right)\right)
    \\&\qquad + p(a_3)\mathbb{H}(\nu) + \KL(\nu_{\pmb{\phi}}, \rho_1)
    \\&= \left(1-p(a_3)\right)\left(\mathbb{H}\left(\text{Ber}(m_1)\right) - \mathbb{H}\left(\text{Ber}\left(p^+\right)\right)\right) + p(a_3)\mathbb{H}(\nu) +  \KL(\nu_{\pmb{\phi}}, \rho_1). 
\end{align*}

Note that this term is only related to $p(a_3)$ or $p(a_1)+p(a_2)$, but not $p(a_1)$ or $p(a_2)$ individually. \\

\noindent\underline{Combining terms} \ \ \ 
Combining the case discussions above, for any $p = (p(a_1), p(a_2), p(a_3))$, with $\hat{p} = (\frac{p(a_1) + p(a_2)}{2}, \frac{p(a_1) + p(a_2)}{2}, p(a_3))$, we have
\begin{align*}
    &\max_{\nu \in \Delta(\Psi)}\left\{\E_{a\sim \hat{p}} \E_{M\sim \nu} \left[-V_M (a) - \frac{1}{\eta}\E_{o\sim M(a)}\left[\KL(\nu_{\pmb{\phi}}(\cdot|a, o), \rho_1)\right]-\frac{1}{\eta}\E_{\phi\sim \rho_1}[D^a(\phi\|M)]\right]\right\}  
    \\&\le \max_{\nu \in \Delta(\Psi)}\left\{\E_{a\sim p} \E_{M\sim \nu} \left[-V_M (a)- \frac{1}{\eta}\E_{o\sim M(a)}\left[\KL(\nu_{\pmb{\phi}}(\cdot|a, o), \rho_1)\right]-\frac{1}{\eta}\E_{\phi\sim \rho_1}[D^a(\phi\|M)]\right]\right\}. 
\end{align*}
To calculate the max value of the left-hand-side,  consider policy distribution $p_s = (\frac{1-s}{2}, \frac{1-s}{2}, s)$. We have
\begin{align}
    &\E_{a\sim p_s} \E_{M\sim \nu} \left[-V_M (a) - \frac{1}{\eta}\E_{o\sim M(a)}\left[\KL(\nu_{\pmb{\phi}}(\cdot|a, o), \rho_1)\right]-\frac{1}{\eta}\E_{\phi\sim \rho_1}[D^a(\phi\|M)]\right] \nonumber
    \\&= \frac{s-1}{2}-\frac{s\epsilon}{2} -\frac{1}{\eta}\left( \left(1-s\right)\left(\mathbb{H}\left(\text{Ber}(m_1)\right) - \mathbb{H}\left(\text{Ber}(p^+)\right) + 2\Delta^2\right) +  \KL\left(\nu_{\pmb{\phi}}, \rho_1\right) + s\mathbb{H}(\nu) \right) \label{eq: G_init}
\end{align}
where $m_1 = \nu(\phi_1)p^{-} +  \nu(\phi_2)p^{+}$. Define
\begin{align*}
    G(\nu) = (1-s)\mathbb{H}(\text{Ber}(m_1)) + \KL\left(\nu_{\pmb{\phi}}, \rho_1\right) + s\mathbb{H}(\nu). 
\end{align*}

To calculate $\max_{\nu}$ of \pref{eq: G_init}, we only need to consider $ \min_{\nu}\left\{G(\nu)\right\}$. By setting $\nu(\phi_2) = 1 - \nu(\phi_1)$,  function $G$ is only related to $\nu(\phi_1)$ and we denote it as $G(\nu(\phi_1))$, after taking derivative, we have
\begin{align*}
    G'(\nu(\phi_1))&=(1-s)\ln\left(\frac{1-m_1}{m_1}\right)\left(p^{-} - p^+\right) + \log\left(\frac{\nu(\phi_1)}{1-\nu(\phi_1)}\right) + s\log\left(\frac{1-\nu(\phi_1)}{\nu(\phi_1)}\right)
    \\&= -\Delta(1-s)\ln\left(\frac{1-m_1}{m_1}\right) + \log\left(\frac{\nu(\phi_1)}{1-\nu(\phi_1)}\right) + s\log\left(\frac{1-\nu(\phi_1)}{\nu(\phi_1)}\right)
\end{align*}
where $m_1 = \nu(\phi_1)p^{-} +  (1- \nu(\phi_1))p^{+}$ and we use the fact that $\frac{\mathrm{d}\mathbb{H}(\text{Ber}(p))}{\mathrm{d}p} = \ln\left(\frac{1-p}{p}\right)$. Note that when $\nu(\phi_1) = \frac{1}{2}$ we have $m_1 = \frac{1}{2}$ and $G'(\frac{1}{2}) = 0$. Thus, $\frac{1}{2}$ is a stationary point. On the other hand, we have $G''(\frac{1}{2}) = 4(1-s - 2(1-s)\Delta^2) \geq 0$ and $G(\nu(\phi_1)) = G(1- \nu(\phi_1))$. This implies $\nu(\phi_1) = \frac{1}{2}$ is the unique minimizer and the minimal value is $G(\frac{1}{2}) = \ln(2)$.

Thus, 
\begin{align*}
  &\max_{\nu \in \Delta(\Psi)}\left\{\E_{a\sim p_s} \E_{M\sim \nu} \left[-V_M (a) - \frac{1}{\eta}\E_{o\sim M(a)}\left[\KL(\nu_{\pmb{\phi}}(\cdot|a, o), \rho_1)\right]-\frac{1}{\eta}\E_{\phi\sim \rho_1}[D^a(\phi\|M)]\right]\right\} \\
  &= \frac{s-1}{2} - \frac{s\epsilon}{2} -\frac{1}{\eta} \left(1-s\right) \left(-\mathbb{H}\left(\text{Ber}(p^+)\right) + 2\Delta^2\right) - \frac{1}{\eta}  \ln(2)
  \\&= (1-s) \left(-\frac{1-\epsilon}{2} + \frac{\mathbb{H}(\text{Ber}(p^+)) - 2\Delta^2}{\eta}\right) - \frac{\ln 2}{\eta} -\frac{\epsilon}{2}.   \numberthis \label{eq: tmp334}
\end{align*}
Note that 
\begin{align*}
    &\mathbb{H}(\text{Ber}(p^+)) - 2\Delta^2 \\
    &= -\KL(\text{Ber}(p^+), \text{Ber}(\tfrac{1}{2})) + \ln 2 - 2D_{\textrm{TV}}^2(\text{Ber}(p^+), \text{Ber}(\tfrac{1}{2})) 
    \\&\geq \ln 2 - 5\KL(\text{Ber}(p^+), \text{Ber}(\tfrac{1}{2}))  \tag{Pinsker's inequality}
    \\&\geq \ln 2 - 15\Delta^2  \tag{$\KL(\text{Ber}(\frac{1}{2}+\Delta), \text{Ber}(\frac{1}{2}))\leq 3\Delta^2$ for $\Delta\leq \frac{1}{2}$}
    \\&\geq \frac{1}{2}.   \tag{by the assumption $\Delta = \frac{1}{16\sqrt{T}}\leq \frac{1}{16}$}
\end{align*}

Hence, the minimum value of \pref{eq: tmp334} is achieved at $s=1$ when $\frac{1}{2\eta} - \frac{1-\epsilon}{2}\geq 0$.  By the condition $\eta\leq 1$, this indeed holds. This means that 
our algorithm always picks the third arm in the first round. After picking arm $a_3$, the belief of $\phi$ will be deterministic, since $\nu_1(\phi|a_3, o)=0$ for any $\phi\neq \phi^\star$. 
This means the algorithm will always choose the optimal action in the following rounds, ensuring that $\E\left[\Reg(\pi_{M^\star})\right] \leq p^+ < 1$.




\end{document}